\documentclass[11pt,a4paper]{article}
\usepackage[utf8]{inputenc}
\usepackage[T1]{fontenc}
\usepackage{lmodern}
\usepackage[english]{babel}
\usepackage[margin=2.5cm]{geometry} % Page margins
\usepackage{graphicx} % For including figures
\usepackage{booktabs} % For better-looking tables
\usepackage{natbib} % For author-year citations
\usepackage{subcaption}
\usepackage{amsmath} % For mathematical equations
\usepackage{amssymb} % For AMS symbols
\usepackage{setspace} % For line spacing
\usepackage{caption} % For customizing captions
\usepackage[colorlinks=true, citecolor=blue, linkcolor=blue, urlcolor=blue]{hyperref} % For hyperlinks

\usepackage{graphicx}
\usepackage{amsmath,amsthm,amssymb,amsfonts}
\usepackage[italicdiff]{physics}
\usepackage[T1]{fontenc}
\usepackage{wrapfig}
\usepackage{lmodern,mathrsfs}
\usepackage[inline,shortlabels]{enumitem}
\usepackage[thinc]{esdiff}
\usepackage{mathtools}
\usepackage{microtype}
\usepackage{graphicx,wrapfig}
\usepackage{booktabs}
\usepackage{comment}
\usepackage{relsize}
\usepackage{adjustbox}
\usepackage{float}
\usepackage{multirow}
\usepackage{placeins} 
\usepackage{epsfig,wrapfig,graphicx} 
\usepackage{threeparttable,multirow}
\usepackage{color}
\usepackage{adjustbox,booktabs,arydshln,multirow,url}
\usepackage{algorithm}
\usepackage{algorithmic}
  {\begin{Sbox}\begin{minipage}}%
  {\end{minipage}\end{Sbox}\fbox{\TheSbox}}

\usepackage{hyperref}
\usepackage{tikz}
 
\newcommand{\E}{{\mathbb{E}}}
\newcommand{\bX}{{\boldsymbol X}}
\newcommand{\bY}{{\boldsymbol Y}}

\newcommand{\bx}{{\boldsymbol x}} 
\newcommand{\bz}{{\boldsymbol z}}

\newcommand{\bb}{{\boldsymbol b}}
\newcommand{\be}{{\boldsymbol e}}
\newcommand{\bw}{{\boldsymbol w}}
\newcommand{\bW}{{\boldsymbol W}}

\newcommand{\bv}{{\boldsymbol v}}
\newcommand{\br}{{\boldsymbol r}}
\newcommand{\bh}{{\boldsymbol h}}

\newcommand{\bp}{{\boldsymbol p}}
\newcommand{\bA}{{\boldsymbol A}}
\newcommand{\bB}{{\boldsymbol B}}
\newcommand{\ba}{{\boldsymbol a}}

\newcommand{\bV}{{\boldsymbol V}}
\newcommand{\bG}{{\boldsymbol G}}

\newcommand{\bu}{{\boldsymbol u}}

\newcommand{\mR}{\mathbb{R}}

\newcommand{\mbE}{{\mathbb{E}}}

\newcommand{\bbeta}{{\boldsymbol \beta}}

\newcommand{\btheta}{{\boldsymbol \theta}}
\newcommand{\bTheta}{{\boldsymbol \Theta}}

\newcommand{\bepsilon}{{\boldsymbol \epsilon}}

\newcommand{\bsigma}{{\boldsymbol \sigma}}
\newcommand{\bSigma}{{\boldsymbol \Sigma}}
\newcommand{\bpi}{\boldsymbol{\pi}}
\newcommand{\corr}{\mbox{Corr}}
\newcommand{\Var}{{\mbox{Var}}}

\newcommand{\Cov}{{\mbox{Cov}}}
\newcommand{\diag}{{\mbox{diag}}}
\newcommand{\bYmis}{{\boldsymbol Y}_{{\rm mis}}}

\newtheorem{theorem}{Theorem}
\newtheorem{lemma}{Lemma}

\newtheorem{proposition}{Proposition}

\newtheorem{remark}{Remark}
\newtheorem{assumption}{Assumption}

\setcitestyle{round}

\let\oldfootnote\footnote
\renewcommand{\footnote}{\fontsize{9}{11}\selectfont\oldfootnote}

\title{Sublinearly Structured Deep Neural Networks Achieve Feature Learning Consistency for Compositional Functions}
\author{
    Sehwan Kim   \\
   Ewha Womans University, Seoul 03760, 
  Republic of Korea \\
    \and
    Yan Sun \\
New Jersey Institute of Technology,
Newark, NJ 07102, USA \\
    \and
    Faming Liang \thanks{Corresponding author: fmliang@purdue.edu} \\
    Purdue University, West Lafayette, IN 47907, USA 
}
\date{} % Remove date

\begin{document}

\maketitle

\begin{abstract}
Over the past decade, deep neural networks (DNNs) have achieved remarkable success on complex machine-learning tasks, yet the theoretical foundations of their performance remain incomplete. From a statistical viewpoint, a natural question is: \emph{can DNNs attain feature-learning and prediction consistency comparable to that of classical models?} While a full characterization is open, we provide positive results for a broad subclass. We establish feature-learning  consistency guarantees for \emph{sublinearly structured DNNs}—architectures whose input/output dimensions and number of hidden neurons grow sublinearly with the sample size—when learning  hierarchically compositional target functions. Importantly, this 
consistency still holds even in the conventional ``over-parameterized'' regime where the total number of parameters exceeds the number of training samples. Empirically, sublinearly structured DNNs match or surpass wide DNNs in prediction. A structural audit further indicates that widely used convolutional neural networks (CNNs), including AlexNet, VGGNet, ResNet, GoogLeNet, are sublinearly structured on their image classification benchmarks. We further prove that the sublinearly structured DNNs achieve universal approximation for hierarchically compositional functions in the large-sample limit. Moreover, images exhibit an inherent hierarchical, compositional structure. Taken together, these results explain, through a statistical lens, why many large-scale deep learning models succeed after adequate training on massive image datasets.
\end{abstract}

% AMS subject classification



\newpage

\section{Introduction}

Over the past decade, DNNs have made major breakthroughs in many research domains, including image generation, protein folding, and language processing. 
% deep neural networks (DNNs) have led to major breakthroughs in many 
% research areas, such as image generation, protein folding, and language processing.   
The ability of these models to automatically learn the problem-specific features hidden in the training data is considered as a major factor contributing to their success, see e.g. \cite{Radhakrishnan2024MechanismFF}, \cite{Shi2022ATA}, and \cite{Yang2021TensorPI}. 
Therefore, understanding the mechanism of feature learning and, by extension, designing the network structure for ensuring the hidden features to be effectively learned has attracted much attention in recent literature.

% Towards the goals, researches have been pursued in different directions. 
% One stream of works is benign overfitting, which aims to explain the phenomenon that interpolating models
% generalize well in the presence of noisy data. The 
% 1. sparse deep learning, lottery paper;  2. benign overfitting and double descent;  
% 3. works cited in  \cite{Radhakrishnan2024MechanismOF}, \cite{Shi2022ATA}, and \cite{Yang2021TensorPI}. 
% \textcolor{red}{to be detailed}

Feature learning for linear models has been well studied in statistics, see e.g., \cite{tibshirani1996regression}, \cite{FanL2001}, and \cite{loh2017support}, where one aims to identify important covariates 
through estimating their regression coefficients. For DNNs, we follow \cite{Radhakrishnan2024MechanismFF} to define a neural feature as an eigenvector of $\bw_l^{T}\bw_l$, where $\bw_l \in \mathbb{R}^{d_l\times d_{l-1}}$ denotes the weight matrix of hidden layer $l$, and $d_l$ denotes the width of layer $l$. It is easy to see that the regression coefficient vector of the linear model can be viewed 
as a special case of this general definition (with $d_1=1$ and $d_0=p$ for the number of covariates, and rescaled by $\bw_1\bw_1^\top$). 
%Thus, if a method produces a consistent estimate of $\bw_l$, it can also be said to be ``feature learning consistent''. 
%As discussed later,  since $\bw_l$ is only identifiable up to certain loss-invariant transformations,  the consistency of neural feature estimates serves as a reliable measure for the consistency of $\bw_l$'s estimates.  
From a statistical perspective, a natural question is whether a DNN can achieve feature learning consistency similar to that of linear models.

To bridge the gap between linear models and deep learning, 
\cite{SunLiang2022kernel} and \cite{LiangSLiang2022} proposed a 
new type of neural network -- stochastic neural network (abbreviated as StoNet).
This network is formulated as a composition of many linear/logistic regressions and provides a framework for transferring theory and methods from linear models to DNNs. Additionally, the StoNet offers a convenient way for addressing many important problems encountered in modern data science, such as sufficient dimension reduction \citep{LiangSLiang2022}, uncertainty quantification \citep{SunLiang2025UQ}, 
and causal effect estimation \citep{FangLiang2024causal,FangLiang2025}. 
These problems are otherwise hard to address using conventional DNNs. 
In this paper, based on the asymptotic equivalence between DNN and StoNet \citep{LiangSLiang2022},
%by treating the StoNet as a surrogate of the DNN,
we prove that the sublinearly structured  DNN achieves 
feature-learning consistency for hierarchically compositional functions in which each constituent depends on at most a bounded number 
 of variables.

 We refer to a DNN as {\it a sublinearly structured DNN} (or sublinear DNN for short) if its structure satisfies the constraints:
\begin{equation} \label{eq:sublinear}
%\small
 d_0 \prec n, \qquad d_{h+1} \prec n, \qquad \sum_{l=1}^{h} d_l \prec n,
 \end{equation}
 where  
 $d_0=p$ denotes the input dimension, 
 $d_l$ denotes the width of layer $l$, $h$ denotes the number of hidden layers, 
 $d_{h+1}$ denotes the output dimension,
 and $n$ denotes the training sample size.
 Here, we denote $a_n \prec b_n$ if $\lim_{n\rightarrow \infty} \frac{a_n}{b_n} = 0$. In equation \eqref{eq:sublinear}, the  dependence of $d_i$'s on  $n$ is implicit.
 This definition of sublinear DNN applies to fully connected neural networks and may require slight modification for convolutional neural networks (CNNs), where the filter size used at each convolutional layer must be considered in defining the corresponding StoNet (see Section \ref{deepCNN}). It is worth noting that the number of parameters in a sublinear DNN can still greatly exceed $n$; in other words, sublinear DNNs may be over-parameterized.

 Hierarchical compositional functions are multivariate maps built as compositions of low-arity modules arranged in a tree or directed acyclic graph (DAG). They also include conventional functions with a fixed input dimension as special cases.
This structure is widespread in science and engineering (e.g., PDE operators, image processing pipelines) 
and captures rich multiscale interactions while keeping each module 
low-dimensional --- precisely the regime where DNNs 
admit a sparse structure and avoid the curse of dimensionality.

% The parameter estimation consistency provides a theoretical guarantee for 
% the asymptotic recovery of the underlying true system by sublinear DNNs. 
% %given its universal approximation ability (see e.g., \citealp{Kidger2020UniversalAW}; %\citealp{Park2020MinimumWF},
% % \citealp{Kim2023MinimumWF}).
% % and \citealp{Montanelli2019NEB}).  
% This, in turn, ensures consistent predictions for future values.
% % \textcolor{black}{Our results therefore delineate a structural regime for deep learning in which the existence of an ``optimal solution'' (in the sense of system recovery or feature learning consistency) is theoretically guaranteed.}
\textcolor{black}{Additionally, we prove that sublinear DNNs achieve 
universal approximation for hierarchically compositional 
functions in the large-sample limit.} 
We analyze the structures of many popular large-scale DNNs, such as AlexNet, VGGNet, ResNet, and GoogLeNet, used in image classification, and find that they 
are all sublinear on their image classification benchmarks. 
Furthermore, images exhibit an inherent hierarchical, compositional structure. 
Taken together, our results explain, through a statistical lens, why many large-scale DNNs succeed after adequate training on massive image datasets.
%This finding helps elucidate their success in prediction after adequate training on massive datasets. 
% Our numerical experiments show that sublinear DNNs can achieve prediction accuracy comparable to, or even exceeding, that of excessively wide DNNs 
% for hierarchically compositional functions.

To our knowledge, this work provides the first theoretical result on feature-learning consistency for DNNs in the over-parameterized regime, although our analysis is restricted to the class of hierarchically compositional functions. This represents a notable distinction from existing studies, where parameter estimation consistency, up to loss-invariant transformations, has been established under the assumption that the total number of parameters or effective parameters is of lower order than $n$ \citep{Farrell2021DeepNN,SunSLiang2021}. Parameter estimation consistency is stronger than feature-learning consistency, but it typically requires more restrictive network-size conditions. Our numerical experiments show that sublinear DNNs can achieve prediction accuracy comparable to, or even exceeding, that of wide DNNs for hierarchically compositional functions.

%\vspace{-0.15in}
\paragraph{Related Works}
Motivated by the observation of benign overfitting \citep{bartlett2020benign}, a line of work has emerged studying 
 the properties of wide DNNs, see e.g., 
 \cite{Yang2021TensorPI} and \cite{Woodworth2020KernelAR}. 
 These studies typically rely on two key assumptions: (i) the wide DNNs are trained using gradient-descent-type methods, and (ii) the scale of initialization is appropriately chosen. 
For instance, \cite{Yang2021TensorPI} noted that the standard initialization of neural networks do not admit infinite-width limits that can learn features and proposed a modification to enable feature learning in the limit. \cite{Woodworth2020KernelAR} showed how the scale of the initialization controls the transition between the kernel and feature learning regimes. 
Although our theory is restricted to 
the class of hierarchically compositional functions,  
it does not depend on the specific optimization algorithm used or the initialization scale adopted. 
On the restricted domain, our theory can be seen as complementary to those on wide DNNs, thereby providing a complete spectrum of theoretical insights on DNNs from narrow to wide.

% Another line of work related to ours focuses on the intrinsic dimensionality of large-scale DNNs (see e.g., \citealp{Li2018MeasuringTI} and \citealp{Aghajanyan2020IntrinsicDE}). These studies have shown that large-scale DNNs often possess a very low intrinsic dimension, which changes little even as models grow in width or depth. 
% This finding suggests that large-scale DNNs can be effectively reparameterized in a low-dimensional space, demonstrating the hierarchical composition nature of the 
% underlying target function and thus 
% making sublinear DNNs a viable and efficient option.

Another line of related work investigates the intrinsic dimensionality of large-scale DNNs (see, e.g., \citealp{Li2018MeasuringTI,Aghajanyan2020IntrinsicDE}). These studies show that such networks often have a very low intrinsic dimension that changes little as width or depth increase.  
Building on this observation, several low-rank adaptation methods have been proposed for fine-tuning large language models, such as LoRA \citep{Hu2021LoRALA} and QLoRA \citep{Dettmers2023QLoRAEF}. 
This, in turn, suggests that large-scale DNNs admit effective low-dimensional reparameterizations, reflecting the hierarchical, compositional structure of the underlying target function and making sublinear DNNs a viable and efficient option. 

\section{DNN and Its Stochastic Surrogate} \label{stonetsection}

%\vspace{-0.1in}
% Consider a DNN model with $h$ hidden layers.
% For the sake of simplicity, we assume that the same activation function $\psi(\cdot)$ is used for all hidden neurons. By separating the feeding and activation operators for each hidden neuron, we can rewrite the DNN in the form: 
Consider a DNN model with $h$ hidden layers defined as follows:
\begin{equation}\label{eq:DNN}
    \begin{split}
    \tilde{\bY}_1 &= \bb_1 + \bw_1\bX, \\
    \tilde{\bY}_i &= \bb_i + \bw_i\Psi(\tilde{\bY}_{i-1}), \quad i=2,3,\dots,h,\\
    \bY &= \bb_{h+1} + \bw_{h+1}\Psi(\tilde{\bY}_{h}) + \be_{h+1}, 
    \end{split}
\end{equation}
where $\be_{h+1}\sim N(0, \sigma_{h+1}^2 I_{d_{h+1}})$ is Gaussian random error; 
%$d_i$ denotes the width of layer $i$; 
 $\bX\in \mR^{d_0}$;
$\tilde{\bY}_i \in \mathbb{R}^{d_i}$ denotes the pre-activation of the $i$-th hidden layer;
$\bw_i \in \mathbb{R}^{d_i \times d_{i-1}}, \bb_i \in \mathbb{R}^{d_i}$ denotes the weights and bias of the $i$-th layer;
$\Psi(\cdot)$ is the element-wise activation function such that $\Psi(\tilde{\bY}_{i-1})=(\Psi(\tilde{Y}_{i-1,1}), \Psi(\tilde{Y}_{i-1,2}), \ldots \Psi(\tilde{Y}_{i-1,d_{i-1}}))^\top$. 
% $d_i$ denotes the width of layer $i$ for $i=1,2,\ldots,h+1$; 
% $\tilde{\bY}_i, \bb_i \in \mathbb{R}^{d_i}$ for $i=1,2,\ldots,h$;  
% $\bY, \bb_{h+1} \in \mathbb{R}^{d_{h+1}}$;
% $\Psi(\tilde{\bY}_{i-1})=(\psi(\tilde{Y}_{i-1,1}), \psi(\tilde{Y}_{i-1,2})$, $\ldots$, 
% $\psi(\tilde{Y}_{i-1,d_{i-1}}))^\top$ for $i=2,3,\ldots,h+1$, 
% %$\psi(\cdot)$ is the activation function, 
% and $\tilde{Y}_{i-1,j}$ is the $j$th element of $\tilde{\bY}_{i-1}$; $\bw_i \in \mathbb{R}^{d_i \times d_{i-1}}$ for $i=1,2,\ldots, h+1$; $\bX \in \mathbb{R}^{d_0}$, and $d_0=p$ denotes the dimension of $\bX$. 
%In formula (\ref{eq: nn reform}), we reformulate (\ref{eq: nn}) by separating the feeding operator ($\bw_i \bm{A}_{i-1} + \bm{b}_i$) and the activation operator $\psi(\cdot)$. As shown later in this paper, such transformation facilitates parameter estimation for the neural network when auxiliary noise are introduced into the model.
For simplicity, we consider only regression problems in (\ref{eq:DNN}). It can be easily extended to classification problems by replacing the third  equation in (\ref{eq:DNN}) with a logit model. 
% By replacing the third equation of (\ref{eq:DNN}) with a logit model, the DNN can be extended to classification problems.
%Even though neural network has proven to exhibit better prediction performance over traditional models for many complex problems, the determinism of neural network makes it hard to establish solid theory and apply it in sufficient dimension reduction field.

% The StoNet \citep{LiangSLiang2022} is {\it a probabilistic deep learning model} and constructed by adding auxiliary noise to
% $\tilde{\bY}_i$'s  in (\ref{eq:DNN}). Mathematically, the StoNet is given by 
The StoNet \citep{LiangSLiang2022} is {\it a probabilistic deep learning model} defined by adding auxiliary noise to the pre-activation $\tilde{\bY}_i$'s in the DNN model \eqref{eq:DNN}:
\begin{equation}\label{eq:stonet}
    \begin{split}
    \bY_1 &= \bb_1 + \bw_1\bX + \be_1, \\
    \bY_i &= \bb_i + \bw_i\Psi(\bY_{i-1})+ \be_i, \quad i=2,3,\dots,h,\\
    \bY &= \bb_{h+1} + \bw_{h+1}\Psi(\bY_{h}) + \be_{h+1}, 
    \end{split}
\end{equation}
It can be viewed as a composition of many simple regressions, where  
$\bY_1, \bY_2, \dots, \bY_h$ are latent variables.
For simplicity, we assume that $\be_i \sim N(0, \sigma_i^2 I_{d_i})$ for $i=1,2,\dots, h, h+1$. 
Other distributions can also be assumed 
for $\be_i$'s. For instance, \cite{SunLiang2022kernel} assumed a modified double exponential distribution for $\be_1$ such that the first layer defines a series of support vector regressions (SVRs). 
The parameters $\{\sigma_1^2,\ldots,\sigma_h^2, \sigma_{h+1}^2\}$ control the variation of latent variables $\{\bY_1,\ldots,\bY_h\}$. For classification problems, $\sigma_{h+1}^2$ works as the temperature for the binomial or multinomial distribution formed at the output layer.
% For classification problems, $\sigma_{h+1}^2$ plays the role of temperature for the binomial or multinomial distribution formed at the output layer,  and it works with $\{\sigma_1^2,\ldots,\sigma_h^2\}$ together to control the variation of  $\{\bY_1,\ldots,\bY_h\}$.  For regression problems, this is similar. 
%Refer to Remark \ref{rem:sigma} for more discussions on the setting of $\sigma_i$'s. 
 
% The property of the StoNet as an approximator to the DNN has been studied in \cite{LiangSLiang2022}.  
% %based on the works \cite{LiangSLiang2022} and \cite{SunLiang2022}.  
% %We first show that the StoNet is a valid approximator to a DNN,  i.e., asymptotically they have the same loss function as the training sample size $n$ becomes large.
% Let $\btheta = (\bw_1, \bb_1, \dots, \bw_{h+1}, \bb_{h+1})$ 
% % denote the collection of all 
% % weights of the StoNet (\ref{eq:stonet}), 
% denote the collection of all parameters of the model, and 
% % let $\bTheta$ denote the space of $\btheta$,
% let $\bYmis = (\bY_1, \bY_2, \dots, \bY_{h})$ denote the collection of all latent variables.
% Let $\pi_{\rm DNN}(\bY|\bX, \btheta)$ denote the likelihood function of the DNN (\ref{eq:DNN}), and 
% let $\pi(\bY, \bYmis | \bX, \btheta)$ denote the likelihood function of the StoNet (\ref{eq:stonet}) with the pseudo-complete data $(\bY,\bYmis)$. 
% %Under appropriate assumptions, \cite{LiangSLiang2022} established the following theoretical result: 

The property of the StoNet, as an approximator to the DNN,  has been studied in \cite{LiangSLiang2022}. A brief review for their theory is provided below, which forms the basis of this work.    
Let $\btheta = \{\bw_1, \bb_1, \dots, \bw_{h+1}, \bb_{h+1}\}$ denote the collection of all parameters of the StoNet (\ref{eq:stonet}), 
let $\bTheta$ denote the space of $\btheta$, and
let $\bYmis = \{\bY_1, \bY_2, \dots, \bY_{h}\}$ denote the collection of all latent variables. Let $\pi(\bY, \bYmis | \bX, \btheta)$ denote the joint density function of the pseudo-complete data $(\bY,\bYmis)$, and let $\pi_{\rm DNN}(\bY|\bX, \btheta)$ denote the probability density function of the DNN model (\ref{eq:DNN}).

Regarding the network structure, activation function, and the variance of the latent variables, they made the following assumption:
\begin{assumption} \label{ass:1}
(i) $\bTheta$ is compact, i.e., $\bTheta$ is contained in a $d_{\theta}$-ball centered at 0 with radius $r$; (ii) $\mathbb{E}(\log\pi(\bY, \bYmis|\bX, \btheta))^2<\infty$ for any $\btheta \in \bTheta$; (iii) the activation function $\Psi(\cdot)$ is $c$-Lipschitz continuous for some constant $c$; (iv) the network's depth $h$ and widths $d_l$'s are both allowed to increase with $n$; (v) $\sigma_{h+1}$ is a constant, and 
for every $k\in\{1,2,\ldots,h\}$,
$d_{h+1} (\prod_{i=k+1}^{h} d_i^2) d_k \sigma_{k}^2 \prec \frac{1}{h}$  and $d_k \sigma_k=o(1)$.
\end{assumption}

 Assumption \ref{ass:1}-(i) \& (ii) are regular and generally 
satisfied; Assumption \ref{ass:1}-(iii) allows the StoNet to work with a wide range of Lipschitz continuous activation functions such as {\it tanh}, {\it sigmoid} and {\it ReLU};
Assumption \ref{ass:1}-(v) constrains the magnitude of noise added to each hidden neuron, where
the factor $d_{h+1} (\prod_{i=k+1}^{h} d_i^2) d_k$ can be understood as the amplification factor of the noise $\be_k$ at the output layer. In general, the noise added to the first few hidden layers should be small to prevent large random errors propagated to the output layer. 
 Under slightly weaker conditions than Assumption~\ref{ass:1} (specifically, without requiring
$d_k\sigma_k=o(1)$), they showed that StoNet~\eqref{eq:stonet} and the DNN~\eqref{eq:DNN}
share the same asymptotic energy landscape, as stated in   Lemma~\ref{lemma:2}.
 Imposing a slightly stronger scaling condition on \(\sigma_l\) further enables a refined analysis of feature learning, rather than limiting the analysis to properties of the hidden neuron outputs.

\begin{lemma}[Theorem 2.1  of \cite{LiangSLiang2022}] \label{lemma:2} 
% \citep{LiangSLiang2022} 
Suppose   Assumptions \ref{ass:1} holds. Then 
%and $\pi(\bY,\bYmis|\bX,\btheta)$ is continuous in $\btheta$. Then
\begin{equation} \label{Lemma: StoNet}
 \begin{split}
   &  \quad  \sup_{\btheta\in \Theta}\Big|\frac{1}{n}\sum_{i=1}^n\log\pi(\bY^{(i)}, \bY^{(i)}_{mis}|\bX^{(i)},\btheta)
 -\frac{1}{n}\sum_{i=1}^n\log\pi_{\rm DNN}(\bY^{(i)}|\bX^{(i)},\btheta)\Big|\overset{p}{\rightarrow} 0. 
 %(ii)  & \quad \|\hat{\btheta}_n-\btheta^*\|\overset{p}{\rightarrow}0, \quad \mbox{as $n\to \infty$},
\end{split}
\end{equation}
% where  $\btheta^*=\arg\max_{\btheta\in \Theta}\mathbb{E}(\log\pi_{\rm DNN}(\bY|\bX,\btheta))$ denotes the true parameters of the DNN model as specified in (\ref{eq:DNN}), and 
% $\hat{\btheta}_n = \arg\max_{\btheta\in\Theta}\{\frac{1}{n}\sum_{i=1}^n\log\pi(\bY^{(i)},\bY_{mis}^{(i)}|\bX^{(i)},\btheta)\}$ denotes the maximum likelihood estimator of the StoNet model (\ref{eq:stonet}) with the pseudo-complete data. 
\end{lemma}

Lemma~\ref{lemma:2} shows that the StoNet and the corresponding DNN have
asymptotically equivalent empirical objective functions as \(n\) becomes large.
This loss equivalence provides the basis for transferring theoretical properties between the two models. It can be understood from two complementary perspectives:
\begin{itemize}
\item If the DNN model~\eqref{eq:DNN} is the target model, then a StoNet with the
same network structure and noise levels satisfying Assumption~\ref{ass:1}-(v) provides an asymptotically equivalent surrogate objective. Thus, for large \(n\),
training the StoNet can be viewed as a latent-variable-augmented approach to training the DNN.

\item Conversely, if the StoNet model~\eqref{eq:stonet} is the target model, then the corresponding DNN objective provides an asymptotically equivalent deterministic counterpart. Under appropriate identifiability and argmax conditions, likelihood maximizers of the two objectives therefore share the same population target.
\end{itemize}
As shown later, this asymptotic loss equivalence is a key ingredient for
transferring the operator-norm consistency theory from StoNets to likelihood-based sublinear DNN estimators.

\section{Feature Learning Consistency in sublinear DNNs} \label{Trainingsection}

This section first describes the imputation-regularized optimization (IRO) algorithm \citep{Liang2018missing} for training sublinear StoNets, and then establishes feature-learning consistency for them. 
This result, by Lemma \ref{lemma:2}, 
further implies feature-learning consistency for the sublinear DNNs   
that are trained with an optimization algorithm such as stochastic gradient descent (SGD) or Adam \citep{adam2015}. 

%Assumption \ref{ass:1} requires $\sigma_{i}^2$ to take small values. As $\sigma_{i}^2$ approaches $0$, the StoNet model (\ref{eq:stonet}) will approximate the DNN  model (\ref{eq:DNN}), which ensures that the StoNet model shares a similar approximation power as the DNN model. Actually, when $\sigma_{i}^2 = 0$ for $i = 1, \ldots, h$, the StoNet (\ref{eq:stonet}) is reduced to the DNN (\ref{eq:DNN}); that is, the DNN can be viewed as a special case of the StoNet. In practice, we can treat $\sigma_{i}^2$ as part of the parameters of the StoNet, i.e. setting $\bvartheta = (\bw_1, \bb_1, \sigma_{1}^2, \dots, \bw_{h}, \bb_{h}, \sigma_{h}^2, \bw_{h+1}, \bb_{h+1},\sigma_{h+1}^2)=(\btheta,\sigma_1^2,\ldots,\sigma_h^2,\sigma_{h+1}^2)$ as the full parameter set of the StoNet. In what follows, we will consider the true model $\pi^*(\bY | \bX)$ of the dataset  $D_n=(\bX,\bY)$ to be some StoNet model with a true parameter set $\bvartheta^*$ satisfying Assumption \ref{ass:1}. In Section \ref{Trainingsection}, we present two algorithms for parameter estimation for the  StoNet. 

\vspace{-0.05in}
\subsection{\small The IRO Algorithm} \label{IROsection}
\vspace{-0.05in}

 %\noindent
{\it Notation:} 
% In this section, we rewrite the network depth $h$ as $h_n$,   rewrite the network widths $(p,d_1,\ldots,d_{h+1})$ as $(p_n,d_{1,n},\ldots, d_{h+1,n})$,  and rewrite  the layer-wise variance $\bsigma^2=(\sigma_1^2,\ldots,\sigma_{h+1}^2)$ as 
% $\bsigma_n^2=(\sigma_{1,n}^2,\ldots,\sigma_{h+1,n}^2)$,
% where the subscript $n$ indicates their dependency on the training sample size. 
Let $D_n=\{\mathbb{Y},\mathbb{X}\}$ denote a dataset of $n$ observations, where $\mathbb{Y} \in\mathbb{R}^{n \times d_{h+1}}$ and $\mathbb{X} \in\mathbb{R}^{n \times p}$
 contain the responses and covariates, respectively. Let $\bsigma_n^2=(\sigma_{1}^2,\ldots,\sigma_{h+1}^2)$,
where the dependence of $\sigma_l$ on  $n$ is implicit as implied by 
Assumption \ref{ass:1}.

% For simplicity of theoretical development, we will assume that for a given dataset $D_n=(\bY,\bX)$, the true model 
%  is a StoNet model with $\bsigma_n^2$ being known  and satisfying Assumption \ref{ass:1}-(v).
% By treating the latent variables $\bYmis$ as missing data, the IRO algorithm can be applied to  train the StoNet. The IRO algorithm starts with an initial weight setting $\hat{\btheta}_n^{(0)}$ and then iterates between the imputation  and regularized optimization steps  as shown in Algorithm \ref{IROforstonet}.

\begin{algorithm}[!ht]
\caption{IRO Algorithm for StoNet} 
\label{IROforstonet}
\begin{algorithmic}
\STATE{\bfseries Input}: Dataset $D_n$, total iteration number $T$,  and Monte Carlo step number $t_{MC}$.
 
\STATE{\bfseries Initialization}: Randomly initialize the network parameters $\hat{\btheta}_n^{(0)}$.
%$=(\hat{\btheta}_1^{(0)},\ldots, \hat{\btheta}_{h+1}^{(0)})$.

\FOR{$t=1$ {\bfseries to} $T$} 
\STATE {\bfseries $\bullet$ Imputation step}: For each sample $(\bX^{(i)},\bY^{(i)})$, draw $\bYmis^{(i, t)}$ from $\pi(\bYmis|\bY^{(i)}, \bX^{(i)}$, $\hat{\btheta}_n^{(t-1)}, \bsigma_n^2)$ 
by running the SGLD \citep{welling2011bayesian}, SGHMC \citep{SGHMC2014}, or Metropolis-Hastings algorithm \citep{Metropolis1953,Hastings1970} for $t_{MC}$ steps. 

\STATE {\bfseries $\bullet$ Regularized optimization step}: Based on the pseudo-complete data $(\bY,\bYmis^{(t)},\bX)$, update $\hat{\btheta}_n^{(t-1)}$ by 
   minimizing a penalized loss function, i.e., setting
   \begin{equation} \label{IRoeq1}
    \begin{split}
    \hat{\btheta}_n^{(t)} & =\arg\min_{\btheta}\Big\{  -\frac{1}{n} \sum_{i=1}^n \log\pi(\bY^{(i)},\bYmis^{(i, t)}  \big|\bX^{(i)}, \btheta,\bsigma_n^2) 
     + P_{{\lambda}_n}(\btheta)\Big\},
\end{split}
\end{equation}
   where the penalty  $P_{{\lambda}_n}(\btheta)$ is chosen such that 
   $\hat{\btheta}_n^{(t)}$ forms a consistent estimator  of 
   %\begin{equation} \label{IROeq2}
   \[
   %\small
   \begin{split}
    & \btheta_*^{(t)}=\arg\max_{\btheta}  \mbE_{\btheta_n^{(t-1)}} \log\pi(\bY,\bYmis |\bX,\btheta,\bsigma_n^2) \\
   & = \arg\max_{\btheta} \int \log \pi(\bYmis, \bY|\bX,\btheta,\bsigma_n^2) 
     \pi(\bYmis|\bY,\bX,\btheta_n^{(t-1)},\bsigma_n^2) 
   \pi(\bY|\bX,\btheta^*,\bsigma_n^2)d\bYmis d\bY,
   \end{split}
   \]
  % \end{equation}
   where $\btheta_*^{(t)}$ is called the working true parameter at iteration $t$. 
\ENDFOR 
\STATE{\bfseries Output}: $\hat{\theta}_n^{(T)}$.
 \end{algorithmic}
 \end{algorithm} 

 To facilitate theoretical development, we assume that for a given dataset $D_n$, the true model 
 is a StoNet model with $\bsigma_n^2$ being known  and satisfying Assumption \ref{ass:1}-(v).
By treating the latent variables $\bYmis$ as missing, the IRO algorithm can be applied to  train the StoNet.
%The IRO algorithm starts with an initial weight setting $\hat{\btheta}_n^{(0)}$ and then iterates between the imputation  and regularized optimization steps  as prescribed in Algorithm \ref{IROforstonet}.
 The key to the IRO algorithm is to find an estimator that is uniformly consistent 
 for the working true parameter $\btheta_*^{(t)}$ over all iterations. 
 For high-dimensional problems, as suggested by  \cite{Liang2018missing}, such a uniformly consistent sparse estimator can typically be obtained by minimizing an appropriately penalized loss function as defined in (\ref{IRoeq1}).
 For a sublinear StoNet, such a penalty term is unnecessary. Therefore, 
 we set  
%\begin{equation} \label{eq:penalty}
 $P_{\lambda_n}(\btheta)=0$ for $\btheta \in \Theta$.
 %\end{equation}
 Under this setting, solving (\ref{IRoeq1}) corresponds to solving a series of  
 linear regressions by noting that the joint distribution $\pi(\bYmis,\bY|\bX,\btheta,\bsigma_n^2)$ can be decomposed in a Markovian structure:
 \begin{equation} \label{eq:decomp}
 \begin{split}
&\pi(\bYmis,\bY|\bX,\btheta,\bsigma_n^2)=\pi(\bY|\bY_h,\btheta,\bsigma_n^2) \pi(\bY_h|\bY_{h-1},\btheta,\bsigma_n^2) \cdots \pi(\bY_1|\bX,\btheta,\bsigma_n^2),
\end{split}
 \end{equation}
 and, furthermore, the components of $\bY_i \in \mathbb{R}^{d_i}$ are mutually independent conditional on $\bY_{i-1}$ for $i=1,2,\ldots,h+1$. 
 For notational simplicity, we let $\bY_0=\bX$ and $\bY_{h+1}=\bY$. 

We note that the IRO algorithm is reduced to the stochastic EM algorithm \citep{CeleuxDiebolt1995,Nielsen2000} when no penalty is used in (\ref{IRoeq1}). However,  the theoretical framework established in \cite{Liang2018missing} still works for the sublinear StoNet models. 
For this reason, Algorithm \ref{IROforstonet} is still referred to as an IRO algorithm.

\subsection{Feature Learning Consistency} \label{sect:theory}

For all theoretical results in this paper,  the proofs are presented in the Appendix.
Let $\lambda_{\min}(A)$ and 
$\lambda_{\max}(A)$ denote, respectively, the minimum and maximum eigenvalues of the matrix $A$. 
For the inputs and network structure, we make the following assumption: 

\begin{assumption} \label{ass:3} 
(i) The network architecture satisfies the condition given in (\ref{eq:sublinear});
%$d_{l} \leq d_{l-1}$ and  
%$d_l  \prec n$ for $l=0,1,\ldots,h+1$;
%where $d_0=p$ denotes the dimension of the input variable, and $d_{h+1}$ denotes the dimension of the output variable;  
(ii) $\bX \in [0,1]^p$ (i.e., in a bounded space); additionally, there exists 
a constant $\kappa_{\min}>0$ such that
 $\lambda_{\min}(\bSigma_{0}) \geq  \kappa_{\min}$, where  $\bSigma_0$
 denotes the covariance matrix of $\bX$.
 \end{assumption}
 
%the input variable $\bX$ is bounded, 
%there exist constants $0< \kappa_{\min} < \kappa_{\max} <\infty$  such that 
%$\liminf_{n \to \infty} \lambda_{\min}(\bSigma_{0,n}) \geq  \kappa_{\min}$ and $\liminf_{n \to \infty} \lambda_{\max}(\bSigma_{0,n}) \leq  \kappa_{\max}$, 
%$\kappa_{\min} \leq  \lambda_{\min}(\bSigma_{0}) \leq \lambda_{\max}(\bSigma_{0}) \leq  \kappa_{\max}$, 
%where $\bSigma_{0}$ denotes the covariance matrix of  $\bX$.
% (iii) there exists constants %$0< \tau_{\min} < 
% $0<\tau_{\max}<\infty$ such that 
% %$\tau_{\min} \leq \lambda_{\min} (\bw_l \bw_l^\top) \leq 
% $\lambda_{\max}(\bw_l \bw_l^\top) \leq  \tau_{\max}$ for any  $l=1,2,\ldots,h+1$, where $\bw_l\in \mathbb{R}^{d_l \times d_{l-1}}$ is the parameter matrix of the DNN at layer $l$; 
%(iii) the network's depth and widths $d_l$'s, for $l=0,1,2,\ldots,h+1$, are all allowed to increase with the sample size $n$. 
%\end{assumption} 

Assumption \ref{ass:3}-(i) restricts the architecture of the DNN, while 
its universal approximation ability 
can still be established for some  
classes of functions under the large sample regime, detailed in 
Section \ref{sect:approx}. 
% of the DNN can still be maintained under such a setting, see e.g.,  
% %\cite{Hanin2017MinWidth}, 
% \cite{Kidger2020UniversalAW}
% %\cite{Park2020MinimumWF}, and 
% and \cite{Kim2023MinimumWF}). 
Assumption \ref{ass:3}-(ii)  is regular, which implies that the 
eigenvalues of the covariance matrix of $\bX$ are uniformly bounded across sample sizes, i.e., 
$\kappa_{\min} \leq  \lambda_{\min}(\bSigma_{0}) \leq \lambda_{\max}(\bSigma_{0}) \leq  \kappa_{\max}$ 
for some constant $\kappa_{\max}>0$ and all sample size $n>0$.  
%For simplicity of notation, the dependence of $\Sigma_0$ on $n$ is depressed. 
%This assumption has often been used in theoretical studies for DNNs.  

% \begin{lemma} \label{lemma:eigen} 
% Consider a random matrix $\bU\in \mathbb{R}^{n\times d}$ with $n \geq d$.  Suppose that the 
% eigenvalues of $\bU^\top \bU$ are upper bounded, i.e., $\lambda_{\max}(\bU^\top\bU) \leq \kappa_{\max}$ for 
% some constant $\kappa_{\max}>0$. 
% Let $\Psi(\bU)$ denote an elementwise transformation of $\bU$. Then 
% %\begin{equation} \label{eigenU:eq}
% $\lambda_{\max}\left( (\Psi(\bU))^\top (\Psi(\bU)) \right) \leq \kappa_{\max}$
% %\end{equation}
% for the {\it tanh}, {\it sigmoid} and ReLU transformations. 
% \end{lemma} 
% \begin{proof}  
% For ReLU, (\ref{eigenU:eq}) follows from Lemma 5 of \cite{Dittmer2018SingularVF}. For 
% {\it tanh} and {\it sigmoid}, since they are Lipschitz continuous with a Lipschitz constant of 1, Lemma 5 of  \cite{Dittmer2018SingularVF} also applies. 
% \end{proof}

% By matrix norm, we know 
% \[
% \sum_{j=1}^p \lambda_j  \left( (\Psi(\bX))^\top (\Psi(\bX)) \right) = trace \left( (\Psi(\bX))^\top (\Psi(\bX)) \right) =\sum_{i=1}^n \sum_{j=1}^p |\Psi(x_{ij})|^2 
% \]
% however, 
% \[
% \sum_{i=1}^n \sum_{j=1}^p |\Psi(x_{ij})|^2  \leq \sum_{i=1}^n \sum_{j=1}^p |x_{ij}|^2=\sum_{j=1}^p  \lambda_j  \left(\bX^\top \bX \right)
% \]
%\end{proof}

\begin{assumption}[Preactivation regularity]
\label{ass:preactivation-regularity}
% For each hidden layer \(l=1,\ldots,h\), let
% \[
% \bY_l^{(t)}
% =
% \widetilde \bY_l^{(t)}+\be_l,
% \qquad
% \be_l\sim N(0,\sigma_l^2 I_{d_l}),
% \]
% where \(\widetilde \bY_l^{(t)}\) denotes the preactivation vector at layer \(l\)
% and IRO iteration \(t\). Write \(\widetilde Y_{l,k}^{(t)}\) for its \(k\)-th
% coordinate.
Let \(\widetilde Y_{l,k}^{(t)}\) denote the preactivation input to 
the neuron $k$ of layer $l$ at iteration $t$ of the IRO algorithm. 
For tanh and sigmoid activations, assume that the preactivations have uniformly
bounded coordinatewise second moments; that is, there exists a constant
\(C_{\widetilde Y}<\infty\), independent of \(n,l,k\), and the IRO iteration
\(t\), such that
\begin{equation} \label{eq:bounded-second-moment}
\sup_{l,k,t}
\mathbb E\left\{
\left(\widetilde Y_{l,k}^{(t)}\right)^2
\right\}
\le
C_{\widetilde Y}.
\end{equation}

For the ReLU activation, assume that each hidden neuron has a nonnegligible
active-region probability. Specifically, there exist constants
\(s\in\mathbb R\) and \(\pi_0>0\), independent of \(n,l,k\), and the IRO
iteration \(t\), such that
\begin{equation} \label{eq:active-region}
\mathbb P\left(
\widetilde Y_{l,k}^{(t)}\ge s\sigma_l
\right)
\ge
\pi_0,
\qquad
l=1,\ldots,h,\quad k=1,\ldots,d_l, \quad t=1,\ldots,T. 
\end{equation}
\end{assumption}

The first part of Assumption~\ref{ass:preactivation-regularity} prevents the smooth activations, such as tanh and sigmoid, from being uniformly saturated. A sufficient condition for \eqref{eq:bounded-second-moment} is provided in Lemma~\ref{lem:preactivation-coordinate-bound}; thus, we state \eqref{eq:bounded-second-moment} explicitly to keep the subsequent covariance lower-bound argument transparent.  
The second part is specific to the ReLU activation. It rules out asymptotically inactive, or removable, ReLU units by requiring the scaled preactivation \(\widetilde Y_{l,k}^{(t)}/\sigma_l\) to have nonvanishing probability mass above a fixed finite threshold. This prevents the preactivation from drifting to the far negative region with probability tending to one. 
% Consequently, the ReLU variance factor
% \[
% g\left(\frac{\widetilde Y_{l,k}^{(t)}}{\sigma_l}\right),
% \qquad
% g(u)=\operatorname{Var}{(u+Z)_+},
% \]
% is bounded below with nonvanishing probability, which yields a variance contribution of order \(\sigma_l^2\) after the ReLU transformation.

 \begin{lemma} \label{eigen_hidden_layer} For any layer $l \in \{1, 2,\ldots,h\}$, let $\bSigma_{l}^{(t)} \in \mathbb{R}^{d_l\times d_l}$ denote the covariance matrix of the covariates of the regressions, which are formed for the neurons of layer $l+1$ 
 at iteration $t$ of Algorithm \ref{IROforstonet}.
%and let $s_{L,n}^{(t)}$ denote the number of hidden units at hidden layer $L$ of the true sparse StoNet model at iteration $t$. 
If Assumptions \ref{ass:1}-\ref{ass:preactivation-regularity} hold, then there exist constants $c>0$ such that 
$\lambda_{\min}(\bSigma_{l}^{(t)}) \geq c \sigma_l^2$ 
for any iteration $t$. 
% , $c'>0$, and $\tilde{\kappa}>0$ such that  
% $c \sigma_{l}^2 \leq \lambda_{\min}(\bSigma_{l}^{(t)})  \leq 
%  \lambda_{\max}(\bSigma_{l}^{(t)}) \leq \tilde{\kappa}+ c' \sigma_l+c' \sigma_{l}^2$ holds for any iteration $t$. 
\end{lemma}
The proof of Lemma \ref{eigen_hidden_layer} is given in Appendix \ref{app:2}. 
To study properties of the coefficient estimators for  
the regressions formed in the StoNet, we introduce the following two 
lemmas, one for linear regression and the other for multinomial 
logistic regression. 

% \begin{lemma} \textcolor{black}{(\citealp[Theorem 7.3]{Rencher2007LinearMI}; \citealp[Chapter 2]{Golub2013MatrixC})}
% \label{lemma:linearR}
% Consider a linear regression with $n$ observations: 
% $\mathbb{Y}=\mathbb{X} \bbeta+\sigma \bepsilon$,
%  where $\mathbb{Y} \in \mathbb{R}^n$, $\mathbb{X}\in \mathbb{R}^{n\times p}$, $\bbeta\in \mathbb{R}^{p}$, 
%  $\sigma>0$, $\bepsilon \sim \mathcal{N}(0,I_{n})$, 
%  and $p < n$. If  $\lambda_{\rm \min}(\mathbb{X}^\top \mathbb{X}) \geq n \kappa_{\min}$,  then 
%  %\begin{equation} \label{LinearEst1}
%   $E\|\hat{\bbeta}-\bbeta^*\|^2 = \sigma^2 \|\mathbb{X}^\top \mathbb{X}\|^{-1} 
%   \leq \frac{ \sigma^2}{n \kappa_{\min}}$,  
%  %\end{equation} 
% where $\hat{\bbeta}$ 
% %=\arg\max_{\bbeta} \{ \log\pi(\by|\bX,\bbeta,\sigma)+ \log \pi(\bbeta|\sigma,\sigma_0,\sigma_1,\rho)\}$ 
% denotes the ordinary least square (OLS) estimator of $\bbeta$.  
%  \end{lemma}

{\color{black} 
 \begin{lemma} (\citealp[Theorem 7.3]{Rencher2007LinearMI}; \citealp[Chapter 2]{Golub2013MatrixC})
\label{lemma:linearR}
Consider the linear regression model
\[
\mathbb Y=\mathbb X\bbeta^*+\sigma\bepsilon,
\qquad
\bepsilon\sim N(0,I_n),
\]
where \(\mathbb Y\in\mathbb R^n\), \(\mathbb X\in\mathbb R^{n\times p}\), 
\(\bbeta^*\in\mathbb R^p\), \(\sigma>0\), and \(p<n\). Let \(\hat{\bbeta}\) denote the OLS estimator of $\bbeta^*$. If
$\lambda_{\min}(\mathbb X^\top\mathbb X)\ge n\kappa_{\min}$, then
\[
\left\|
E\left[
(\hat{\bbeta}-\bbeta^*)(\hat{\bbeta}-\bbeta^*)^\top
\mid \mathbb X
\right]
\right\|_2
\le
\frac{\sigma^2}{n\kappa_{\min}}.
\]
Equivalently, for every unit vector \(\bu\in\mathbb R^p\),
$E\left[
\{\bu^\top(\hat{\bbeta}-\bbeta^*)\}^2
\mid \mathbb X
\right]
\le
\frac{\sigma^2}{n\kappa_{\min}}$.
\end{lemma}
}

\begin{lemma} \label{lemma:logistic}
Consider a multinomial logistic regression, which contains $m+1$ classes and $n$ observations. 
Assume that (i) \textcolor{black}{each covariate} in  $\mathbb{X}\in \mathbb{R}^{n\times p}$ is normally distributed with the variance decreasing with $n$ at a rate of $O(n^{-\alpha})$ for some $\alpha>0$; (ii)  
 $p\leq n$;  and 
 (iii) the design matrix $\mathbb{X}$ is nondegenerate in the sense 
 that there exists a constant \(\kappa_{\min}>0\) such that
$\lambda_{\min}(X^\top X)\ge n\kappa_{\min}$
with probability tending to one.
%  (iii) the eigenvalues  of $\mathbb{X}^\top\mathbb{X}$ are positive and bounded, i.e., there exist constants $\kappa_{\min}>0$ and $\kappa_{\max}>0$ such that 
% $n\kappa_{\min} \leq \lambda_i(\mathbb{X}^\top\mathbb{X}) \leq  n \kappa_{\max}$ holds for any 
% $i\in \{1,2,\ldots, p\}$.
%\textcolor{black}{with probability converging to 1.} 
Let $\vec{\bB}=(\bbeta_1^\top, \bbeta_2^\top,\cdots,\bbeta_m^\top)^\top$ denote the vector of true regression coefficients of the model, where $\bbeta_i \in \mathbb{R}^{p}$, and 
let  $\widehat{\vec{\bB}}$ denote the maximum likelihood estimator (MLE) of $\vec{\bB}$.
 Then there exists a constant $\nu_0$ such that 
\[
\left\|
\mathbb E_{\mathbb Y\mid \mathbb X}
\left[
(\widehat{\vec{\bB}}-\vec{\bB})
(\widehat{\vec{\bB}}-\vec{\bB})^\top
\right]
\right\|_2
\le
\frac{1}{n\nu_0\kappa_{\min}}
\]
with probability tending to one, where the expectation is taken with respect to the conditional distribution \(\pi(\mathbb Y\mid\mathbb X)\).
%\begin{equation} \label{LogistEst2}
% $\mathbb{E}_{\mathbb{Y}|\mathbb{X}}   
% \|\widehat{\vec{\bB}}-\vec{\bB}\|_2 \leq \frac{1}{n \nu_0\kappa_{\min}}$ holds
% with probability converging to 1,
% \textcolor{black}{where the expectation is taken with respect to the conditional distribution $\pi(\mathbb{Y}|\mathbb{X})$}. 
%\end{equation}
%where $\widehat{\vec{\bB}}$ denotes maximum likelihood estimator (MLE) of $\vec{\bB}$.
\end{lemma}

The proof of Lemma \ref{lemma:logistic} is given in Appendix \ref{app:3}. 
\textcolor{black}{The conditions of Lemma \ref{lemma:logistic} are specifically tailored to the auxiliary StoNet used in the proof. In this StoNet, the output of each hidden neuron is modeled as a Gaussian random variable with variance chosen by the user. As such, the conditions, including the rate at which the variance decreases, can be satisfied.}
Specifically, condition (i) 
%of Lemma \ref{lemma:logistic} 
aligns with Assumption \ref{ass:1}-(v), where
%the depth and width of the StoNet are allowed to grow with $n$ and thus 
the variance of the random noise  added to each hidden neuron tends to decrease as $n$ increases; condition (ii)
 aligns with Assumption \ref{ass:3}-(i);  and condition (iii)
 aligns  with Lemma \ref{eigen_hidden_layer}. 

 \begin{lemma}[One iteration layerwise operator-norm error]
\label{lemma:one-step-op}
Suppose Assumptions~\ref{ass:1}--\ref{ass:preactivation-regularity} hold. 
For each layer \(l=1,\ldots,h+1\), let
$\bar\bW_l=(\bb_l,\bw_l)
\in \mathbb R^{d_l\times(d_{l-1}+1)}$
denote the augmented coefficient matrix including the bias term, and define
$\Delta_l^{(t)}
=
\widehat{\bar\bW}_{l,n}^{(t)}-\bar\bW_{l,*}^{(t)}$.
Let \(p_l=d_{l-1}+1\), so that
$\Delta_l^{(t)}\in\mathbb R^{d_l\times p_l}$.
Suppose that, conditional on the imputed covariates used in the \(l\)th layer regression, the row errors
\[
\Delta_{lk}^{(t)}
=
\hat{\bbeta}_{lk}^{(t)}-\bbeta_{lk,*}^{(t)}
\in\mathbb R^{p_l},
\qquad k=1,\ldots,d_l,
\]
are independent mean-zero sub-Gaussian vectors satisfying, for every unit vector
\(\bv\in\mathbb R^{p_l}\),
\[
\left\|
\bv^\top \Delta_{lk}^{(t)}
\right\|_{\psi_2}
\le
C
\frac{\sigma_l}{\sigma_{l-1}\sqrt n},
\]
where $\|\cdot\|_{\psi_2}$ denotes the sub-Gaussian Orlicz norm and 
$C$ is a constant.
Then
\[
\left\|
\widehat{\bar\bW}_{l,n}^{(t)}-\bar\bW_{l,*}^{(t)}
\right\|_{\rm op}
=
O_p\left\{
\frac{\sigma_l}{\sigma_{l-1}}
\sqrt{\frac{d_l+d_{l-1}+1}{n}}
\right\},
\]
where $\|\bA\|_{\rm op}
=\lambda_{\max}\{(\bA^\top\bA)^{1/2}\}$.
Consequently, if
\begin{equation} \label{eq:an}
a_n^2 :=
\frac{1}{n}
\sum_{l=1}^{h+1}
(d_l+d_{l-1}+1)
\frac{\sigma_l^2}{\sigma_{l-1}^2}
\rightarrow0, \qquad \mbox{as $n\to \infty$},
\end{equation}
then the layerwise operator-norm error  
\begin{equation}\label{eq:an2}
\sum_{l=1}^{h+1}
\left\|
\widehat{\bar\bW}_{l,n}^{(t)}-\bar\bW_{l,*}^{(t)}
\right\|_{\rm op}^2
= O_p(a_n^2) =o_p(1).
\end{equation}
\end{lemma}

\begin{remark}
The row-wise second-moment operator bound alone would not yield the   operator-norm rate  in (\ref{eq:an2}). 
That rate is obtained using the sub-Gaussian
structure of the layerwise regression estimators,  see the proof of Lemma \ref{lemma:one-step-op} in Appendix \ref{proof:lem5}. 
In the present StoNet setting, the Gaussian layerwise regression model
and the local asymptotic normal approximation for the logistic layer justify the sub-Gaussian operator-norm formulation used above. 
%If only second moments were available, a weaker Frobenius-type bound with a larger dimension factor would result.
\end{remark}

\begin{remark}[Hidden-layer noise levels] \label{rem:noise}
To satisfy (\ref{eq:an}), we suggest to 
choose the hidden-layer noise levels to increase toward the output layer. 
Let
\[
S_{l,n}=d_l+d_{l-1}+1,\qquad l=1,\ldots,h+1.
\]
Define 
\[
T_n=\sum_{l=2}^{h+1}S_{l,n}, \qquad 
B_{l,n}
=
d_{h+1}
\left(\prod_{i=l+1}^{h}d_i^2\right)d_l,
\qquad l=1,\ldots,h.
\]
Suppose that there exists a sequence \(r_n\to\infty\) such that
\begin{equation} \label{eq:noise-cond}
\max_{1\le l\le h}
\left\{
(hB_{l,n})^{1/(h+1-l)},
d_l^{2/(h+1-l)}
\right\}
\prec r_n
\prec
\frac{n}{T_n}.
\end{equation}
Then, with fixed \(\sigma_{h+1}=O(1)\), choose
$\sigma_l^2
=
\sigma_{h+1}^2 r_n^{-(h+1-l)}$ for 
$l=1,\ldots,h$.
This choice gives $\sigma_1<\sigma_2<\cdots<\sigma_h<\sigma_{h+1}$ and
$\frac{\sigma_l^2}{\sigma_{l-1}^2}=r_n$ for $l=2,\ldots,h+1$.
Therefore,
\[
\sum_{l=2}^{h+1}
\frac{S_{l,n}}{n}
\frac{\sigma_l^2}{\sigma_{l-1}^2}
=
\frac{r_nT_n}{n}
=o(1).
\]
The first term in \(a_n^2\) also satisfies
\[
\frac{S_{1,n}}{n}
\frac{\sigma_1^2}{\sigma_0^2}
=
\frac{S_{1,n}}{n}
\frac{\sigma_{h+1}^2r_n^{-h}}{\sigma_0^2}
=o(1),
\]
where \(\sigma_0^2=\kappa_{\min}\) is fixed. Hence
$a_n^2
=
\frac{1}{n}
\sum_{l=1}^{h+1}
S_{l,n}
\frac{\sigma_l^2}{\sigma_{l-1}^2}
\to 0 $.

Moreover, for each \(l=1,\ldots,h\),
\[
B_{l,n}\sigma_l^2
=
B_{l,n}\sigma_{h+1}^2r_n^{-(h+1-l)}
=
o\left(\frac1h\right),
\]
because $(hB_{l,n})^{1/(h+1-l)}\prec r_n$. Similarly,
\[
d_l\sigma_l
=
d_l\sigma_{h+1}r_n^{-(h+1-l)/2}
=o(1),
\]
because $d_l^{2/(h+1-l)}\prec r_n$.
Thus Assumption~\ref{ass:1}-(v) is also satisfied. 

Additionally, we note that \eqref{eq:noise-cond} controls the growth rates of the layer widths, but it does not preclude the network from being over-parameterized in the usual parameter-count sense; see Remark~\ref{rem:architecture} in the Appendix for an instance of architecture design for sublinear DNNs.
\end{remark}

Regarding the energy landscape of the DNN,  we make Assumption \ref{ass:2}.  Let 
\[
Q^*(\btheta)=\mathbb{E}(\log\pi_{\rm DNN}(\bY|\bX,\btheta)),
\]
be the expected loss, taken with respect to the joint distribution $\pi(\bX,\bY)$, of the DNN.
By Assumption \ref{ass:1}-(i)\&(ii) and the law of large numbers,
\begin{equation}\label{eq:sameloss2}
    \frac{1}{n}\sum_{i=1}^n\log\pi_{\rm DNN}(\bY^{(i)}|\bX^{(i)},\btheta)-Q^*(\btheta)\overset{p}{\rightarrow} 0
\end{equation}
holds uniformly over $\Theta$, where the superscript $i$ indexes 
observations of the dataset. 
%Then they assume that $Q^*(\btheta)$ satisfies the following regularity conditions:

\begin{assumption}\label{ass:2}
(i) The expected loss function $Q^*(\btheta)$ is continuous in $\btheta$ and uniquely maximized \textcolor{black}{(up to loss-invariant transformations)}
at $\btheta^*$;
(ii) for any $\epsilon>0$, $sup_{\btheta\in\Theta\backslash B(\epsilon)}Q^*(\btheta)$ exists, where $B(\epsilon)=\{\btheta:d_{\rm op}(\btheta,\btheta^*)<\epsilon\}$, and $\delta_{\epsilon}=Q^*(\btheta^*)-sup_{\btheta\in\Theta\backslash B(\epsilon)}Q^*(\btheta)>0$.
\end{assumption}

Assumption \ref{ass:2} restricts the shape of $Q^*(\btheta)$ around the global maximizer,  ensuring that it is neither discontinuous nor too flat. Given nonidentifiability of the neural network model, 
Assumption \ref{ass:2} implicitly assumes that each $\btheta$ is unique up to loss-invariant transformations, such as reordering the hidden neurons within the same layer or simultaneously altering the signs or scales of certain weights and biases; see e.g., \cite{Liang2018BNN} and \cite{SunSLiang2021} for further discussions. 
Alternatively, the optimal solutions can be considered as belonging to an equivalent class, subject to appropriate loss-invariant transformations, with the uniqueness assumption applying to this equivalent class. 

Further, let's consider the IRO parameter update mapping $M(\btheta)$ as defined in (\ref{IRoeq1}), i.e.,  
\[
M(\btheta)=\arg\max_{\tilde{\btheta}}  \mathbb{E}_{\btheta} \log\pi(\bY,\bYmis|\bX,\tilde{\btheta}).
\]
As argued in \cite{Liang2018missing} and \cite{Nielsen2000}, 
it is reasonable to assume this  
 mapping is contractive. A recursive application of the mapping, i.e., setting
$\hat{\btheta}_n^{(t+1)}=\btheta_*^{(t+1)}=M(\hat{\btheta}_n^{(t)})$, leads to a monotone increase of the target expectations $\mathbb{E}_{\hat{\btheta}_n^{(t)}} \log \pi(\bY, \bYmis|\bX,\hat{\btheta}_n^{(t+1)})$ for $t=1,2,\ldots,T$. 
 
\begin{assumption}[Local stability of the IRO parameter update mapping]
\label{ass:4}
Let \(M(\btheta)\) denote the IRO parameter update mapping. Let
$\Theta^* :=\arg\max_{\btheta\in\Theta}Q^*(\btheta)$
denote the set of population maximizers, where parameters in \(\Theta^*\) are
identified up to loss-invariant transformations. Suppose that
$M(\btheta^*)=\btheta^*$
for some representative \(\btheta^*\in\Theta^*\).

Write
$\btheta=(\bar\bW_1,\ldots,\bar\bW_{h+1})$ and $\bar\bW_l=(\bb_l,\bW_l)$.
For a layerwise perturbation
$\Delta=(\Delta_1,\ldots,\Delta_{h+1})$,
where each \(\Delta_l\) has the same dimension as \(\bar\bW_l\), define the $\ell_2$-aggregated layerwise operator norm
\begin{equation} \label{eq:distance}
\|\Delta\|_{d_{\rm op}}
:=\left(
\sum_{l=1}^{h+1}
\|\Delta_l\|_{\rm op}^2
\right)^{1/2}.
\end{equation}
For the derivative of \(M\), write
\[
DM(\btheta)[\Delta]
:=
\left.
\frac{d}{d\epsilon}
M(\btheta+\epsilon\Delta)
\right|_{\epsilon=0},
\]
and define \(\|DM(\btheta)[\Delta]\|_{d_{\rm op}}\) by applying the same
layerwise norm to the resulting layerwise perturbation.
There exist a neighborhood \(U(\btheta^*)\) of \(\btheta^*\) such that \(M\) is differentiable on \(U(\btheta^*)\) and a constant
\(0<\lambda^*<1\) such that 
%and for almost every training dataset $D_n$, 
\[
\rho_{\rm op}(\btheta)
:=
\sup_{\Delta\neq 0}
\frac{
\|DM(\btheta)[\Delta]\|_{d_{\rm op}}
}{
\|\Delta\|_{d_{\rm op}}
}
\le
\lambda^*,
\qquad
\btheta\in U(\btheta^*).
\]
\end{assumption}

% The norm \(\|\cdot\|_{\rm lay}\) is the linear-space counterpart of the
% metric \(d_{\rm op}\). In particular, if
% \[
% \Delta=\btheta-\btheta',
% \]
% then
% \[
% \|\Delta\|_{\rm lay}
% =
% d_{\rm op}(\btheta,\btheta').
% \]
% Thus \(\rho_{\rm op}(\btheta)\) is the local Lipschitz constant of the
% population IRO map \(M\) when both the input and output perturbations are
% measured in the layerwise operator-norm metric. The condition
% \[
% \rho_{\rm op}(\btheta)\le \lambda^*<1
% \]
% therefore means that, locally around \(\btheta^*\), one population IRO update
% shrinks layerwise operator-norm perturbations by a factor strictly smaller than
% one.

% This formulation is tailored to the consistency result in
% Lemma~\ref{lemma:IRO-op}, where the estimation error is measured by
% \(d_{\rm op}\) rather than by the Euclidean norm of the vectorized parameter.
% It also avoids imposing Frobenius-norm or Euclidean-norm contraction, which
% would be stronger than needed when the layer widths grow with \(n\). Indeed,
% operator-norm consistency is the relevant notion for the eigenvalue and
% feature-learning analysis, because the neural features are defined through
% matrices such as \(\bW_l^\top\bW_l\), whose perturbations are naturally
% controlled by layerwise operator norms.

This assumption is essentially a local stability condition for the IRO dynamics. Under nondegenerate imputation noise, if this stability condition is violated, the stochastic perturbations may excite unstable directions and cause the iterates to move away from the target fixed point. Thus, from a practical point of view, the condition is closely related to the observed long-run stability and convergence of the algorithm. Although this condition may not be directly verifiable in complex models, stable convergence across sufficiently long runs and multiple initializations provides practical evidence supporting its validity. Consequently, we obtain the following IRO estimation consistency in an operator norm metric. 

% The assumption can be viewed as a local stability condition for the IRO dynamics. Because stochastic IRO repeatedly introduces imputation noise, any locally expanding direction of the population update map may amplify this noise and prevent the iterates from remaining in a tight neighborhood of the fixed point. Thus, under nondegenerate imputation noise, a natural sufficient condition for stable local behavior is that the layerwise operator-norm derivative of the update map be uniformly smaller than one, namely
% $\rho_{\rm op}(\btheta)<1$ for $\btheta\in U(\btheta^*)$.

\begin{lemma}[IRO estimation consistency in an operator-norm metric]
\label{lemma:IRO-op}
Suppose Assumptions~\ref{ass:1}--\ref{ass:4} hold and \(a_n\to0\), where \(a_n\) is defined in \eqref{eq:an}.   
Define the $\ell_2$-aggregated layerwise operator norm metric
$d_{\rm op}(\btheta,\btheta'):= \|\btheta-\btheta'\|_{d_{\rm op}}$. 
%Assume that the initialization lies in the local basin of attraction of \(\btheta^*\).
%and that, for the iteration sequence under consideration, $P\{\widehat\btheta_n^{(t)}\in U(\btheta^*),\ t=0,\ldots,T\}\to 1$.
Then, for the IRO estimator
\(\widehat{\btheta}_n^{(t)}\),
\[
d_{\rm op}\!\left(\widehat{\btheta}_n^{(t)},\btheta^*\right)
=
O_p\!\left((\lambda^*)^t\right)+O_p(a_n),
\]
where \(0<\lambda^*<1\) is the local contraction constant as defined 
in Assumption \ref{ass:4}. Consequently,
\[
d_{\rm op}\!\left(\widehat{\btheta}_n^{(t)},\btheta^*\right)
\stackrel{p}{\rightarrow}0, \quad 
\mbox{as \(t\to\infty\) and \(n\to\infty\)}.
\]
\end{lemma}

% \begin{proof}
% Let \(M(\btheta)\) denote the population IRO update map. By
% Lemma~\ref{lemma:one-step-op},
% \[
% d_{\rm op}\!\left(
% \widehat{\btheta}_n^{(t)},M(\widehat{\btheta}_n^{(t-1)})
% \right)
% =
% O_p(a_n).
% \]
% By the local contraction condition in Assumption~\ref{ass:4},
% \[
% d_{\rm op}\!\left(
% M(\widehat{\btheta}_n^{(t-1)}),M(\btheta^*)
% \right)
% \le
% \lambda^*
% d_{\rm op}\!\left(
% \widehat{\btheta}_n^{(t-1)},\btheta^*
% \right).
% \]
% Since \(M(\btheta^*)=\btheta^*\), we obtain
% \[
% d_{\rm op}\!\left(\widehat{\btheta}_n^{(t)},\btheta^*\right)
% \le
% O_p(a_n)
% +
% \lambda^*
% d_{\rm op}\!\left(\widehat{\btheta}_n^{(t-1)},\btheta^*\right).
% \]
% Iterating the recursion gives
% \[
% d_{\rm op}\!\left(\widehat{\btheta}_n^{(t)},\btheta^*\right)
% \le
% (\lambda^*)^t
% d_{\rm op}\!\left(\widehat{\btheta}_n^{(0)},\btheta^*\right)
% +
% \frac{O_p(a_n)}{1-\lambda^*}.
% \]
% Thus
% \[
% d_{\rm op}\!\left(\widehat{\btheta}_n^{(t)},\btheta^*\right)
% =
% O_p\!\left((\lambda^*)^t\right)+O_p(a_n),
% \]
% which converges to zero as \(t\to\infty\) and \(n\to\infty\).
% \end{proof}

%\begin{remark}[Operator-norm versus Euclidean consistency]
Lemma~\ref{lemma:IRO-op}, whose proof is given in Appendix \ref{proof:lem6},  establishes consistency of the IRO estimator in an 
operator norm metric. This notion should be distinguished from consistency in the usual
Euclidean norm of the vectorized parameter. For a matrix \(\bA\),
\[
\|\bA\|_{\rm op}
\le
\|\bA\|_{F}
\le
\sqrt{\operatorname{rank}(\bA)}\,\|\bA\|_{\rm op},
\]
where \(\|\cdot\|_F\) denotes the Frobenius norm. Since the Euclidean norm of the
vectorized coefficient matrix is exactly its Frobenius norm, Euclidean parameter
consistency is stronger than operator-norm consistency when the layer widths grow with
\(n\).
More explicitly, if we define
$d_F^2(\btheta,\btheta')
=\sum_{l=1}^{h+1} \|\bar\bW_l-\bar\bW_l'\|_F^2$,
then $d_{\rm op}(\btheta,\btheta')
\le
d_F(\btheta,\btheta')$.
Thus Euclidean/Frobenius consistency implies operator-norm consistency. The converse,
however, need not hold when the dimensions \(d_l\) increase with \(n\), because the
factor \(\sqrt{\operatorname{rank}(\bar\bW_l-\bar\bW_l')}\) may diverge. In this sense,
operator-norm consistency is a weaker but more appropriate notion for studying neural
features, as shown below, in growing neural networks.

\begin{assumption}[Eigen-gap condition for neural features] \label{ass:5}
For each layer \(l=1,\ldots,h+1\), let
\[
\bA_l^*=\bW_l^{*\top}\bW_l^*
\]
and let its eigenvalues be ordered as
$\lambda_1^{(l)}\ge \lambda_2^{(l)}\ge \cdots \ge \lambda_{d_{l-1}}^{(l)}\ge 0$.
Suppose that the top \(r_l\)-dimensional neural-feature subspace is the object of interest, and there exists a constant \(\delta_l>0\) such that
$\lambda_{r_l}^{(l)}-\lambda_{r_l+1}^{(l)}\ge \delta_l$.
\end{assumption}

\begin{theorem}[Eigenvalue and feature-learning consistency of sublinear DNNs]
\label{thm:new1}
Suppose Assumptions~\ref{ass:1}--\ref{ass:5} hold.
%where Assumption~\ref{ass:4} is understood as a local contraction condition for the population IRO map, and Assumption~\ref{ass:5} is the eigengap condition for the neural-feature subspaces of interest.
% For each layer \(l=1,\ldots,h+1\), define the augmented coefficient matrix
% \[
% \bar\bW_l=(\bb_l,\bW_l)\in\mathbb R^{d_l\times(d_{l-1}+1)}.
% \]
% For two parameter values \(\btheta\) and \(\btheta'\), define the layerwise operator-norm distance
% \[
% d_{\rm op}^2(\btheta,\btheta')
% =
% \sum_{l=1}^{h+1}
% \|\bar\bW_l-\bar\bW_l'\|_{\rm op}^2,
% \]
% where
% \[
% \|\bA\|_{\rm op}
% =
% \lambda_{\max}\{(\bA^\top\bA)^{1/2}\}
% \]
% denotes the largest singular value of \(\bA\).
Let \(\hat{\btheta}_n^{(t)}\) denote the IRO estimator at iteration \(t\), and let
\(\btheta^*\) denote the population target, up to loss-invariant transformations.

\begin{enumerate}
 \item[(i)] 
%Suppose the hidden-layer noise levels satisfy
% \[
% d_{h+1}
% \left(\prod_{i=l+1}^{h} d_i^2\right)
% d_l\sigma_l^2
% \prec \frac{1}{h},
% \qquad l=1,\ldots,h,
% \]
% and
% \[
% a_n^2
% :=
% \frac{1}{n}
% \sum_{l=1}^{h+1}
% (d_l+d_{l-1}+1)
% \frac{\sigma_l^2}{\sigma_{l-1}^2}
% \longrightarrow 0,
% \]
% where \(\sigma_0^2=\kappa_{\min}\). Then
% \[
% d_{\rm op}\!\left(\hat{\btheta}_n^{(t)},\btheta^*\right)
% =
% O_p\!\left((\lambda^*)^t\right)+O_p(a_n),
% \]
% where \(0<\lambda^*<1\) is the local contraction constant in Assumption~\ref{ass:4}. Consequently, as \(t\to\infty\) and \(n\to\infty\),
% \[
% d_{\rm op}\!\left(\hat{\btheta}_n^{(t)},\btheta^*\right)
% \stackrel{p}{\longrightarrow}0 .
% \]

For each layer \(l=1,\ldots,h+1\), define
\[
\bA_l^*=\bW_l^{*\top}\bW_l^*,
\qquad
\widehat{\bA}_l^{(t)}
=
\widehat{\bW}_l^{(t)\top}\widehat{\bW}_l^{(t)} .
\]
Let
$\lambda_1(\bA_l^*)\ge \cdots \ge \lambda_{d_{l-1}}(\bA_l^*)$ and 
$\lambda_1(\widehat{\bA}_l^{(t)})\ge \cdots \ge 
\lambda_{d_{l-1}}(\widehat{\bA}_l^{(t)})$
denote their ordered eigenvalues. Then
\[
\max_{1\le j\le d_{l-1}}
\left|
\lambda_j(\widehat{\bA}_l^{(t)})
-
\lambda_j(\bA_l^*)
\right|
\stackrel{p}{\rightarrow}0, \quad \mbox{as $t\to \infty$ and $n\to \infty$}.
\]
Moreover, let \(\bV_l^*\) and \(\widehat{\bV}_l^{(t)}\) contain orthonormal bases for the top
\(r_l\)-dimensional eigenspaces of \(\bA_l^*\) and \(\widehat{\bA}_l^{(t)}\), respectively. If the eigengap condition
$\lambda_{r_l}(\bA_l^*)-\lambda_{r_l+1}(\bA_l^*)\ge \delta_l>0$ 
holds, then
\[
\left\|
\widehat{\bV}_l^{(t)}\widehat{\bV}_l^{(t)\top}
-
\bV_l^*\bV_l^{*\top}
\right\|_{\rm op}
\stackrel{p}{\rightarrow}0, \quad \mbox{as $t\to \infty$ and $n\to \infty$}.
\]
Thus the IRO-produced estimator is eigenvalue consistent and feature-subspace consistent.

If, in addition, the \(k\)th eigenvalue of \(\bA_l^*\) is simple, in the sense that
\[
\delta_{l,k}
=
\min\{
\lambda_{k-1}(\bA_l^*)-\lambda_k(\bA_l^*),
\lambda_k(\bA_l^*)-\lambda_{k+1}(\bA_l^*)
\}
>0,
\]
with the obvious one-sided modification for \(k=1\), then the corresponding individual neural feature is consistent up to sign. That is, if \(\bv_{l,k}^*\) and
\(\widehat{\bv}_{l,k}^{(t)}\) are unit eigenvectors corresponding to
\(\lambda_k(\bA_l^*)\) and \(\lambda_k(\widehat{\bA}_l^{(t)})\), respectively, then there exists
\(s_{l,k}^{(t)}\in\{-1,1\}\) such that
\[
\left\|
\widehat{\bv}_{l,k}^{(t)}
-
s_{l,k}^{(t)}\bv_{l,k}^*
\right\|_2
\stackrel{p}{\longrightarrow}0 .
\]

\item[(ii)] Define the conventional DNN estimator 
\begin{equation} \label{est2}
\widehat{\btheta}_{\rm DNN,n}= \arg\max_{\btheta} \Big \{ \frac{1}{n} \log\pi_{DNN}(\mathbb{Y}|\mathbb{X}, \btheta) \Big\}.
\end{equation}
Then 
\begin{equation} \label{eq:transfer-theta}
d_{\rm op}(\hat{\btheta}_n^{(t)},\widehat{\btheta}_{\rm DNN,n}) \stackrel{p}{\to} 0, \quad \mbox{as $t\to \infty$ and $n\to \infty$}, 
\end{equation}
and the same conclusions hold for the  DNN estimator. In particular
for each  layer \(l=1,\ldots,h+1\),
\[
\max_{1\le j\le d_{l-1}}
\left|
\lambda_j(\widehat{\bA}_{{\rm DNN},l})
-
\lambda_j(\bA_l^*)
\right|
\stackrel{p}{\rightarrow}0, \quad \mbox{as  $n\to \infty$}.
\]
and, under the eigengap condition,
\[
\left\|
\widehat{\bV}_{{\rm DNN},l}
\widehat{\bV}_{{\rm DNN},l}^{\top}
-
\bV_l^*\bV_l^{*\top}
\right\|_{\rm op}
\stackrel{p}{\rightarrow}0, \quad \mbox{as $n\to \infty$}.
\]
If the relevant eigenvalue is simple, then the corresponding individual DNN neural feature is also consistent up to sign.
\end{enumerate}
\end{theorem}

By Theorem~\ref{thm:new1}, sublinear DNNs can achieve eigenvalue and feature-learning consistency, provided that their training errors are 
sufficiently small. As noted in Remark~\ref{rem:noise}, however, the total number of parameters in a sublinear DNN can still be much larger than \(n\). In other words, a sublinear DNN may be over-parameterized in the conventional parameter-count sense while still achieving consistency in eigenvalues and feature learning.
Nevertheless, eigenvalue and feature-learning consistency alone do not imply prediction consistency. Prediction consistency additionally requires control of the forward-stability factor, which measures how layerwise operator-norm errors propagate and are amplified through the forward pass of the neural network.

\begin{lemma}[Forward-stability factor in the layerwise operator norm]
\label{lem:Gamma-forward-stability}
Consider the deterministic DNN
\[
\bar\bh_0(\bX)=(1,\bX^\top)^\top, 
% \begin{pmatrix}
% 1 \\
% \bX^\top
% \end{pmatrix},
\qquad
\bh_l(\bX;\btheta)
=
\Psi_l\!\left(\bar\bW_l\bar\bh_{l-1}(\bX;\btheta)\right),
\quad l=1,\ldots,h,
\]
with augmented hidden representation
$\bar\bh_l(\bX;\btheta)=(1,\bh_l^\top(\bX;\btheta))^\top$, 
% =
% \begin{pmatrix}
% 1\\
% \bh_l(\bX;\btheta)
% \end{pmatrix},
% \]
and output
\[
f_{\btheta}(\bX)
=
\Psi_{h+1}\!\left(\bar\bW_{h+1}\bar\bh_h(\bX;\btheta)\right).
\]
Suppose each activation map \(\Psi_l\) is Lipschitz with constant \(L_l\).
Let \(V(\btheta^*)\) be a sufficiently small neighborhood of \(\btheta^*\), and define
\[
\tau_{l,n}(V)
=
\sup_{\btheta\in V(\btheta^*)}
\mathbb E\|\bar\bh_l(\bX;\btheta)\|_2^2,
\qquad l=0,\ldots,h,
\]
and
\[
K_{j,n}(V)
=
\sup_{\btheta\in V(\btheta^*)}
L_j\|\bar\bW_j\|_{\rm op},
\qquad j=1,\ldots,h+1.
\]
Then, for all \(\btheta\in V(\btheta^*)\),
\[
\|f_{\btheta}-f_{\btheta^*}\|_{L^2(P_{\bX})}
\le
\Gamma_n(V)d_{\rm op}(\btheta,\btheta^*),
\]
where $\Gamma_n(V)$ is the local forward-stability factor, and 
one may take
\[
\Gamma_n^2(V)
=
\sum_{l=1}^{h+1}
L_l^2\tau_{l-1,n}(V)
\prod_{j=l+1}^{h+1}K_{j,n}^2(V),
\]
with the convention that an empty product equals one.

In particular, evaluating the local bound at \(\btheta^*\), one obtains the sharper expression
\begin{equation} \label{eq:Gamma-factor}
\Gamma_n^2 \lesssim
\sum_{l=1}^{h+1}
L_l^2\tau_{l-1,n}
\prod_{j=l+1}^{h+1}
\{L_j\|\bar\bW_j^*\|_{\rm op}\}^2,
\end{equation}
where
\[
\tau_{l,n}
=
\mathbb E\|\bar\bh_l(\bX;\btheta^*)\|_2^2
=
\operatorname{tr}
\left[
\mathbb E\{
\bar\bh_l(\bX;\btheta^*)
\bar\bh_l(\bX;\btheta^*)^\top
\}
\right].
\]
If all activations have a common Lipschitz constant \(L_\Psi\), then
\begin{equation} \label{eq:Gamma-factor2}
\Gamma_n^2
\lesssim
\sum_{l=1}^{h+1}
\tau_{l-1,n}
\prod_{j=l+1}^{h+1}
\{L_\Psi\|\bar\bW_j^*\|_{\rm op}\}^2 .
\end{equation}
\end{lemma}

% \begin{proof}
% Let
% \[
% \Delta_l=\bar\bW_l-\bar\bW_l^*,
% \qquad l=1,\ldots,h+1.
% \]
% For a perturbation at layer \(l\), the resulting output perturbation is
% propagated through layers \(l+1,\ldots,h+1\). By the Lipschitz property of the
% activation functions,
% \[
% \|D_{\bar\bW_l}f_{\btheta}[\Delta_l](\bX)\|_2
% \le
% L_l
% \left[
% \prod_{j=l+1}^{h+1}
% L_j\|\bar\bW_j\|_{\rm op}
% \right]
% \|\Delta_l\|_{\rm op}
% \|\bar\bh_{l-1}(\bX;\btheta)\|_2 .
% \]
% Taking \(L^2(P_{\bX})\)-norms and then taking the supremum over
% \(\btheta\in V(\btheta^*)\) gives
% \[
% \|D_{\bar\bW_l}f_{\btheta}[\Delta_l]\|_{L^2(P_{\bX})}
% \le
% L_l\tau_{l-1,n}^{1/2}(V)
% \left[
% \prod_{j=l+1}^{h+1}K_{j,n}(V)
% \right]
% \|\Delta_l\|_{\rm op}.
% \]
% Summing over layers yields
% \[
% \|f_{\btheta}-f_{\btheta^*}\|_{L^2(P_{\bX})}
% \le
% \sum_{l=1}^{h+1}
% B_{l,n}(V)\|\Delta_l\|_{\rm op},
% \]
% where
% \[
% B_{l,n}(V)
% =
% L_l\tau_{l-1,n}^{1/2}(V)
% \prod_{j=l+1}^{h+1}K_{j,n}(V).
% \]
% By Cauchy's inequality,
% \[
% \sum_{l=1}^{h+1}
% B_{l,n}(V)\|\Delta_l\|_{\rm op}
% \le
% \left\{
% \sum_{l=1}^{h+1}B_{l,n}^2(V)
% \right\}^{1/2}
% \left\{
% \sum_{l=1}^{h+1}\|\Delta_l\|_{\rm op}^2
% \right\}^{1/2}.
% \]
% Since
% \[
% d_{\rm op}(\btheta,\btheta^*)
% =
% \left\{
% \sum_{l=1}^{h+1}\|\Delta_l\|_{\rm op}^2
% \right\}^{1/2},
% \]
% the desired bound follows with
% \[
% \Gamma_n^2(V)
% =
% \sum_{l=1}^{h+1}
% L_l^2\tau_{l-1,n}(V)
% \prod_{j=l+1}^{h+1}K_{j,n}^2(V).
% \]
% \end{proof}

The proof of the lemma is given in Appendix~\ref{proof:lem7}. The factor \(\Gamma_n\) measures how layerwise operator-norm parameter errors are propagated and amplified through the forward pass of the network. By Lemma~\ref{lem:Gamma-forward-stability}, a local upper bound for \(\Gamma_n\) is given in \eqref{eq:Gamma-factor2}. This bound shows that \(\Gamma_n\) is governed by two quantities: the effective size of the hidden representations, measured by \(\tau_{l,n}\), and the downstream spectral amplification, measured by products of layerwise spectral norms.

This forward-stability analysis complements the eigenvalue and feature-learning consistency results in Theorem~\ref{thm:new1}. Theorem~\ref{thm:new1} ensures that the learned feature matrices recover the population feature spectra and eigenspaces in operator norm. However, feature-learning consistency alone does not automatically imply prediction consistency, because prediction also depends on how errors in the learned layers propagate through subsequent layers. Therefore, to establish prediction consistency, we must additionally control the order of \(\Gamma_n\). This leads to the following theorem.

\begin{theorem}[Prediction consistency] \label{thm:pred-consistency}
Suppose the conditions of Theorem~\ref{thm:new1} hold.
Let \(f_{\btheta}\) denote the deterministic DNN map associated with parameter \(\btheta\). 
Assume further that the network is locally forward-stable around \(\btheta^*\), in the sense that there exists a deterministic sequence \(\Gamma_n\) such that, for all \(\btheta\) in a neighborhood of \(\btheta^*\),
\[
\|f_{\btheta}-f_{\btheta^*}\|_{L^2(P_{\bX})}
\le
\Gamma_n d_{\rm op}(\btheta,\btheta^*).
\]
Therefore, if
\begin{equation} \label{eq:Gamma-prediction}
\Gamma_n\{(\lambda^*)^t+a_n\}\to0,
\end{equation}
then
\begin{equation} \label{eq:pred-consistency}
\|f_{\hat{\btheta}_n^{(t)}}-f_{\btheta^*}\|_{L^2(P_{\bX})}
\stackrel{p}{\rightarrow}0 .
\end{equation}
Consequently, 
% under the well-specified regression model
% \[
% Y=f_{\btheta^*}(\bX)+\varepsilon,
% \qquad
% E(\varepsilon\mid\bX)=0,
% \]
the IRO-produced DNN is prediction consistent.
\end{theorem}

In what follows, we analyze the order of $\Gamma_n$ in three scenarios.
For each layer $l=0,\ldots,h_n$, we define 
\[
\tau_{l,n} = \operatorname{tr}(\bar\bG_{l,n}^*) =
\sum_{j}\lambda_j(\bar\bG_{l,n}^*), \qquad \mbox{with \ } \bar\bG_{l,n}^*:=\mathbb E\{
\bar\bh_l(\bX;\btheta^*)
\bar\bh_l(\bX;\btheta^*)^\top\}. 
\]
First, suppose that the DNN has a low-dimensional representation structure. More
precisely, assume that the hidden-feature spectra have uniformly bounded total
mass:
\[
\tau_{l,n}
=
\sum_j\lambda_j(\bar\bG_{l,n}^*)
=
O(1),
\qquad l=0,\ldots,h_n,
\]
even though the ambient widths \(d_l\) may diverge. This situation occurs, for
example, when only finitely many eigenvalues of \(\bar\bG_{l,n}^*\) are
non-negligible, or when the hidden representations are effectively low-rank,
sparse, or concentrated on a low-dimensional manifold. By eigenvalue
consistency, the empirical spectra of the learned feature matrices then also
exhibit the same low-dimensional structure. By feature-learning consistency,
the corresponding empirical feature subspaces are consistently recovered.
Consequently, the IRO estimator learns the relevant low-dimensional feature
directions rather than the full ambient-width representation. If, in addition,
the depth is fixed and the layerwise spectral norms are uniformly bounded, then
\[
\Gamma_n=O(1).
\]
In this case, the prediction consistency condition reduces to \eqref{eq:Gamma-prediction}. 
Thus, over-parameterization in ambient width does not necessarily harm
prediction consistency, provided the learned features have bounded effective
spectral dimension.

Second, suppose the network is wide and the hidden representations occupy the
ambient dimension. For example, if the activations are uniformly bounded and a
non-negligible fraction of the hidden units carry signal, then
\[
\|\bar\bh_l(\bX;\btheta^*)\|_2^2=O(d_l),
\qquad
\tau_{l,n}=O(d_l).
\]
Equivalently, the spectrum of \(\bar\bG_{l,n}^*\) has total mass of order
\(d_l\). In this case, eigenvalue consistency implies that the learned feature
matrices also have many non-negligible empirical eigenvalues, and feature
learning takes place in a genuinely high-dimensional feature space. If the
depth is fixed and the layerwise spectral norms are uniformly bounded, then
\[
\Gamma_n^2
=
O\left(\sum_{l=0}^{h}d_l\right),
\qquad
\Gamma_n
=
O\left\{
\left(\sum_{l=0}^{h}d_l\right)^{1/2}
\right\}.
\]
Thus prediction consistency requires the stronger condition
\[
\left(\sum_{l=0}^{h}d_l\right)^{1/2}
\{(\lambda^*)^t+a_n\}\to 0.
\]
Using the rate expression in \eqref{eq:an}, this condition is implied by the more restrictive sublinear growth condition
\[
\left(\sum_{l=0}^{h}d_l\right)
\left(\sum_{l=1}^{h+1}(d_l+d_{l-1}+1)\right)
\prec n,
\]
which says that the network must be sub-$\sqrt{n}$ 
 in total width, and consequently sublinear in total parameter count, consistent with the existing results as established in \citet{SunSLiang2021}.
 This is a worst-case width-based regime. It corresponds to the case where the feature spectra do not concentrate on a low-dimensional subspace, so the ambient widths enter the stability factor directly.
% Consequently, although
% eigenvalue and feature-learning consistency still hold, prediction consistency
% requires a stronger sample-size condition because the forward map can amplify
% operator-norm parameter errors across many active feature directions.

Third, suppose the hidden representations are spectrally low-dimensional, but
the depth \(h_n\) grows. Assume
$\tau_{l,n}=O(1)$ for $l=0,\ldots,h_n$,
so that eigenvalue consistency and feature-learning consistency still imply
stable recovery of low-dimensional feature directions at each layer. Let
\[
K_n=
\max_{1\le j\le h_n+1}
L_\Psi\|\bar\bW_j^*\|_{\rm op}.
\]
Then
\[
\Gamma_n^2
=
O\left(
\sum_{m=0}^{h_n}K_n^{2m}
\right).
\]
Consequently,
\[
\Gamma_n=
\begin{cases}
O(1), & K_n\le K<1,\\[4pt]
O(\sqrt{h_n}), & K_n=1,\\[4pt]
O(K_n^{h_n}), & K_n>1 \text{ and bounded away from }1.
\end{cases}
\]
See Appendix~\ref{proof:Gamma-order} for the justification. Thus, even when
each layer learns a low-dimensional feature representation consistently,
prediction consistency also requires control of the downstream amplification
across depth. In the contractive case \(K_n<1\), depth does not create
additional instability. In the neutral case \(K_n=1\), the stability factor
grows only as \(\sqrt{h_n}\). In the expansive case \(K_n>1\), the forward map
can amplify small parameter errors exponentially in depth.

\subsection{Approximation Power of Sublinear DNNs} \label{sect:approx}

Theorem~\ref{thm:new1} rests on the implicit assumption that the sublinear DNN can adequately approximate the target function. Given the model’s structural constraints, a natural question arises: {\it can it still approximate common target classes, e.g., continuous functions on compact sets, arbitrarily well as the sample size \(n \to \infty\)?}  While a complete characterization remains open, we establish positive results for several important classes of functions, as detailed below.

To understand the approximation power of DNNs, a line of work has analyzed compositional functions, see, e.g., \citet{Schmidt-Hieber2020Nonparametric,Bauler2019On,Poggio2017WhyWhen},
motivated by the compositional structure of DNNs.
Combining the approximation theory of \citet{Poggio2017WhyWhen} with Theorem~\ref{thm:new1}, we obtain:

\begin{theorem}\label{prop:1}
Let $f(\bx)$ be defined on a compact domain in $\mathbb{R}^{d_0}$ and admit a hierarchical compositional representation in which each constituent depends on at most $s$ variables (with $s\leq d_0$ and being fixed).

(i) 
%If $f$ is Lipschitz, then for any $\varepsilon>0$ \textcolor{red}{
If $f$ is Lipschitz with the dimension $d_0=O(n^{\alpha})$ for some $0<\alpha<1$, then for any $\varepsilon=n^{-(1-\alpha-\delta)/s}$ with $0<\delta<1-\alpha$, there exists a sublinear ReLU DNN such that $\|f-f_{\theta}\|\le \varepsilon$ as $n\to \infty$, where $f_{\theta}$ denotes the DNN function.  

(ii) If $f$ is continuously differentiable with $d_0=O(n^{\alpha})$ for some $0<\alpha<1$, then 
% for any $\varepsilon>0$, there exists a sublinear DNN with a smooth activation, e.g., sigmoid or tanh,  
% such that $\|f-f_{\theta}\|\le \varepsilon$ as $n\to \infty$.
the same conclusion holds for  sublinear DNNs that have a smooth activation function, such as sigmoid or tanh. 
\end{theorem}

 Refer to Appendix \ref{app:proof-prop1} for the proof. 
 Beyond compositional classes, the approximation theory of deep-ReLU networks established in  \citet{Montanelli2019NEB} yield analogous guarantees for functions in the Korobov space (denoted by $\mathcal{K}^{2,p}$ with an equipped $\ell^p$-norm) --- a subspace of the Sobolev space  with dominating mixed smoothness, i.e., all mixed partial derivatives up to order \(2\) exist.
 The proof of \cite{Montanelli2019NEB}  leverages the ability of deep networks to approximate sparse grids
\citep{Zenger1991sparsegrids} 
via a binary tree structure, resembling the compositional structure used in \citet{Poggio2017WhyWhen}. 

 \begin{theorem} \label{prop:2} 
If $f(\bx) \in \mathcal{K}^{2,p}([0,1]^{d_0})$ (i.e.,  defined 
on a compact domain in $\mathbb{R}^{d_0}$), where 
 the input dimension $d_0$ is fixed or grows with $n$ at the rate 
$d_0=o\left({\log n}/{\log\log n}\right)$,  then for any 
%$\varepsilon>0$
$\varepsilon = n^{-(1/2-\delta)}$ with $0 < \delta < 1/2$, there exists a sublinear ReLU DNN such that $\|f-f_{\theta}\|\le \varepsilon$ as $n\to \infty$.  
\end{theorem}

% \begin{proof}
% \citet{Montanelli2019NEB} studied the functions in the Korobov space 
% $\mathcal{K}^{2,p}$ and 
% proved the following result:
 
% For any $0<\varepsilon<1$ and 
% any function $f\in \mathcal{K}^{2,p}([0,1]^{d_0})$ 
% that satisfies  
% $\big\|\partial_{x_1}^{2}\cdots \partial_{x_d}^{2} f\big\|_{\infty}\le 1$, 
% there exists a deep ReLU network on inputs $(x_1,x_2,\ldots,x_d)^\top \in [0,1]^{d_0}$ that approximates
% $f$ to accuracy $\varepsilon$, with depth $\mathcal{O}\!\big(|\log_2 \varepsilon|\,\log_2 d_0\big)$ and the number of hidden neurons 
% \begin{equation}\label{eq:DMLP-3}
% m=\mathcal{O}\!\Big(\varepsilon^{-2}\,|\log_2 \varepsilon|^{\frac{3}{2}(d_0-1)+1}\,(d_0-1)\Big).
% \end{equation}

%  If we set 
% $\varepsilon = n^{-(1/2-\delta)}$ for some $0 < \delta < 1/2$ as the target approximation accuracy, 
% then a deep ReLU network can achieve this accuracy for the target function $f$, provided that the network has depth  
% $O\!\big((\tfrac{1}{2}-\delta)\log_2 n \log_2 d_0\big)$
% and the number of hidden neurons  
% \begin{equation} \label{eq:DMLP-4}
% m \;=\; O\!\Big(
%   n^{\,1-2\delta}\,
%   (\log_2 n)^{\tfrac{3}{2}(d_0-1)+1}\,
%   \bigl(\tfrac{1}{2}-\delta\bigr)^{\tfrac{3}{2}(d_0-1)+1}\,
%   (d_0-1)
% \Big)
% \;=\; o\!\big(n^{\,1-\delta}\big),
% \end{equation}
% by noting $\bigl(\tfrac{1}{2}-\delta\bigr)^{\tfrac{3}{2}(d_0-1)+1}(d_0-1) \prec 1$ when $d_0$ is reasonably large.  
% \end{proof}

% \subsection{Neural Scaling Law}

Additionally, we note that Theorem \ref{thm:new1} complies with the neural scaling law. 
Both \cite{Hestness2017DeepLS} and \cite{Kaplan2020ScalingLF}
investigated scaling between model size (i.e., the number of parameters) and data size;  the former found sub-linear scaling of
model size with data size, whereas the latter found a super-linear scaling. 
Specifically,  by \cite{Kaplan2020ScalingLF}, the network width can increase with the data size at a polynomial rate of $d_l \prec n^{0.676} (\approx n^{0.5/0.74})$ for neural language models;  
%\footnote{\cite{Kaplan2020ScalingLF} mentioned this rate is suboptimal, as they have not optimized regularization while varying the dataset and model size; and their largest model increased in this rate encountered mild overfitting.};
and by \cite{Hestness2017DeepLS}, the scaling law $d_l \prec n^{0.5}$ holds
for different model architectures in four deep learning domains: machine translation, language modeling, image processing, and speech recognition. 
For both scaling laws, 
the conditions of Theorem \ref{thm:new1}-(ii) 
can be satisfied by choosing an appropriate growth rate 
for the depth of the DNN, ensuring the feature-learning consistency holds.

\textcolor{black}{
Sublinear DNNs can be trained effectively using SGD. For conventional nonlinear models, obtaining the exact maximum likelihood estimator (MLE) is often challenging, however, DNNs present a different picture: in practice, they often interpolate the training data (achieving essentially zero empirical risk), which coincides with attaining  the MLE. 
This interpolation phenomenon has been widely documented in the deep-learning literature; for example, \citet{Zhou2019SGDCT} noted that SGD, although considered as a randomized algorithm, converges in an intrinsically deterministic manner to a global minimum. 
See Section~\ref{suppRegression} of the supplement for an ablation study on SGD’s sensitivity to learning rates. We find that sublinear DNNs maintain stable training and test error across a wide range of learning rates, whereas wide DNNs are markedly more sensitive.}

\section{Numerical Experiments} \label{numericsection}

We first test the performance of the IRO algorithm \ref{IROforstonet} for StoNet training; see Supplement \ref{supp:IRO} for details. Our numerical experiments show that the StoNet trained with IRO and the DNN trained with SGD perform similarly, which is consistent with the theory established in Lemma \ref{lemma:2}. In practice, the IRO algorithm requires solving a series of regressions on the entire dataset for each iteration, which can be slow for large datasets. 
Therefore, we use  SGD in all subsequent experiments, while using StoNet with IRO as a bridge for transferring theory and methods from linear models to DNNs.

% In this section, we first illustrate the performance of the StoNet using a simulated example, then we illustrate its use in sparse deep learning for large-scale data, and finally present a novel application of the StoNet in biosystem modeling and data integration.  

\subsection{Feature Learning Consistency} \label{sect:feature}

To illustrate the consistency of feature learning in sublinear DNNs, we consider the following two-hidden-layer neural network model: 
\begin{equation} \label{DNN:sim1}
y_i = \bw_3 \tanh(\bw_2 \tanh(\bw_1 \bx_i + \bb_1) + \bb_2) + b_3 + \sigma \epsilon_i, \quad i=1,2,\ldots, n,  
\end{equation} 
where $\epsilon_i$'s are {\it i.i.d.} standard Gaussian random errors. 
 The neural network has a structure of $p$-5-5-1 with $p=20$,  
 $\bx_i$'s are drawn independently from $N_{p}(0, I_p)$. 
 The neural network parameters   
 include $\bw_1 \in \mathbb{R}^{5\times 20}$, 
 $\bw_2 \in \mathbb{R}^{5\times 5}$, $\bw_3 \in \mathbb{R}^{1\times 5}$, 
 $\bb_1\in \mathbb{R}^5$, $\bb_2\in \mathbb{R}^5$, and $b_3 \in \mathbb{R}$, and each of their elements is randomly drawn  
 the set $\{-1, -0.5, 0, 0.5, 1\}$. Multiple datasets have been simulated from the model (\ref{DNN:sim1}) under each setting: 
 $n=500$ and $50,000$. 
 Obviously, this function belongs to the Korobov space and is also a hierarchical composition function, where $p$ is considered as a fixed value. 

 We modeled the simulated data using 6 different DNNs with the respective 
 structures given by $p$-5-5-1, $p$-1000-1000-1, $p$-10-10-1, 
 $p$-10-10-10-1, $p$-10-10-10-10-1, and $p$-10-10-10-10-10-1. To demonstrate the consistency of neural feature learning, we calculate the canonical correlation (CC)  coefficient $\rho_{k,1:k'}=\rho(\nu_k(\bw_1^\top\bw_1),\nu_{1:k'}(\hat{\bw}_1^\top\hat{\bw}_1))$,
 where $\nu_k(\bw_1^\top\bw_1)$ denotes the $k$-th eigenvector of 
 $\bw_1^\top\bw_1$, $\nu_{1:k'}(\hat{\bw}_1^\top \hat{\bw}_1)$ denotes   top $k'$ eigenvectors of $\hat{\bw}_1^\top\hat{\bw}_1$, and 
 $\hat{\bw}_1$ denotes an estimator of $\bw_1$. 
 Under this setting,  $\rho_{k,1:k'}$ is given by  
 \[
 \rho_{k,1:k'}=\max_{(c_1,\ldots,c_{k'})^\top \in \mathbb{R}^{k'}} \corr\left(\nu_k(\bw_1^\top\bw_1), c_1 \nu_{1}(\hat{\bw}_1^\top\hat{\bw}_1) + \cdots+c_{k'} \nu_{k'}(\hat{\bw}_1^\top\hat{\bw}_1)\right),
 \]
 which measures the extent to which the neural feature 
 $\nu_k(\bw_1^\top\bw_1)$ is recovered by the learned neural network. 
 As shown in Table \ref{CCtab:eigenvalue}, $\bw_1^\top\bw_1$ contains 
 three major eigenvalues. 
 Table \ref{CCtab1} presents the values of $\rho_{k,1:k'}$ for  $k=1,2,3$ and $k'=3$.

% \begin{table}[!htbp]
%     \centering
%     \begin{tabular}{cccc} \hline 
%         N & Evaluation & True NN & W-NN \\ \hline 
%         \multirow{3}{*}{100} &  Correlation & 0.45(0.08)& 0.24(0.02)\\  
%          & Train MSE & 0.06(0.02)& 0.00(0.00)\\
%          & Test MSE & 0.67(0.30)& 0.62(0.22)\\ \hline 
%         \multirow{3}{*}{200} &  Correlation &0.51(0.09) & 0.33(0.05) \\  
%          & Train MSE & 0.05(0.00)& 0.00(0.00)\\
%          & Test MSE & 0.37(0.11) & 0.51(0.16)\\ \hline 
%         \multirow{3}{*}{500} &  Correlation & 0.95(0.02) & 0.33(0.11) \\  
%          & Train MSE & 0.01(0.00) & 0.00(0.00) \\
%          & Test MSE & 0.02(0.00)  & 0.26(0.06) \\ \hline 
%         % \multirow{3}{*}{1000} & Correlation & 0.96(0.02) & 0.32(0.05)  \\  
%         %  & Train MSE & 0.00(0.00)& 0.00(0.00) \\
%         %  & Test MSE & 0.01(0.00)& 0.15(0.03) \\ \hline 
%     \end{tabular}
%     \caption{Correlation coefficients $\rho(\nu_1(w_1^\topw_1),\nu_1(\hat{w}_1^\top\hat{w}_1))$ and training MSE for different sample sizes (N=100, 200, 500, 50000). The mean and standard deviation, given in the parentheses, were calculated by averaging 5 different simulations.} 
% \end{table}

\begin{table}[!ht]
\centering
\caption{
Canonical correlations $\rho_{k,1:k'}$ ($k=1,2,3$) for different DNN structures, where the mean CC coefficient and its standard deviation (reported in parenthesis) are calculated by averaging over 5 independent datasets.}
%\textcolor{red}{Sehwan: It appears there is a typo in the table (see the blue entry); should the standard deviation be 0.02?}} 
\label{CCtab1}
\begin{adjustbox}{max width=1.0\textwidth} 
\begin{tabular}{cccccccc} \toprule 
 $n$ &  CC  & $p$-5-5-1 & $p$-1000-1000-1 & $p$-10-10-1 & $p$-10-10-10-1 & $p$-10-10-10-10-1 &  $p$-10-10-10-10-10-1 \\ \midrule 
 500 & $\rho_{1,1:1}$ & 0.95(0.02) & 0.16(0.03) & 0.60(0.15) & 0.68(0.14) &  0.58(0.14) &  0.58(0.08) \\
500 & $\rho_{1,1:2}$ & 0.97(0.02) & 0.28(0.05) & 0.84(0.06)& 0.89(0.04) & 0.77(0.12) &  0.66(0.09) \\
500 & $\rho_{1,1:3}$ & 0.99(0.00) & 0.37(0.07) &  0.89(0.04)&0.90(0.04) & 0.79(0.11) &0.81(0.05) \\ %\midrule
50,000& $\rho_{1,1:1}$ & 1.00(0.00) & 0.88(0.03) & 0.88(0.08) & 0.88(0.07)  & 0.96(0.02)  & 0.88(0.08) \\
50,000& $\rho_{1,1:2}$ & 1.00(0.00) & 0.95(0.01) & 0.96(0.03)& 0.99(0.00) & 0.97(0.01) & 0.99(0.01) \\
50,000 & $\rho_{1,1:3}$ & 1.00(0.00) & 0.97(0.01) & 0.99(0.00)& 1.00(0.00) & 0.99(0.00) & 1.00(0.00) \\ \midrule
% \end{tabular}
% \end{adjustbox}
% \end{table}

% \begin{table}[H]
% \centering
% \vspace{-0.2in}
% \caption{
% Canonical correlations $\rho_{2,1:k'}$ across DNN architectures; see the caption of Table~\ref{CCtab1} for a description of the entries.} 
% % Canonical correlations $\rho_{2,1:k'}$ for different DNN structures, 
% % see the caption of Table \ref{CCtab1} for table entries.}
% %where the mean CC coefficient and its standard deviation (reported in parenthesis) are calculated by averaging over 5 independent datasets.} 
% \label{CCtab2}
% \begin{adjustbox}{max width=1.0\textwidth} 
% \begin{tabular}{cccccccc} \toprule 
%  $n$ &  CC  & $p$-5-5-1 & $p$-1000-1000-1 & $p$-10-10-1 & $p$-10-10-10-1 &  $p$-10-10-10-10-1  &  $p$-10-10-10-10-10-1\\ \midrule 
 500 & $\rho_{2,1:1}$ & 0.22(0.08) & 0.22(0.04) & 0.48(0.11)& 0.49(0.11) & 0.55(0.10) & 0.59(0.08) \\
500 & $\rho_{2,1:2}$ & 0.88(0.09) & 0.47(0.02) & 0.79(0.05) & 0.81(0.04) & 0.64(0.09) &0.76(0.03) \\
 500  & $\rho_{2,1:3}$ & 0.90(0.09) & 0.52(0.04) & 0.89(0.04)& 0.87(0.05) & 0.82(0.04) & 0.83(0.03) \\ %\midrule
50,000 & $\rho_{2,1:1}$ & 0.06(0.01) & 0.29(0.10) & 0.25(0.04)& 0.33(0.12) &0.13(0.03) & 0.27(0.07) \\
50,000 & $\rho_{2,1:2}$ & 1.00(0.00) &0.94(0.01) & 0.97(0.03)& 0.99(0.00) & 0.83(0.11) & 0.89(0.08) \\
50,000 & $\rho_{2,1:3}$ &1.00(0.00) &0.97(0.01) & 0.99(0.01)& 0.99(0.00) & 0.97(0.02) &1.00(0.00) \\ \midrule
% \end{tabular}
% \end{adjustbox}
% \end{table}

% \begin{table}[!ht]
% \centering
% \caption{
% Canonical correlations $\rho_{3,1:k'}$ across DNN architectures; see the caption of Table~\ref{CCtab1} for a description of the entries.}

% % Canonical correlations $\rho_{3,1:k'}$ for different DNN structures, where the mean CC coefficient and its standard deviation (reported in parenthesis) are calculated by averaging over 5 independent datasets.} 
% \label{CCtab3}
% \begin{adjustbox}{max width=1.0\textwidth} 
% \begin{tabular}{cccccccc} \toprule 
%  $n$ &  CC  & $p$-5-5-1 & $p$-1000-1000-1 & $p$-10-10-1 & $p$-10-10-10-1 &  $p$-10-10-10-10-1 & p-10-10-10-10-10-1 \\ \midrule 
 500& $\rho_{3,1:1}$ & 0.12(0.03) & 0.24(0.07) & 0.28(0.01)& 0.19(0.02) & 0.26(0.07) &  0.27(0.09) \\ 
500 & $\rho_{3,1:2}$ & 0.28(0.09) & 0.34(0.06) & 0.49(0.08)& 0.38(0.08) & 0.61(0.07) & 0.38(0.07) \\
 500& $\rho_{3,1:3}$ & 0.96(0.03) & 0.42(0.07) & 0.83(0.09)& 0.62(0.14) & 0.82(0.05) & 0.63(0.08) \\ %\midrule 
50,000& $\rho_{3,1:1}$ & 0.04(0.01) & 0.14(0.05) & 0.19(0.11)& 0.10(0.03) & 0.16(0.05) & 0.17(0.12) \\ 
50,000 &$\rho_{3,1:2}$ & 0.06(0.02) & 0.28(0.05) & 0.22(0.11)& 0.13(0.02)& 0.44(0.13) & 0.29(0.15) \\ 
50,000 & $\rho_{3,1:3}$ & 1.00(0.00) & 0.93(0.04) & 0.99(0.01)&0.99(0.00) & 0.97(0.01) & 0.99(0.01) \\ \bottomrule
\end{tabular}
\end{adjustbox}
\end{table}

A careful examination of Table \ref{CCtab1} shows that 
the sublinear DNNs not only recover the neural features 
but also preserve their orders as $n$ becomes large. Note that 
 the network $p$-1000-1000-1 is considered wide when $n=500$ but becomes 
 sublinear in width for $n=50,000$, and its results clearly highlight 
the importance of a sublinear structure for effective neural feature  recovery.  
It is worth noting that this neural network has a total of 1,023,001 parameters, making it highly over-parameterized --- a scenario commonly encountered in our deep learning practice.

Furthermore, the recovery of low-dimensional neural features by the network  $p$-1000-1000-1 suggests that it contains a large number of redundant parameters, supporting the use of sublinear DNNs and 
the low-rank approximation method
proposed in LoRA \citep{Hu2021LoRALA}. 
The sublinear DNN actually provides  
loose, yet effective, upper bounds for the ranks that can be used for each hidden layer of the wide DNN in LoRA.
Our results also indicate that the depth of the DNN affects the recovery of neural features, but not significantly. 
The similar performance of all the different DNNs in neural feature recovery aligns well with the findings of \cite{Li2018MeasuringTI}, where it was observed that the intrinsic dimension of the DNN 
remains stable even as models grow in width and depth.

For a thorough comparison for the performance of sublinear and wide DNNs, we reported their training and test errors   in Table \ref{CCpred}. The comparison highlights the importance of consistent feature learning in DNN prediction. 
In particular, all the networks achieve oracle-level training and test errors when $n=50,000$, where major neural features of the data have been successfully recovered as implied by Table \ref{CCtab1}.
In contrast, when $n=500$, the test errors appear to depend on 
the extent of neural feature recovery.

\begin{table}[!ht]
\centering
\caption{
Mean squared training and test errors produced by different DNN structures, where the mean and standard deviation (reported in parenthesis) are calculated by averaging over 5 independent datasets.} 
\label{CCpred}
\begin{adjustbox}{max width=1.0\textwidth} 
\begin{tabular}{cccccccc} \toprule 
 $n$ &  & $p$-5-5-1 & $p$-1000-1000-1 & $p$-10-10-1 & $p$-10-10-10-1 &  $p$-10-10-10-10-1 & $p$-10-10-10-10-10-1\\ \midrule 
 --- & Model size & 141 & 1,023,001 & 331 & 441 & 551 & 661 \\ \midrule
 \multirow{2}{*}{500} & train & 0.01(0.00) & 0.00(0.00)  & 0.01(0.00) & 0.01(0.00) & 0.01(0.00)&  0.00(0.00) \\ 
    & test     & 0.05(0.02) & 0.21(0.04) & 0.08(0.02) &  0.14(0.05)&  0.15(0.03)& 0.16(0.04)\\ \midrule 
\multirow{2}{*}{50,000} &  train &  0.01(0.00) & 0.01(0.00) &   0.01(0.00)&  0.01(0.00)& 0.01(0.00) & 0.01(0.00) \\ 
      & test &  0.01(0.00)&  0.01(0.00)& 0.01(0.00) & 0.01(0.00)&  0.01(0.00)&0.01(0.00)\\ \bottomrule
\end{tabular}
\end{adjustbox}
\end{table}

% Finally, we note that since the neural features are invariant with respect to the elementary operations on $\bw_l$, such as interchanging 
% two rows or changing the sign or scale of a row,  
% Tables \ref{CCtab1}--\ref{CCtab3} provide a reasonable, though not complete, illustration of the theoretical results established in Theorem \ref{thm:new1}.

As mentioned earlier, $\btheta$ is unique up to loss-invariant transformations, such as reordering hidden neurons within the same layer or simultaneously altering the signs or scales of certain weights and biases. This invariance property makes it particularly challenging to demonstrate the consistency of DNN parameter estimation.
%as established in Theorem \ref{thm:new1}.
 To address this challenge, we present in Table \ref{CCtab:supp}
 the canonical correlations $\rho_{4,1:k'}$ and $\rho_{5,1:k'}$ (for $k'=1,2,\ldots,5$) 
 achieved by the network $p$-5-5-1 with $n=50,000$, and in Table \ref{CCtab:eigenvalue} the   eigenvalues.
Based on the results shown in Table \ref{CCtab1}, Table \ref{CCtab:supp}, and Table \ref{CCtab:eigenvalue},  we can conclude that for this network, the eigenvalues and eigenvectors of 
$\bw_1^\top\bw_1$ have been asymptotically recovered by those of $\hat{\bw}_1^\top\hat{\bw}_1$. 
% Furthermore, since $d_1 \leq p$ in this network, 
% the recovery of the eigenvalues and eigenvectors implies the recovery of 
% $\bw_1$ (up to elementary operations). Note that in the case where $d_1\leq p$, each element of $\bw_1$ can be uniquely determined 
% by solving the equation 
% \[
% \bw_1^\top \bw_1=   [\bv_1,\bv_2,\ldots,\bv_{d_1}] \diag\{\lambda_1, \lambda_2,\ldots, \lambda_{d_1}\}
% [\bv_1,\bv_2,\ldots,\bv_{d_1}]^\top,
% \]
% where $\{\lambda_1,\ldots,\lambda_{d_1}\}$ and $\{\bv_1,\ldots,\bv_{d_1}\}$ denote the given eigenvalues and eigenvectors, respectively.
%  In summary, we have provided an effective method to verify the consistency of parameter estimation for the DNNs with an upper triangular structure (i.e., where $d_i \leq d_{i-1}$ for $i=1,2,\ldots,h,h+1$).
%  This eigen analysis-based method accounts for some loss-invariant transformations that $\btheta$ is subject to. 

The recovery of the eigenvalues and eigenvectors of \(\bw_1^\top\bw_1\) implies recovery of the neural-feature directions and the singular values of \(\bw_1\). When \(d_1\le p\) and \(\bw_1\) has full row rank, \(\bw_1\) is determined by \(\bw_1^\top\bw_1\) only up to a left orthogonal transformation. Thus the eigen-analysis identifies the row space and feature directions of \(\bw_1\), but not the exact weight matrix itself without choosing an additional representative.

\subsection{Double Descent and Beyond} 

Double descent is a surprising phenomenon in machine learning, which describes 
the observation that the test error of a model drops as the model   
grows ever larger into the highly overparameterized regime 
relative to the training sample size, see e.g., \cite{Belkin2019ReconcilingMM,Adlam2020UnderstandingDD,Schaeffer2023DoubleDD}. This phenomenon will be explained at the end of this subsection from a perspective of neural feature learning. 

\paragraph{MNIST}
As in \cite{Belkin2019ReconcilingMM}, we worked with a subset of MNIST (with 
$n_{train}=4000$, $p=784$, and $K=10$ classes) as training data.  
We trained a one-hidden-layer neural network: 784-L-10, where $L$ is the hidden layer width, and measured its prediction performance on a test dataset with $n_{test}=10,000$. Figure \ref{DoubleMNIST}(a) shows 
the resulting training and test errors, 
where the second descent in test errors occurs with 
$L$ ranging 50$\sim$1000. Notably, for each $L\in [50,1000]$, the resulting DNN is sublinear in width, although its total number of parameters can be much 
greater than $n_{train}$. 
Our feature-learning consistency theory provides a principled explanation for the second descent phenomenon, as detailed below.

\begin{figure}[htbp]
\begin{center}
\begin{tabular}{cc}
(a)  & (b) \\ 
\includegraphics[height=1.75in,width=2.75in]{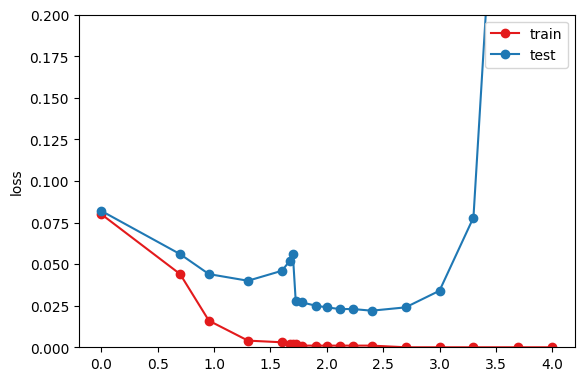}
 & 
 \includegraphics[height=1.75in,width=2.75in]{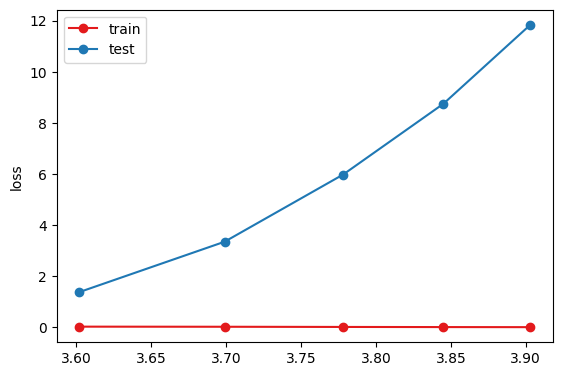} \\ 
\end{tabular}
\end{center}
\vspace{-0.2in}
\caption{MNIST example, where the $y$-axis represents $\ell_2$-loss and the $x$-axis represents $\log_{10}(L)$: (a) training both layers' weights; (b) training the second layer's weights only.}
\label{DoubleMNIST}
\end{figure}

Importantly, Figure \ref{DoubleMNIST} shows that as $L$ further increases, the test error increases again. That is, this example also exhibits a ``double ascent'' phenomenon. We would attribute the second ascent to the lack of feature learning consistency. To make this point clearer, 
we trained the one-hidden-layer neural networks again with $L\in [4000,8000]$, each of which forms a wide neural network.
 For each of the wide networks, we fixed the first layer's weights, initialized with $N(0, 0.5^2)$; and trained the second layer's weights only, which were initialized 
with $N(0,0.1^2)$.  Since $L>n_{train}$, the second layer
forms $K$ small-$n$-large-$p$ logistic regressions and  zero-training error solutions exist. 
%However, since the first layer weights are random, the features represented by the first hidden layer are purely noise. 
As shown in Figure \ref{DoubleMNIST}(b), these networks can still attain zero training errors,  but their test errors are very large.  
By design, the features represented by the first hidden layer 
of the wide DNNs are purely noise;   the resulting 
large test errors indicate the importance of feature learning consistency. 

% \begin{wraptable}{r}{0.6\textwidth}
%  \vspace{-0.3in} 
%   \caption{Training and test errors produced by a wide neural network with structure 784-5000-10, different learning rates and epochs, for the subset MNIST data.}
%   \vspace{-0.1in}
%         \label{MNIST_tab_test}
%         \centering
%         \begin{tabular}{cccc} \toprule
%          Learning rate &  \#epoch & training loss & test loss \\  \midrule 
%              0.001     &   4,000  &    0.000          & 0.523 \\
%              0.001     &   40,000 &    0.000          & 0.515 \\ 
%              0.001     & 100,000  &    0.000          & 0.497 \\ 
%              0.01     &     4,000      &    0.000         & 0.048       \\ \bottomrule
%         \end{tabular}
%     \vspace{-0.2in}
%     \end{wraptable}

In what follows, we further explain the importance of feature learning consistency from two perspectives. First, let's understand why a wide DNN 
can predict well, if it is trained by a gradient descent method.  
% Theoretically, as shown in a series of papers, see e.g., \cite{Gunasekar2017ImplicitRI}, \cite{Gunasekar2018CharacterizingIB}, 
%     \cite{Soudry2018TheIB}, and \cite{Ji2019TheIB}, the gradient descent method provides an implicit regularization for the models in training. Specifically, for high-dimensional linear and logistic regression models initialized at origin, the gradient descent method converges to the minimum Euclidean norm solution. 
    % That is, for the linear regression, it converges to a ridge regression with 
    % a zero penalty parameter; and for logistic regression, it converges to the maximum margin solution like the support vector regression.  
    % Combining these results with the StoNet leads to the following explanation for the success of wide DNNs. 
    % For example, 
    Consider a StoNet model for nonlinear regression.
    Suppose that the StoNet model is true, and its hidden layer outputs $\bYmis$ are observed. Training such a StoNet is reduced to solving a series of high-dimensional linear regressions. If a gradient descent method is used, then training the StoNet is equivalent to solving a series of ridge regressions with zero penalty, as the gradient descent method provides an implicit regularization for the models in training, see e.g., \cite{Gunasekar2017ImplicitRI}, 
    %\cite{Gunasekar2018CharacterizingIB}, 
    \cite{Soudry2018TheIB}, and \cite{Ji2019TheIB}. 
    By the recovery property of ridge regression  \citep{Kobak2020TheOR}, each of the ridge regressions leads to 
    consistent feature learning for its relevant variables as well as 
    consistent estimation for its response. Therefore, the wide DNN can still predict well, if sufficiently trained with gradient descent. For classification problems, the explanation is similar.   

    \begin{table}[!ht]
 %\vspace{-0.35in} 
  \caption{Training and test errors produced by a wide neural network with structure 784-5000-10, different learning rates and epochs, for the subset MNIST data.}
 % \vspace{-0.1in}
        \label{MNIST_tab_test}
        \centering
        \begin{tabular}{cccc} \toprule
         Learning rate &  \#epoch & training loss & test loss \\  \midrule 
             0.001     &   4,000  &    0.000          & 0.523 \\
             0.001     &   40,000 &    0.000          & 0.515 \\ 
             0.001     & 100,000  &    0.000          & 0.497 \\ 
             0.01     &     4,000      &    0.000         & 0.048       \\ \bottomrule
        \end{tabular}
    %\vspace{-0.2in}
    \end{table}

 On the other hand, as pointed out by \cite{Soudry2018TheIB}, 
the convergence of the gradient descent to its implicit regularization limit is very slow, only logarithmic in the convergence of the loss itself. 
%``This can help explain the benefit of continuing to optimize the logistic or cross-entropy loss even after the training error is zero'' \citep{Soudry2018TheIB}. 
This suggests that the  ``double ascent'' phenomenon in Figure 
\ref{DoubleMNIST}(a) might be due to insufficient training.  
To examine this, we re-trained a wide neural network with structure $p$-5000-10. We set the learning rate 
 to 0.001 and increased the number of epochs 
 from 4,000 to 40,000 and 100,000. 
We found that with lengthened runs, the test error of the wide DNN 
  can be reduced, but at a very slow rate,
  %as noted by \cite{Soudry2018TheIB}, 
  see Table \ref{MNIST_tab_test}. We have also set the learning rate to 0.01 and re-trained the model 
  for 4,000 epochs, which yielded low test errors and recovered 
  the double descent phenomenon. 
  Other than test errors, we compared the features learned in different runs. Figure \ref{fig:epochs} shows that consistent feature learning can be achieved by 
  the wide network  with a learning rate of 0.01, while a learning rate of 0.001 may require an extremely long time to achieve the same result. In summary, Table \ref{MNIST_tab_test} and Figure \ref{fig:epochs} underscore the 
  importance of feature learning consistency, 
  reinforcing the evidence we observed in Section \ref{sect:feature}.
  
% \begin{wraptable}{r}{0.6\textwidth}
%  \vspace{-0.3in} 
%   \caption{Training and test errors produced by a wide neural network with structure 784-5000-10, different learning rates and epochs, for the subset MNIST data.}
%   \vspace{-0.1in}
%         \label{MNIST_tab_test}
%         \centering
%         \begin{tabular}{cccc} \toprule
%          Learning rate &  \#epoch & training loss & test loss \\  \midrule 
%              0.001     &   4,000  &    0.000          & 0.523 \\
%              0.001     &   40,000 &    0.000          & 0.515 \\ 
%              0.001     & 100,000  &    0.000          & 0.497 \\ 
%              0.01     &     4,000      &    0.000         & 0.048       \\ \bottomrule
%         \end{tabular}
%     \vspace{-0.2in}
%     \end{wraptable}

  Additionally, we compared in Figure \ref{fig:transition} 
  the features learned by the sublinear and wide DNNs 
  with a learning rate of 0.001, 
  where all the networks have been sufficiently trained to zero training errors. The comparison indicates that
   the sublinear DNN works better than the wide ones in  feature extraction.

\begin{figure}[!ht]
\centering
 \begin{tabular}{cccc}
%(a)  & (b) & (c) & (d) \\  
  \includegraphics[width=0.6in]{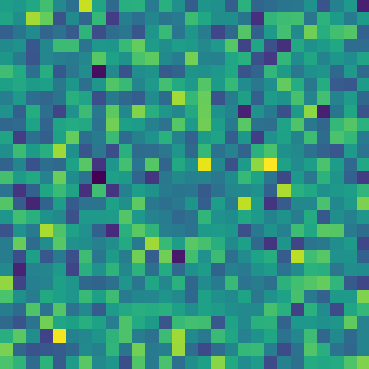} & 
  %\caption{LR: 0.001, Epoch: 4000}
  %\label{fig:sub1}
  \includegraphics[width=0.6in]{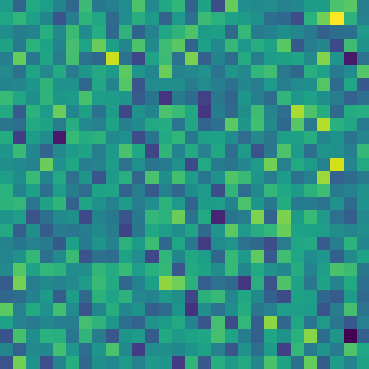} &
  % \caption{LR: 0.001, Epoch: 40000}
  % \label{fig:sub1}
  \includegraphics[width=0.6in]{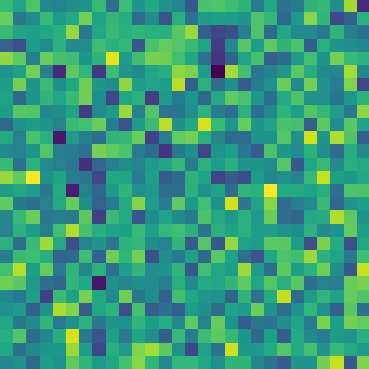} &
  % \caption{LR: 0.001, Epoch: 100000}
  % \label{fig:sub1}
  \includegraphics[width=0.6in]{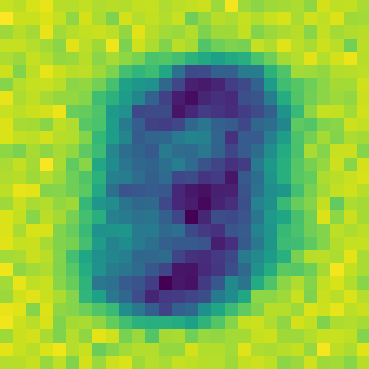}
  % \caption{LR: 0.01, Epoch: 4000}
  % \label{fig:sub1}
\end{tabular}
\caption{Features learned by a wide neural network $p$-5000-10 for the subset MNIST data under the settings (learning rate, epoch)=(0.001, 4,000), (0.001, 40,000), (0.001, 100,000), and (0.01, 4,000), from left to right, where
the features were extracted from the first hidden layer using the method as described in 
 \cite{Radhakrishnan2024MechanismFF}.
  }
\label{fig:epochs}
\end{figure}

 \begin{figure}[!ht]
\centering
\begin{tabular}{ccccccc} 
  L=80  & L=170 &  L=500 & 
L=1000 &  L=2000 & L=5000 & L=10000 \\
  \includegraphics[width=0.5in]{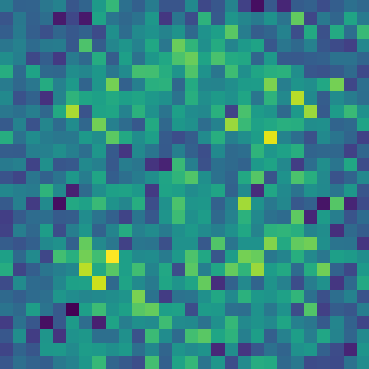} &
  \includegraphics[width=0.5in]{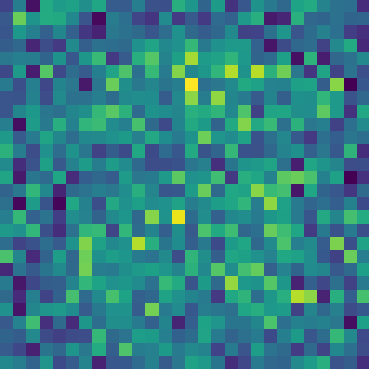} & 
  \includegraphics[width=0.5in]{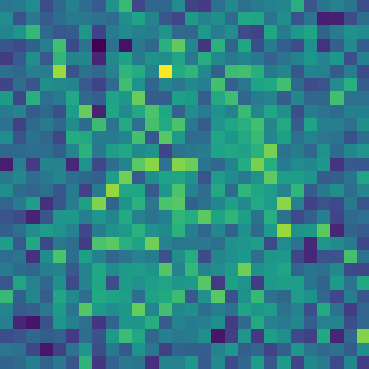} & 
  \includegraphics[width=0.5in]{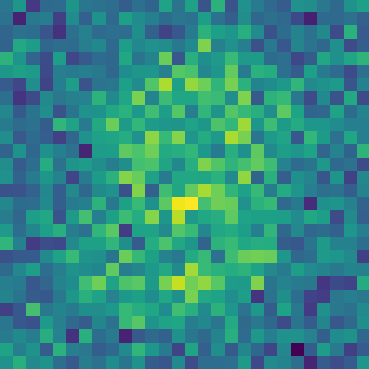} & 
  \includegraphics[width=0.5in]{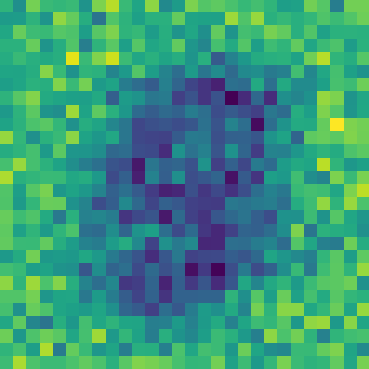} & 
  \includegraphics[width=0.5in]{MNIST_feature/5000.png} & 
  \includegraphics[width=0.5in]{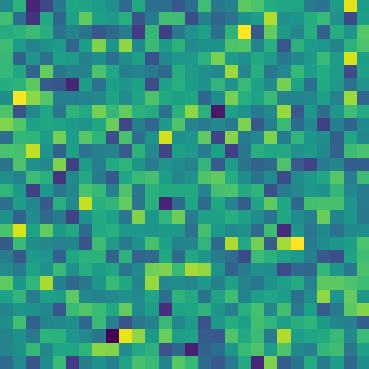}
  \end{tabular} 
\caption{Features learned by the neural networks 
 with structure $p$-$L$-10 for different values of $L$, where the 
 features were extracted from the first hidden layer using 
 the method as described in 
 \cite{Radhakrishnan2024MechanismFF}.
 }
\label{fig:transition}
\end{figure}

 Other than classification, we have tried a nonlinear regression 
 problem, see Supplement \ref{suppRegression}. 
  As shown by Figure \ref{fig:nonlinearReg}, if we fix the first layer weights at random values and trained the second layer weights only,  the training error can be reduced to 0, while the test error is large. Again, this underscores the importance of feature learning consistency for prediction.

  In summary, our experiments indicate that {\it feature learning consistency is crucial for DNNs to achieve accurate predictions}.
  Whether a network is narrow or wide, it can perform well as long as it effectively extracts the correct features. 
  From a feature learning  perspective, we provide the following conjectural explanation for 
  the double descent phenomenon:
  
  First of all, a large value of $L$ allows  
  more features to be extracted from the data; however, as mathematically shown by \cite{Chang1983OnUP}, the importance of features in separating data classes is not necessarily aligned with their eigenvalues. For instance, 
  \cite{Chang1983OnUP} constructed a two-component mixture Gaussian example, 
  where the two components are only well-separated in the subspace of the first and last eigenvectors.
  In the context of DNNs, some useful features with large eigenvalues can  be learned when $L$ is small, leading to the first descent in test errors. As $L$ increases to a moderate value, additional noisy features may be learned,  resulting in an ascent in test errors. Finally,
   some useful features with small eigenvalues may only  be learned when $L$ is sufficiently large, leading to the second descent in test errors,  where the   network is highly over-parameterized.

\subsection{More Examples: Sublinear or Wide?}

In this subsection, we delve deeper into the choice of hidden layer widths from the perspective of feature learning. Theoretically, as shown in a series of papers (e.g., \cite{Gunasekar2017ImplicitRI}, 
%\cite{Gunasekar2018CharacterizingIB}, 
\cite{Soudry2018TheIB}, \cite{Ji2019TheIB}), the gradient descent method provides implicit regularization during model training. Specifically, for high-dimensional linear regression models initialized at the origin, gradient descent converges to the solution with the minimum Euclidean norm.
By applying this result to the StoNet model (\ref{eq:stonet}), we arrive at the following solution for $\bw_i$ at each hidden layer: 
\[
\hat{\bw}_i= \mathbb{Y}_i^\top \Psi(\mathbb{Y}_{i-1}) (\Psi(\mathbb{Y}_{i-1})^\top \Psi(\mathbb{Y}_{i-1})+ \tilde{\sigma}^2 I)^{-1},
\]
where $\tilde{\sigma}^2 \geq 0$ represents the implicit penalty coefficient, and $\mathbb{Y}_i \in \mathbb{R}^{n\times d_i}$ denotes 
imputed $\bY_i$ values for all $n$ observations. This leads to the rank constraint:
\[
{\rm rank}(\hat{\bw}_i^\top \hat{\bw}_i) \leq \min\{ {\rm rank}(\Psi(\mathbb{Y}_{i-1})), 
{\rm rank}(\mathbb{Y}_i)\} \leq \min\{n,d_i,d_{i-1}) \leq n,
\]
which indicates that an overly wide neural network will not learn more than $n$ features at each hidden layer. By the singular value decomposition of $\hat{\bw}_i$, it is clear that features corresponding to zero eigenvalues will not affect the values of $\mathbb{Y}_i$.
Therefore, to enable the DNN to extract more features from the data, one should set  $d_i$'s to be reasonably large, but not necessarily greater than $n$.
 
To illustrate this finding, \textcolor{black}{we considered a simulation study, see 
 Section \ref{result:sublinear} of the supplement, where the true regression function has a hierarchical composition structure and the networks 
 $p-L-1$, $p-L-L-1$, and $p-L-L-L-L-1$ are trained. We set $n=500$ and 
$L \in \{2,\ldots,1000\}$.}
%The results, see Table \ref{tab:narrow-wide comparison}, suggest that sublinear DNNs, when designed with an appropriate width, can perform equally or better than wide DNNs in prediction. 
Additionally, we considered three UCI datasets:  
Boston housing, Yacht,  and Energy.  
For each dataset, we tried DNNs with structure $p$-$L-\cdots-L$-1,   with $L$ ranging from 100 to 2,000 and $h$ ranging from 2 to 7. 
For real-data problems, a slightly deeper architecture may better capture the unknown compositional structure of the true function. For evaluation, we performed five random train/test splits and trained a fresh network on each split. Refer to Tables \ref{boston}--\ref{energy} for the results. 

The results indicate that, with appropriate width and depth, sublinear DNNs can perform as well as or better than wide DNNs in prediction. Our findings recommend using sublinear architectures with a reasonably large width and suitable depth so that useful features are extracted and the compositional structure is captured. Moreover, to match the predictive performance of wider counterparts, the depth of a sublinear DNN may need to increase roughly inversely with its width.

\subsection{CelebA}

As another application of the sublinear DNN, we consider an example of feature extraction 
in classifying images from the CelebA dataset \citep{Liu2015DeepLF}.
As in \cite{Radhakrishnan2024MechanismFF}, we employed a fully connected ReLU DNN 
for the task. The DNN we used has a structure of $3\times 64\times 64-L-L-L-L-2$ with $L=1024$. 
Therefore, the DNN is still of sublinear width when applied to the CelebA data  with a training sample size $n_{train}=14,000$. 
We trained the fully connected DNN using SGD with a momentum 
parameter of 0.9, a learning rate of 0.05, a mini-batch size of 64, and 100 epochs.  Figure \ref{fig:celebA} shows four features 
extracted in training, which indicate the success of 
feature learning by the sublinear DNN.

\begin{figure}[!ht]
    \centering
    \includegraphics[width=0.2\textwidth]{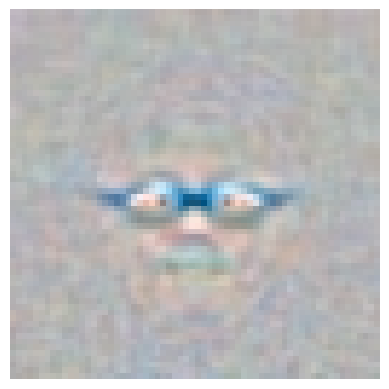}
    \includegraphics[width=0.2\textwidth]{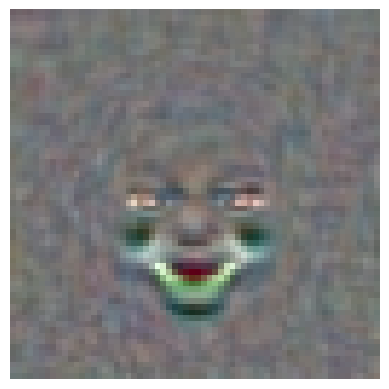} 
    \includegraphics[width=0.2\textwidth]{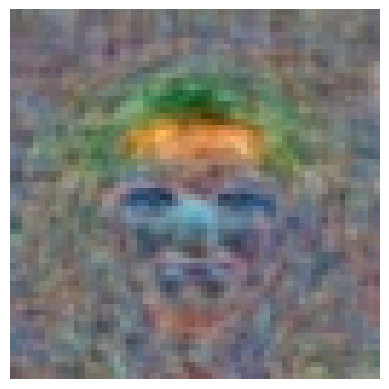}
    \includegraphics[width=0.2\textwidth]{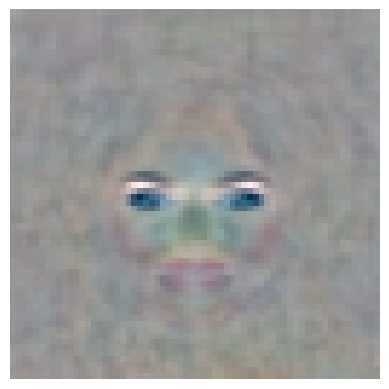}
    \caption{Four features extracted from the first hidden layer, as described in \cite{Radhakrishnan2024MechanismFF}, in training a fully connected ReLU DNN with structure $3\times 64\times 64$-L-L-L-L-2 for the CelebA data:  glass, smile, hat, and  arched eyebrow, from left to right.}
    \label{fig:celebA}
\end{figure}

% \begin{figure}
%     \centering
%     \includegraphics[width=0.4\textwidth]{Figure/smile.png}
%     \caption{Smile}
%     \label{fig:enter-label}
% \end{figure}

% \begin{figure}
%     \centering
%     \includegraphics[width=0.4\textwidth]{Figure/hat.png}
%     \caption{Hat}
%     \label{fig:enter-label}
% \end{figure}

 \vspace{-0.1in}
\section{Structure Analysis for Large-Scale DNNs} \label{deepCNN}

It is worth noting that many large-scale DNNs, 
such as AlexNet \citep{Krizhevsky2012ImageNetCW}, VGGNet \citep{Simonyan2014VeryDC}, 
ResNet \citep{he2016deep}, and GoogLeNet \citep{Szegedy2015GoogleNet}, 
belong to the class of sublinear DNNs in their benchmark 
studies, despite containing   
 a huge number of parameters.  
For any deep CNN, we can still randomize the feeding value
of each node with incoming trainable connections  
as in (\ref{eq:stonet}), thereby enabling the construction of an 
asymptotically equivalent StoNet for it.  
For each node, the number of incoming connections, i.e., the dimension of explanatory variables of the corresponding regression, is calculated as $(s_l^{(1)}*s_l^{(2)}*d_{l-1}+1)$, where 
$s_l^{(1)}*s_l^{(2)}$ denotes the filter  size and corresponds to the fixed $s$ value in the constituent map 
of the compositional function (see Theorem \ref{prop:1}), and $d_{l-1}$ denotes the number of filters in the previous layer, and `1' represents the bias term. 
For a deep CNN belonging to the class of sublinear DNNs, the following two conditions need to be satisfied: 
$\max_{l} (s_l^{(1)} * s_l^{(2)}*d_{l-1}+1) \prec n$ and $\sum_{l} d_l \prec n$.  
The latter condition can also be interpreted as the total number of regressions formed in the stochastic deep CNN.  
%used at each convolutional layer, it 
% is easy to see that the weight matrices $\bw_l$'s of the deep CNNs 
% are still of full-row rank, satisfying  Assumption \ref{ass:3}-(iii). 
The structures of the deep CNNs are analyzed in the following, 
based on the summary provided by Aqeel Anwar at
\url{https://towardsdatascience.com/the-w3h-of-alexnet-vggnet-resnet-and-inception-7baaaecccc96}.

 AlexNet is one of the earliest deep CNNs, which won the 2012 ImageNet LSVRC-2012 challenge.
 It comprises a total of 62.4 million trainable parameters, including 5 convolutional layers and 3 fully connected (FC) layers. In this network, the maximum number of incoming connections to a single node is 9,217, achieved at the first FC layer, and the total number of nodes 
 with incoming trainable connections is 10,568.  
VGG16 has approximately 138.4 million parameters. In VGG16, the maximum number of incoming connections to a single node is 25,089, achieved at the first FC layer, and the total number of nodes with incoming trainable connections is 13,544. 
ResNets have many variants, e.g., ResNet18, ResNet50, and ResNet101. 
Let's consider ResNet18 as an example. It comprises approximately 11.5 million trainable parameters, its maximum number of incoming connections to a single node is 4,609, achieved at layers 15, 16 and 17, and its total number of nodes with incoming trainable connections is 4,904.  
GoogLeNet has about 6.4 million trainable parameters, in which the maximum 
number of incoming connections to a single node is 1,729, achieved in Inception 5b, 
and the total number of nodes with incoming trainable connections is 8,280. 

 In summary, all these networks are  sublinear when trained on large-scale  datasets such as ImageNet, CIFAR10, CIFAR100, and MNIST, each with $n\geq 50,000$ training samples.
% Moreover,  due to the inherent 
% hierarchical composition structure 
% of images, 
% Theorems \ref{thm:new1} and 
% \ref{prop:1} generally hold for these sublinear networks.
% With these theoretical results, the preceding analysis may help elucidate why these large-scale  networks perform exceptionally well in prediction after adequate training with a massive dataset. 
 Moreover, because images exhibit an inherent hierarchical, compositional structure, Theorems~\ref{thm:new1} and~\ref{prop:1} apply to these sublinear networks. Taken together, these theoretical insights and the preceding analysis help explain why such large-scale networks achieve exceptional predictive performance after sufficient training on large-scale datasets.

\begin{figure}[!ht] 
    \centering
\includegraphics[width=0.5\textwidth]{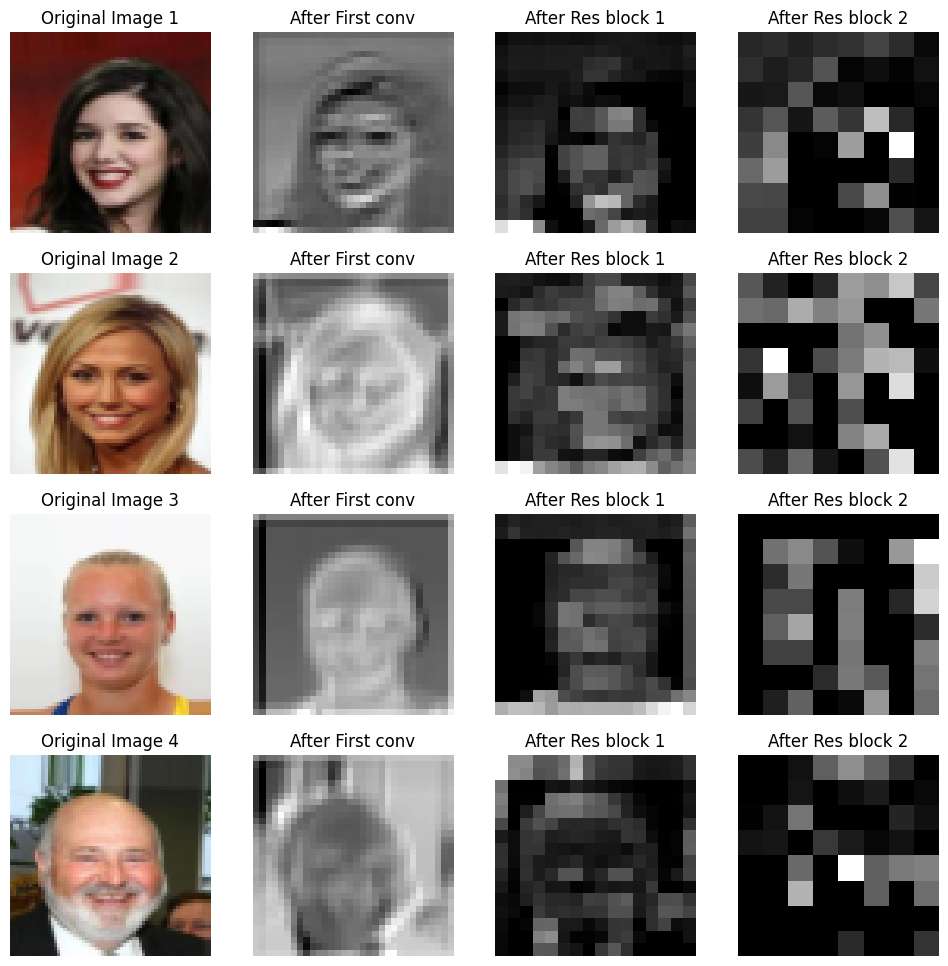}
    \caption{Activation maps of the smile features extracted from the first convolutional layer, first residual block, and second residual block.}
    \label{fig:ResNet_view}
\end{figure}

We have tested feature consistency on CelebA using ResNet18. Instead of  the original ResNet18, which has a kernel size configuration of (64, 128, 256, 512), we utilized a modified version of (16, 32, 64, 128).
The modified ResNet18 was trained with SGD with momentum 0.9, a learning rate of 0.1, and a batch size of 64.
For feature visualization, many methods are available, as listed at \url{https://github.com/justinbellucci/cnn-visualizations-pytorch}. 
We employed activation map visualization, processing images at each layer and visualizing the respective image representation. 
In other words, activation map visualization demonstrates what the image looks like after the application of each filter. 
Figure \ref{fig:ResNet_view} shows 
that the sublinear CNN works well for feature learning.

\section{Conclusion} \label{conclusionsection}

We study sublinear DNNs and prove that, in the large-sample limit, they achieve universal approximation and feature-learning consistency for hierarchically compositional functions. We also analyze AlexNet, VGGNet, and ResNet, showing that these deep CNNs are sublinear on their image-classification benchmarks. Because natural images are hierarchically compositional, our results offer a statistical explanation for the strong performance of large-scale deep learning models in image processing. Our theory identifies a regime in which consistent prediction is guaranteed for large-scale deep learning models despite possible over-parameterization. \textcolor{black}{Empirically, sublinear DNNs match or outperform wide DNNs in prediction accuracy and are more robust to training hyperparameter settings}.

The theoretical proof of this paper leverages StoNet as a surrogate for the DNN, creating a bridge between linear models and DNNs.
Beyond sublinear DNNs, this approach can be applied to sparse deep learning by extending the sparse learning theory from linear models to DNNs. Additionally, we conjecture that the StoNet could enable the
extension of benign overfitting theory from linear models to super-wide DNNs, leveraging its capability in sufficient dimension reduction \citep{LiangSLiang2022}.
% Benign overfitting has been studied within a regularization framework for linear and logistic regression, see e.g. \cite{Tsigler2020BenignOI} and \cite{Wang2022TightBF}.
% Extending these theories to super-wide DNNs via StoNets will be
% of great interest. 
%a focus of future research.

In summary, this work validates the effectiveness of sublinear DNNs for learning features from hierarchically compositional functions and provides theoretical guidance for designing appropriate network architectures for tasks such as image processing, where hierarchical composition is intrinsic. The main takeaways are: (i) sublinear DNNs achieve feature-learning consistency for hierarchically compositional functions, even when the total number of parameters exceeds the sample size; 
%(ii) they must be sufficiently wide and reasonably deep to capture the compositional structure; 
(ii) although wide DNNs can drive training error near zero, their predictive performance can be sensitive to the optimization algorithms and hyperparameter settings, whereas sublinear DNNs are notably more robust; and (iii) sublinear DNNs comply with neural scaling laws, 
achieve universal approximation for hierarchically compositional 
functions in the large-sample limit, and may extend to other classes 
of functions, 
%beyond those studied in Section~\ref{sect:approx}, 
a direction that merits further study.

\section*{Availability} 

The code used to run the experiments is available at \url{https://github.com/sehwankimstat/Sublinear-DNN}.

 \section*{Acknowledgments}
 Liang's research is supported in part by the NSF grant DMS-2210819 and the NIH grant R01-GM152717. Kim's research is supported by the Global-Learning \& Academic
Research Institution for Master’s and PhD Students, and Postdocs (G-LAMP) Program of the National
Research Foundation of Korea (NRF), funded by the Ministry of Education (No. RS-2025-25442252).

\FloatBarrier

\bibliographystyle{asa}
\bibliography{reference} % Your BibTeX file

\newpage 
% Appendix section
\appendix

\setcounter{section}{0}

\renewcommand{\theassumption}{A\arabic{assumption}}
\providecommand{\theHassumption}{}
\renewcommand{\theHassumption}{A.\arabic{assumption}}
\setcounter{table}{0}
\renewcommand{\thetable}{A\arabic{table}}
\providecommand{\theHtable}{}
\renewcommand{\theHtable}{A.\arabic{table}}
\setcounter{equation}{0}
\renewcommand{\theequation}{A\arabic{equation}}
\providecommand{\theHequation}{}
\renewcommand{\theHequation}{A.\arabic{equation}}
\setcounter{figure}{0}
\renewcommand{\thefigure}{A\arabic{figure}}
\providecommand{\theHfigure}{}
\renewcommand{\theHfigure}{A.\arabic{figure}}
\setcounter{lemma}{0}
\renewcommand{\thelemma}{A\arabic{lemma}}
\providecommand{\theHlemma}{}
\renewcommand{\theHlemma}{A.\arabic{lemma}}

\vspace{0.2in}
\begin{center}
{\bf \large APPENDIX}
\end{center}

\section{Theoretical Proofs} \label{section:proofs}

\subsection{Useful Lemmas}

\begin{lemma}[Coordinatewise preactivation bound]
\label{lem:preactivation-coordinate-bound}
Suppose that \(X\in[0,1]^{d_0}\). Assume that the biases are uniformly
bounded and that the incoming weights of each neuron have uniformly bounded
\(\ell_1\)-norm; that is, there exist constants \(B_b,B_w<\infty\) such that
\[
\max_{l,k}|b_{l,k}|\le B_b,
\qquad
\max_{l,k}\sum_{j=1}^{d_{l-1}}|w_{l,kj}|\le B_w .
\]
Then, for tanh and sigmoid activations, there exists a constant
\(C_{\widetilde Y}<\infty\), independent of \(n,l,k\), such that
\[
\mathbb E\left(\widetilde Y_{l,k}^{(t)}\right)^2
\le
C_{\widetilde Y},
\qquad
l=1,\ldots,h,\quad k=1,\ldots,d_l .
\]
For ReLU activations, the same conclusion holds provided
\(\sup_{l\le h}\sigma_l^2<\infty\).
\end{lemma}

\begin{proof}
We prove the result by induction over the layers. For the first layer,
\[
\widetilde Y_{1,k}
=
b_{1,k}
+
\sum_{j=1}^{d_0}w_{1,kj}X_j .
\]
Since \(X_j\in[0,1]\), we have
\[
|\widetilde Y_{1,k}|
\le
|b_{1,k}|
+
\sum_{j=1}^{d_0}|w_{1,kj}||X_j|
\le
B_b+B_w.
\]
Hence
\[
\mathbb E(\widetilde Y_{1,k})^2
\le
(B_b+B_w)^2.
\]

Now suppose that, for some \(l\ge2\), the previous hidden variables satisfy
\[
\sup_j \mathbb E\{Y_{l-1,j}^{2}\}\le C_{l-1}
\]
for a constant \(C_{l-1}<\infty\). Write
\[
\widetilde Y_{l,k}
=
b_{l,k}
+
\sum_{j=1}^{d_{l-1}}
w_{l,kj}\Psi(Y_{l-1,j}).
\]
For tanh and sigmoid, \(|\Psi(x)|\le 1\), and therefore
\[
|\widetilde Y_{l,k}|
\le
B_b+B_w.
\]
Thus
\[
\mathbb E(\widetilde Y_{l,k})^2
\le
(B_b+B_w)^2.
\]

For ReLU, \(\Psi(x)=x_+\), so \(|\Psi(x)|\le |x|\). By Cauchy's inequality in
weighted form,
\[
\left(
\sum_{j=1}^{d_{l-1}}
|w_{l,kj}|\,|\Psi(Y_{l-1,j})|
\right)^2
\le
\left(\sum_{j=1}^{d_{l-1}}|w_{l,kj}|\right)
\left(\sum_{j=1}^{d_{l-1}}|w_{l,kj}|\Psi(Y_{l-1,j})^2\right).
\]
Taking expectations gives
\[
\mathbb E
\left(
\sum_{j=1}^{d_{l-1}}
|w_{l,kj}|\,|\Psi(Y_{l-1,j})|
\right)^2
\le
B_w^2 C_{l-1}.
\]
Therefore,
\[
\mathbb E(\widetilde Y_{l,k})^2
\le
2B_b^2+2B_w^2C_{l-1}.
\]
Since
\[
Y_{l,k}
=
\widetilde Y_{l,k}+e_{l,k},
\qquad
e_{l,k}\sim N(0,\sigma_l^2),
\]
we also have
\[
\mathbb E(Y_{l,k})^2
\le
2\mathbb E(\widetilde Y_{l,k})^2+2\sigma_l^2.
\]
Thus, if \(\sup_l\sigma_l^2<\infty\), the second moments remain uniformly
bounded by induction. This proves the claim.
\end{proof}

%  \begin{lemma} \label{lemma:eigen} 
% Consider a random matrix $\mathbb{U}\in \mathbb{R}^{n\times d}$ with $n \geq d$.  Suppose that the 
% eigenvalues of $\mathbb{U}^\top \mathbb{U}$ are upper bounded, i.e., $\lambda_{\max}(\mathbb{U}^\top\mathbb{U}) \leq \kappa_{\max}$ for 
% some constant $\kappa_{\max}>0$. 
% Let $\Psi(\mathbb{U})$ denote an elementwise transformation of $\mathbb{U}$. Then 
% %\begin{equation} \label{eigenU:eq}
% $\lambda_{\max}\left( (\Psi(\mathbb{U}))^\top (\Psi(\mathbb{U})) \right) \leq \kappa_{\max}$
% %\end{equation}
% for the {\it tanh}, {\it sigmoid} and ReLU transformations. 
% \end{lemma} 
% \begin{proof}  
% For ReLU, the result follows from Lemma 5 of \cite{Dittmer2018SingularVF}. For 
% {\it tanh} and {\it sigmoid}, since they are Lipschitz continuous with a Lipschitz constant of 1, Lemma 5 of  \cite{Dittmer2018SingularVF} also applies. 
% \end{proof}

\begin{lemma}[First-order expansion with stochastic remainder]\label{lem:taylor_op}
Let $\Psi:\mathbb{R}\to\mathbb{R}$ be an activation function, which act componentwise on vectors.  Assume $\Psi\in C^2(\mathbb{R})$ with
uniformly bounded second derivatives, i.e., 
\[
\|\Psi''\|_\infty := \sup_{x\in\mathbb{R}}|\Psi''(x)| < \infty .
\]
Fix $l$ and write 
\[
\bY_l=\tilde{\bY}_l+\be_l\in\mathbb{R}^{d_l},
\]
where $\be_l$ is independent of $\tilde{\bY}_l$,
 satisfies $\mathbb{E}[\be_l]=0$ and
$\mathrm{Var}(\be_l)=\sigma_l^2 I_{d_l}$, with $\sigma_l\to 0$.
Then there exists a   remainder vector $\br_l\in\mathbb{R}^{d_l}$ such that
\begin{equation}\label{eq:taylor_expansion}
\Psi(\bY_l)
=
\Psi(\tilde{\bY}_l)
+
\nabla_{\tilde{\bY}_l}\Psi(\tilde{\bY}_l)\circ \be_l
+
\br_l,
\end{equation}
where $\nabla_{\tilde{\bY}_l}\Psi(\tilde{\bY}_l)=(\Psi'(\tilde Y_{l,1}),\ldots,\Psi'(\tilde Y_{l,d_l}))^\top$
and $\circ$ denotes the Hadamard product. Moreover, the remainder satisfies the coordinatewise bound
\begin{equation}\label{eq:remainder_bound_coord}
|r_{l,i}|\le \frac{\|\Psi''\|_\infty}{2}\,e_{l,i}^2,\qquad i=1,\ldots,d_l,
\end{equation}
where $r_{l,i}$ denotes the $i$-th element of $\br_l$ (defined below). 
In particular, if $d_l \sigma_l\to 0$, then   $\|\br_l\|_2 = o_{\mathbb{P}}(\sigma_l)$.
\end{lemma}
\begin{proof}
For each coordinate $i$, apply Taylor's expansion to $\Psi$:
there exists $\eta_{l,i}\in(0,1)$ such that
\[
\Psi(\tilde Y_{l,i}+e_{l,i})
=
\Psi(\tilde Y_{l,i})
+\Psi'(\tilde Y_{l,i})e_{l,i}
+\frac12\Psi''(\tilde Y_{l,i}+\eta_{l,i}e_{l,i})e_{l,i}^2.
\]
Define $r_{l,i}:= \frac12\Psi''(\tilde Y_{l,i}+\eta_{l,i}e_{l,i})e_{l,i}^2$ and stack over $i$ to obtain
\eqref{eq:taylor_expansion}. The bound \eqref{eq:remainder_bound_coord} follows immediately from
$|\Psi''(\cdot)|\le \|\Psi''\|_\infty$.

Next, by \eqref{eq:remainder_bound_coord},
\[
\|\br_l\|_2
\le
\frac{\|\Psi''\|_\infty}{2}\Big(\sum_{i=1}^{d_l} e_{l,i}^4\Big)^{1/2}
\le
\frac{\|\Psi''\|_\infty}{2}\sum_{i=1}^{d_l} e_{l,i}^2
=
\frac{\|\Psi''\|_\infty}{2}\|\be_l\|_2^2,
\]
where we used $(\sum a_i^2)^{1/2}\le \sum |a_i|$ with $a_i=e_{l,i}^2$.
Therefore, for any $\varepsilon>0$,
\[
\mathbb{P}\big(\|\br_l\|_2>\varepsilon\sigma_l\big)
\le
\mathbb{P}\Big(\|\be_l\|_2^2>\frac{2\varepsilon}{\|\Psi''\|_\infty}\sigma_l \Big)
\le
\frac{\mathbb{E}\|\be_l\|_2^2}{(2\varepsilon/\|\Psi''\|_\infty)\sigma_l}
=
\frac{\|\Psi''\|_\infty}{2\varepsilon}\,d_l\sigma_l,
\]
by Markov's inequality and $\mathbb{E}\|\be_l\|_2^2=\mathrm{tr}(\sigma_l^2 I_{d_l})=d_l\sigma_l^2$.
If $d_l \sigma_l\to 0$, the right-hand side tends to $0$, proving $\|\br_l\|_2/\sigma_l \to 0$ in probability,
i.e.\ $\|\br_l\|_2=o_{\mathbb{P}}(\sigma_l)$.
\end{proof}

Define 
$\|A\|_{\mathrm{op}}=\sqrt{\lambda_{\max}(A^\top A)}$, where 
$\lambda_{\max}(\cdot)$ denotes the maximum eigenvalue of a matrix.

\begin{lemma}[Covariance expansion for $\bSigma_l$]\label{lem:Sigma_op}
Assume the conditions of Lemma~\ref{lem:taylor_op} hold, and define 
$\bSigma_l:=\Var\!\big(\Psi(\bY_l)\big)\in\mathbb{R}^{d_l\times d_l}$.
Assume in addition that $\be_l$ has independent coordinates with $\E(e_{l,i})=0$,
$\Var(e_{l,i})=\sigma_l^2$, and $\E(e_{l,i}^4)\le C_e\sigma_l^4$ for some constant $C_e<\infty$, 
and that $\|\Psi'\|_\infty<\infty$.
Let $U_l=\E[\br_l|\tilde{\bY}_l]$ and $G(\tilde{\bY}_l)=\Psi(\tilde{\bY}_l)+U_l$. 
If $d_l\sigma_l\to 0$, then there exists a  remainder matrix
$R_{l,n}\in\mathbb{R}^{d_l\times d_l}$ such that
\begin{equation} \label{eq:Sigmabound}
\bSigma_l
=
\Var\!\big(G(\tilde{\bY}_l)\big)
+
\diag\!\left\{\sigma_l^2\,\E\!\big[(\nabla_{\tilde{\bY}_l}\Psi(\tilde{\bY}_l))\circ
(\nabla_{\tilde{\bY}_l}\Psi(\tilde{\bY}_l))\big]\right\}
+ R_{l,n},
\end{equation}
where $R_{l,n}$ satisfies $\|R_{l,n}\|_{\mathrm{op}}=o(\sigma_l^2)$.  
\end{lemma}

\begin{proof} 
Fix $l$. 
%Let $\bY_l=\tilde{\bY}_l+\be_l$ with $\be_l\perp \tilde{\bY}_l$, $\E(\be_l)=0$, $\Var(\be_l)=\sigma^2 I_{d_l}$, and assume $\|\Psi'\|_\infty<\infty$.
By Lemma~\ref{lem:taylor_op}, there exists a measurable remainder $\br_l\in\mathbb{R}^{d_l}$ such that
\begin{equation}\label{eq:psi-decomp-opt3}
\Psi(\bY_l)=\Psi(\tilde{\bY}_l)+A_l+\br_l=G(\tilde{\bY}_l)+A_l+\tilde{\br}_l,
\end{equation}
where $A_l:=\nabla_{\tilde{\bY}_l}\Psi(\tilde{\bY}_l)\circ \be_l$
 and $\tilde{\br}_l=\br_l-U_l$. Therefore,  $\E[\tilde{\br}_l \mid \tilde{\bY}]=0$. 

% \medskip\noindent
% \textbf{Step 1: Apply the law of total variance.}
%Let $\bSigma_l:=\Var(\Psi(\bY_l))$. 
By the law of total variance,
\begin{equation}\label{eq:LTV-opt3}
\bSigma_l
=
\Var\!\Big(\E[\Psi(\bY_l)\mid \tilde{\bY}_l]\Big)
+\E\!\Big(\Var(\Psi(\bY_l)\mid \tilde{\bY}_l)\Big).
\end{equation}
Since $\E(A_l\mid \tilde{\bY}_l)=0$ and $\E(\tilde \br_l \mid \tilde{\bY}_l)=0$, by \eqref{eq:psi-decomp-opt3},
 the \emph{conditional mean term} is
\begin{equation}\label{eq:cond-mean-term-opt3}
\Var\!\Big(\E[\Psi(\bY_l)\mid \tilde{\bY}_l]\Big)=\Var\big(G(\tilde{\bY}_l)\big).
\end{equation}
 
For the \emph{conditional variance term}, conditioning on $\tilde{\bY}_l$ and using 
$\E(A_l\mid \tilde{\bY}_l)=0$ and 
$\E(\tilde{\br}_l\mid \tilde{\bY}_l)=0$,
\[
\Var(\Psi(\bY_l)\mid \tilde{\bY}_l)
=
\Var(A_l+\tilde \br_l\mid \tilde{\bY}_l)
=
\Var(A_l\mid \tilde{\bY}_l)
+\Var(\tilde \br_l\mid \tilde{\bY}_l)
+2\Cov(A_l,\tilde \br_l\mid \tilde{\bY}_l).
\]

\smallskip\noindent
\emph{(i)}
For $\Var(A_l\mid \tilde{\bY}_l)$,  taking expectations yields
\begin{equation}\label{eq:main-diag-opt3}
\E[\Var(A_l\mid \tilde{\bY}_l)]
=
\diag\!\left\{\sigma_l^2\,\E\big[(\nabla_{\tilde{\bY}_l}\Psi(\tilde{\bY}_l))\circ
(\nabla_{\tilde{\bY}_l}\Psi(\tilde{\bY}_l))\big]\right\}.
\end{equation}
 
%We use the standard inequalities $\|\Var(Z)\|_{\mathrm{op}}\le \E\|Z\|_2^2$ and $\|\Cov(X,Z)\|_{\mathrm{op}}\le \sqrt{\E\|X\|_2^2}\sqrt{\E\|Z\|_2^2}$.

\smallskip\noindent
\emph{(ii)} 
%Bound $\big\|\E[\Var(\tilde r\mid \tilde{\bY})]\big\|_{\mathrm{op}}$.}
By the law of iterated expectations and $\|\Var(Z)\|_{\mathrm{op}}\le \E\|Z\|_2^2$,
\[
\Big\|\E[\Var(\tilde{\br}_l\mid \tilde{\bY}_l)]\Big\|_{\mathrm{op}}
\le \E\|\tilde{\br}_l\|_2^2
\le \E\|\br_l\|_2^2.
\]
Using the coordinatewise bound from Lemma~\ref{lem:taylor_op},
$|r_{l,i}|\le (\|\Psi''\|_\infty/2)e_{l,i}^2$, and the moment assumption $\E(e_{l,i}^4)\le C_e\sigma_l^4$,
we get
\begin{equation}\label{eq:Er2-opt3}
\E\|\br_l\|_2^2
=\sum_{i=1}^d \E(r_{l,i}^2)
\le \frac{\|\Psi''\|_\infty^2}{4}\sum_{i=1}^d \E(e_{l,i}^4)
\le C\,d_l\,\sigma_l^4,
\end{equation}
hence
\begin{equation}\label{eq:Var-tilder-opt3}
\Big\|\E[\Var(\tilde{\br}_l\mid \tilde{\bY}_l)]\Big\|_{\mathrm{op}}
=O(d_l \sigma_l^4)=o(\sigma_l^3)\qquad\text{if } d_l\sigma_l\to0.
\end{equation}

\smallskip\noindent
\emph{(iii)}
%Bound $\big\|\E[\Cov(A_l,\tilde{\br}_l\mid \tilde{\bY}_l)]\big\|_{\mathrm{op}}$.}
Using $\|\Cov(U,V)\|_{\mathrm{op}}\le \E\|U\|_2\|V\|_2$ and then Cauchy--Schwarz,
\[
\Big\|\E[\Cov(A_l,\tilde{\br}_l\mid \tilde{\bY}_l)]\Big\|_{\mathrm{op}}
\le \E\|A_l\|_2\|\tilde{\br}_l\|_2
\le \sqrt{\E\|A_l\|_2^2}\,\sqrt{\E\|\tilde{\br}_l\|_2^2}
\le \sqrt{\E\|A_l\|_2^2}\,\sqrt{\E\|\br_l\|_2^2}.
\]
Moreover,
\[
\E\|A_l\|_2^2=\sum_{i=1}^d \E\big[\Psi'(\tilde{Y}_{l,i})^2 e_i^2\big]
=\sigma_l^2\sum_{i=1}^d \E[\Psi'(\tilde{Y}_{l,i})^2]
\le d_l\,\sigma_l^2\,\|\Psi'\|_\infty^2.
\]
Combining this with \eqref{eq:Er2-opt3} yields
\begin{equation}\label{eq:CovA-tilder-opt3}
\Big\|\E[\Cov(A_l,\tilde \br_l\mid \tilde{\bY}_l)]\Big\|_{\mathrm{op}}
\le \|\Psi'\|_\infty \sqrt{d_l\sigma_l^2} \cdot \sqrt{C d_l\sigma_l^4}
=O(d_l \sigma_l^3)=o(\sigma_l^2)\qquad\text{if } d_l\sigma_l\to0.
\end{equation}
%The same bound holds for $\E[\Cov(\tilde r,A\mid \tilde{\bY})]$.

% \medskip\noindent
% \textbf{Step 3: Combine terms and define $R_{l,n}$.}

Plugging \eqref{eq:cond-mean-term-opt3} and the bounds \eqref{eq:main-diag-opt3}, \eqref{eq:Var-tilder-opt3},
\eqref{eq:CovA-tilder-opt3} into \eqref{eq:LTV-opt3} yields
(\ref{eq:Sigmabound}), which concludes the proof. 
% \[
% \bSigma_l
% =
% \Var( G(\tilde{\bY}_l))
% +
% \diag\!\left\{\sigma_l^2\,\E\big[(\nabla_{\tilde{\bY}_l}\Psi(\tilde{\bY}_l))\circ
% (\nabla_{\tilde{\bY}_l}\Psi(\tilde{\bY}_l))\big]\right\}
% +R_{l,n},
% \]
% where $R_{l,n}$ is the sum of all remainder terms, and the preceding bounds imply
% $\|R_{l,n}\|_{\mathrm{op}}=o(\sigma_l^2)$.
%This completes the proof.
\end{proof}

 \begin{remark}
Recall $G(\tilde{\bY}_l)=\Psi(\tilde{\bY}_l)+U_l$ with $U_l=\E[\br_l\mid \tilde{\bY}_l]$. Then
\begin{equation*}
\Var(G(\tilde{\bY}_l))
= \Var(\Psi(\tilde{\bY}_l))
+\Var(U_l)+2\Cov(\Psi(\tilde{\bY}_l),U_l).
\end{equation*}
% and, under $d_l\sigma_l\to 0$,
% \begin{equation} \label{eq:dominate-term}
% \Var(G(\tilde{\bY}_l))=\Var(\Psi(\tilde{\bY}_l))+o(\sigma_l)
% \quad\text{in }\|\cdot\|_{\mathrm{op}}.
% \end{equation}

By Jensen's inequality, \eqref{eq:Er2-opt3}, and the condition 
$d_l\sigma_l\to 0$, 
\[
\E\|U_l\|_2^2=\E\|\E[\br_l\mid \tilde{\bY}_l]\|_2^2
\le \E\|\br_l\|_2^2 \le C\, d_l\sigma_l^4
=o(\sigma_l^3).
\]
  Hence
\[
\|\Var(U_l)\|_{\mathrm{op}}\le \E\|U_l\|_2^2=o(\sigma_l^3).
\]
Next, using $\|\Cov(X,Z)\|_{\mathrm{op}}\le \sqrt{\E\|X\|_2^2}\sqrt{\E\|Z\|_2^2}$ and the boundedness
of $\Psi$ for $\tanh$ and sigmoid (say $|\Psi(x)|\le B$), we have
\[
\E\|\Psi(\tilde{\bY}_l)\|_2^2\le d_l B^2,
\qquad
\E\|U_l\|_2^2\le C d_l\sigma_l^4,
\]
and thus
\[
\|\Cov(\Psi(\tilde{\bY}_l),U_l)\|_{\mathrm{op}}
\le \sqrt{d_l B^2}\,\sqrt{C d_l\sigma_l^4}
=O(d_l\sigma_l^2)
=o(\sigma_l).
\]
Therefore,
\[
\|\Var(G(\tilde{\bY}_l))-\Var(\Psi(\tilde{\bY}_l))\|_{\mathrm{op}}
\le \|\Var(U_l)\|_{\mathrm{op}} + 2\|\Cov(\Psi(\tilde{\bY}_l),U_l)\|_{\mathrm{op}}
= o(\sigma_l),
\]
that is, 
\begin{equation} \label{eq:dominate-term}
\Var(G(\tilde{\bY}_l))=\Var(\Psi(\tilde{\bY}_l))+o(\sigma_l)
\quad\text{in }\|\cdot\|_{\mathrm{op}}.
\end{equation}
\end{remark}

\begin{lemma}[Uniform derivative lower bound from bounded second moments]
\label{lem:derivative-lower-bound}
Let \(\{Z_{\alpha}:\alpha\in\mathcal A\}\) be a collection of real-valued
random variables satisfying
\[
\sup_{\alpha\in\mathcal A}\mathbb E Z_\alpha^2 \le C_Z<\infty .
\]
Let \(\Psi\) be either the tanh or sigmoid activation. Then there exists a
constant \(c_\Psi>0\), depending only on \(C_Z\) and \(\Psi\), such that
\[
\inf_{\alpha\in\mathcal A}
\mathbb E\left\{\Psi'(Z_\alpha)^2\right\}
\ge
c_\Psi .
\]
\end{lemma}

\begin{proof}
Choose \(R>0\) such that
$\frac{C_Z}{R^2}\le \frac12$.
By Markov's inequality, for every \(\alpha\in\mathcal A\),
\[
P(|Z_\alpha|>R)
\le
\frac{\mathbb E Z_\alpha^2}{R^2}
\le
\frac{C_Z}{R^2}
\le
\frac12.
\]
Hence
\[
P(|Z_\alpha|\le R)\ge \frac12 .
\]
For tanh and sigmoid activations, \(\Psi'\) is continuous and strictly positive
on every compact interval. Therefore,
\[
m_R
:=
\inf_{|z|\le R}|\Psi'(z)|
>0.
\]
It follows that
\[
\begin{aligned}
\mathbb E\{\Psi'(Z_\alpha)^2\}
&\ge
\mathbb E\left[
\Psi'(Z_\alpha)^2\mathbf 1\{|Z_\alpha|\le R\}
\right] \ge
m_R^2 P(|Z_\alpha|\le R) \ge
\frac{m_R^2}{2}.
\end{aligned}
\]
Thus the claim holds with $c_\Psi=\frac{m_R^2}{2}>0$.
\end{proof}

\begin{lemma}[Argmax transfer to the DNN estimator in \(d_{\rm op}\)]
\label{lem:dnn-argmax-dop} 
Suppose Assumptions \ref{ass:1} and 
\ref{ass:2} hold. 
Let
\[
Q_{n,{\rm DNN}}(\btheta)
=
\frac1n
\sum_{i=1}^n
\log \pi_{\rm DNN}(\bY^{(i)}\mid \bX^{(i)},\btheta)
\]
denote the empirical DNN objective. 
% Suppose that
% \[
% \sup_{\btheta\in\Theta}
% \left|
% Q_{n,{\rm DNN}}(\btheta)-Q^*(\btheta)
% \right|
% =o_p(1).
% \]
% Suppose also that \(Q^*\) is separated at \(\btheta^*\) in the
% \(d_{\rm op}\)-metric: for every \(\epsilon>0\), there exists
% \(\delta_\epsilon>0\) such that
% \[
% Q^*(\btheta^*)
% -
% \sup_{\btheta\in\Theta:\,d_{\rm op}(\btheta,\btheta^*)\ge \epsilon}
% Q^*(\btheta)
% \ge
% \delta_\epsilon .
% \]
Let \(\widehat\btheta_{{\rm DNN},n}\) be an approximate maximizer of
\(Q_{n,{\rm DNN}}\), in the sense that
\[
Q_{n,{\rm DNN}}(\widehat\btheta_{{\rm DNN},n})
\ge
\sup_{\btheta\in\Theta}Q_{n,{\rm DNN}}(\btheta)-\eta_n,
\qquad
\eta_n=o_p(1).
\]
Then
\begin{equation} \label{eq:op-transfer}
d_{\rm op}(\widehat\btheta_{{\rm DNN},n},\btheta^*)
\stackrel{p}{\rightarrow}0, \quad \mbox{as $n\to \infty$}.
\end{equation}
\end{lemma}

\begin{proof}
Let
\[
\xi_n
=
\sup_{\btheta\in\Theta}
\left|
Q_{n,{\rm DNN}}(\btheta)-Q^*(\btheta)
\right|.
\]
Under Assumption \ref{ass:1}-(i)\&(ii), by invoking the uniform law of large numbers, we have  \(\xi_n=o_p(1)\); see also \eqref{eq:sameloss2}. 

Fix \(\epsilon>0\). On the event
$A_{n,\epsilon}=\{d_{\rm op}(\widehat\btheta_{{\rm DNN},n},\btheta^*)\ge \epsilon\}$,
the population separation condition gives
\[
Q^*(\widehat\btheta_{{\rm DNN},n})
\le
Q^*(\btheta^*)-\delta_\epsilon.
\]
Hence
\[
Q_{n,{\rm DNN}}(\widehat\btheta_{{\rm DNN},n})
\le
Q^*(\widehat\btheta_{{\rm DNN},n})+\xi_n
\le
Q^*(\btheta^*)-\delta_\epsilon+\xi_n.
\]
On the other hand, by approximate optimality,
\[
Q_{n,{\rm DNN}}(\widehat\btheta_{{\rm DNN},n})
\ge
Q_{n,{\rm DNN}}(\btheta^*)-\eta_n.
\]
Again using the definition of \(\xi_n\),
\[
Q_{n,{\rm DNN}}(\btheta^*)
\ge
Q^*(\btheta^*)-\xi_n.
\]
Therefore,
\[
Q_{n,{\rm DNN}}(\widehat\btheta_{{\rm DNN},n})
\ge
Q^*(\btheta^*)-\xi_n-\eta_n.
\]
Combining the upper and lower bounds yields
\[
Q^*(\btheta^*)-\xi_n-\eta_n
\le
Q^*(\btheta^*)-\delta_\epsilon+\xi_n,
\]
and therefore
\[
\delta_\epsilon\le 2\xi_n+\eta_n.
\]
 Hence
\[
A_{n,\epsilon}
\subseteq
\left\{
2\xi_n+\eta_n\ge \delta_\epsilon
\right\}.
\]
Since \(\delta_\epsilon>0\) is fixed for the given \(\epsilon>0\), and
$2\xi_n+\eta_n=o_p(1)$,
we have
\[
P\left\{
2\xi_n+\eta_n\ge \delta_\epsilon
\right\}
\to0.
\]
Therefore,
\[
P\left\{
d_{\rm op}(\widehat\btheta_{{\rm DNN},n},\btheta^*)\ge \epsilon
\right\}
\le
P\left\{
2\xi_n+\eta_n\ge \delta_\epsilon
\right\}
\to0.
\]
Because \(\epsilon>0\) is arbitrary, (\ref{eq:op-transfer}) holds. 
\end{proof}

\subsection{Proof of Lemma \ref{eigen_hidden_layer}} \label{app:2}

\begin{proof}
For simplicity, we suppress the iteration index \(t\). Let
\[
\widetilde{\bY}_{1}=\bb_1+\bw_1\bX,
\qquad
\widetilde{\bY}_{l}
=
\bb_l+\bw_l\Psi(\bY_{l-1}),
\quad l=2,\ldots,h.
\]
By the StoNet construction,
\[
\bY_l=\widetilde{\bY}_l+\be_l,
\qquad
\be_l\sim N(0,\sigma_l^2 I_{d_l}),
\]
where \(\be_l\) is independent of \(\widetilde{\bY}_l\). Let
\[
\bSigma_l=\operatorname{Cov}\{\Psi(\bY_l)\}.
\]
We prove that there exists a constant \(c>0\) such that
$\lambda_{\min}(\bSigma_l)\ge c\sigma_l^2$.

We consider two cases.

\medskip
\noindent
\textbf{Case 1: tanh and sigmoid activations.}
Assume \(\Psi\) is either tanh or sigmoid. By Lemma~\ref{lem:Sigma_op},
\[
\bSigma_l
=
\operatorname{Var}\{G(\widetilde{\bY}_l)\}
+
\diag\left[
\sigma_l^2
\E\left\{
\Psi'(\widetilde Y_{l,1})^2
\right\},
\ldots,
\sigma_l^2
\E\left\{
\Psi'(\widetilde Y_{l,d_l})^2
\right\}
\right]
+
R_{l,n},
\]
where
$G(\widetilde{\bY}_l)
=
\Psi(\widetilde{\bY}_l)
+
\E(\br_l\mid \widetilde{\bY}_l)$ and 
$\|R_{l,n}\|_{\rm op}=o(\sigma_l^2)$.
The matrix \(\operatorname{Var}\{G(\widetilde{\bY}_l)\}\) is positive
semidefinite.

By Assumption~\ref{ass:preactivation-regularity},
\[
\sup_{l,k,t}
\E\left\{
\left(\widetilde Y_{l,k}^{(t)}\right)^2
\right\}
\le
C_{\widetilde Y}.
\]
Applying Lemma~\ref{lem:derivative-lower-bound} to the collection
$\left\{
\widetilde Y_{l,k}^{(t)}:
l=1,\ldots,h,\ k=1,\ldots,d_l,\ t\ge 1
\right\}$,
there exists a constant \(c_\Psi>0\), independent of \(n,l,k\), and \(t\),
such that
$\inf_{l,k,t}
\E\left\{
\Psi'\left(\widetilde Y_{l,k}^{(t)}\right)^2
\right\} \ge c_\Psi$.
Therefore,
\[
\lambda_{\min}
\left[
\diag\left\{
\sigma_l^2
\E\left[
\Psi'\left(\widetilde Y_{l,k}\right)^2
\right]
\right\}_{k=1}^{d_l}
\right]
\ge
c_\Psi\sigma_l^2 .
\]
By Weyl's inequality,
\[
\begin{aligned}
\lambda_{\min}(\bSigma_l)
&\ge
\lambda_{\min}
\left[
\diag\left\{
\sigma_l^2
\E\left[
\Psi'\left(\widetilde Y_{l,k}\right)^2
\right]
\right\}_{k=1}^{d_l}
\right]
-
\|R_{l,n}\|_{\rm op} \\
&\ge
c_\Psi\sigma_l^2-o(\sigma_l^2).
\end{aligned}
\]
Hence, for all sufficiently large \(n\),
$\lambda_{\min}(\bSigma_l)
\ge \frac{c_\Psi}{2}\sigma_l^2$.

\medskip
\noindent
\textbf{Case 2: ReLU activation.} Without loss of generality, let's work under the scalar setting. 
Let \(Y=\tilde{Y}+e\), where \(e\sim N(0,\sigma^2)\) is independent of \(\tilde{Y}\),
and \(\Psi(y)=y_{+}:=\max\{y,0\}\). 
For simplicity, we suppress indices and work component-wisely.
The following exact truncated–normal identities hold for
\(u:=\tilde{Y}/\sigma\):
\begin{align}
\label{eq:relu-mean-exact}
\mathbb{E}\!\left[\Psi(Y)\mid \tilde{Y}\right]
&= \tilde{Y}\,\Phi(u)+\sigma\,\phi(u),\\
\label{eq:relu-second-exact}
\mathbb{E}\!\left[\Psi(Y)^2\mid \tilde{Y}\right]
&= \big(\tilde{Y}^2+\sigma^2\big)\,\Phi(u)+\tilde{Y}\sigma\,\phi(u),
\end{align}
where $\Phi$ and $\phi$ denote, respectively, the 
CDF and PDF of the standard normal distribution.

\medskip\noindent
\textbf{Conditional mean correction.}
Define  
\begin{equation} \label{eq:meancorrection}
r(\tilde{Y})\;:=\;\mathbb{E}\!\left[\Psi(Y)\mid \tilde{Y}\right]-\Psi(\tilde{Y})
=\sigma\,\phi(u)+\tilde{Y}\big(\Phi(u)-\mathbf{1}\{\tilde{Y}>0\}\big).
\end{equation}
% In particular, \(|r(\tilde{Y})|\le r(0)=\sigma/\sqrt{2\pi}\), hence
% \begin{equation}\label{eq:relu-r-variance}
% \Var\!\big(r(\tilde{Y})\big)\;\le\;\mathbb{E}[r(\tilde{Y})^2]\;\le\;\sigma^2/(2\pi)
% \;=\;O(\sigma^2).
% \end{equation}

% Let $u=\tilde Y/\sigma$ and recall
% \[
% r(\tilde Y)=\mathbb{E}[(\tilde Y+\sigma Z)_+\mid \tilde Y]-\tilde Y_+
% =\sigma\,\phi(u)+\tilde Y\big(\Phi(u)-\mathbf 1\{\tilde Y>0\}\big).
% \]
%\paragraph{(i) Evenness and nonnegativity.}
%Using $\phi(-u)=\phi(u)$ and $\Phi(-u)=1-\Phi(u)$, 
In what follows, we show $r(\tilde{Y})$ is nonnegative and symmetric about 0. In particular,  
\[
r(-\tilde Y)
=\sigma\phi(-u)-\tilde Y\,\Phi(-u)
=\sigma\phi(u)+\tilde Y\big(\Phi(u)-1\big)=r(\tilde Y),
\]
so $r$ is symmetric about 0. Furthermore, since $x\mapsto x_+$ is convex, 
we have 
\[
\mathbb{E}[(\tilde Y+\sigma Z)_+\mid \tilde Y]\ge (\mathbb{E}[\tilde Y+\sigma Z\mid\tilde Y])_+=\tilde Y_+,
\]
by Jensen's inequality.  Therefore,  $r(\tilde Y)\ge 0$.

To find the maximum of $r(\tilde{Y})$, we 
write $r(\tilde Y)=\sigma f(u)$ with
\[
f(u)=
\begin{cases}
\phi(u)+u(\Phi(u)-1), & u>0,\\[2pt]
\phi(u)+u\,\Phi(u), & u<0.
\end{cases}
\]
Then, using $\phi'(u)=-u\phi(u)$ and $\Phi'(u)=\phi(u)$, we obtain 
\[
 \frac{d r(\tilde{Y})}{d \tilde{Y}}=f'(u)
=\begin{cases}
\Phi(u)-1<0, & \tilde Y>0,\\[2pt]
\Phi(u)>0, & \tilde Y<0,
\end{cases}
\]
so $r$ is strictly decreasing on $(0,\infty)$ and strictly increasing on $(-\infty,0)$.
Therefore, $r$ attains its global maximum at $\tilde Y=0$, where
\[
r(0)=\sigma\,\phi(0)=\frac{\sigma}{\sqrt{2\pi}}.
\]
Hence, for all $\tilde Y\in\mathbb R$, we have 
\begin{equation} \label{eq:rbound}
0\le r(\tilde Y)\le r(0)=\frac{\sigma}{\sqrt{2\pi}}. 
% \qquad\Longrightarrow\qquad
% \lvert r(\tilde Y)\rvert\le \frac{\sigma}{\sqrt{2\pi}}.
\end{equation}
%\qedhere

\medskip\noindent
\textbf{Conditional variance: scalar bounds.} Let's first derive some scalar bounds for the conditional variance $\Var(\Psi(Y)|\tilde{Y})$. 
Let $\mu=\tilde Y$, $u=\mu/\sigma$, and write
\[
m_1:=\mathbb{E}[\Psi(Y)\mid\tilde Y]=\mu\,\Phi(u)+\sigma\,\phi(u),\quad
m_2:=\mathbb{E}[\Psi(Y)^2\mid\tilde Y]=(\mu^2+\sigma^2)\,\Phi(u)+\mu\sigma\,\phi(u).
\]
Then
\[
\Var\!\big(\Psi(Y)\mid\tilde Y\big)
= m_2 - m_1^2
= \big[(\mu^2+\sigma^2)\,\Phi(u)+\mu\sigma\,\phi(u)\big]
      - \big[\mu\,\Phi(u)+\sigma\,\phi(u)\big]^2.
\]
Expand the square and collect terms:
\[
\begin{aligned}
\Var\!\big(\Psi(Y)\mid\tilde Y\big)
&= \mu^2\Phi(u)+\sigma^2\Phi(u)+\mu\sigma\phi(u)
   -\mu^2\Phi(u)^2 - 2\mu\sigma\Phi(u)\phi(u) - \sigma^2\phi(u)^2 \\
&= \mu^2\big[\Phi(u)-\Phi(u)^2\big]
   + \sigma^2\big[\Phi(u)-\phi(u)^2\big]
   + \mu\sigma\big[\phi(u)-2\Phi(u)\phi(u)\big].
\end{aligned}
\]
Now factor out $\sigma^2$ using $\mu=\sigma u$:
\[
\begin{split} 
\Var\!\big(\Psi(Y)\mid\tilde Y\big)
& = \sigma^2\!\left\{
u^2\Phi(u)\big[1-\Phi(u)\big]
+\big[\Phi(u)-\phi(u)^2\big]
+u\,\phi(u)\big[1-2\Phi(u)\big]\right\} \\ 
 &:=\sigma^2 g(u), \\
 \end{split}
 \]
where 
\[
g(u)=\Phi(u)-\phi(u)^2+u\,\phi(u)\big(1-2\Phi(u)\big)+u^2\Phi(u)\big(1-\Phi(u)\big). 
\]
Taking derivative for $g(u)$ (using $\phi'(u)=-u\phi(u)$ and $\Phi'(u)=\phi(u)$) leads to 
\[
\frac{d g(u)}{du}=2\big(1-\Phi(u)\big)\,\big(\phi(u)+u\,\Phi(u)\big).
\]
Note that
\[
\phi(u)+u\,\Phi(u)
= \mathbb{E}\!\big[(u+Z)\,\mathbf 1\{Z>-u\}\big]
= \int_{-u}^{\infty} (u+z)\,\phi(z)\,dz \;\ge\;0,
\]
since the integrand is nonnegative on $[{-}u,\infty)$. Therefore, $dg(u)/du \ge 0$ for all $u$. That is, $g$ is increasing on $\mR$. It is easy to verify that  $\lim_{u\to -\infty} g(u)=0$ and 
$\lim_{u\to\infty} g(u)=1$,
so we conclude that 
%$g(u)\ge 0$ for all $u$.
  $0< g(u)<1$ for any finite $u \in \mR$.

% It is easy to verify that $g(u)$ is increasing on $\mR=(-\infty,\infty)$  
% and satisfies the properties: 
% \[
%  \lim_{u\to -\infty} g(u)=0, \quad g(0)=\frac{1}{2}-\frac{1}{2\pi}:=c_+, \quad 
%  \lim_{u\to \infty} g(u)=1, 
% \]
% where $c_+ \approx 0.34$. 

% Yes. Since $g'(u)=2(1-\Phi(u))\big(\phi(u)+u\Phi(u)\big)\ge 0$ for all $u$ and
% $\lim_{u\to-\infty}g(u)=0$, $\lim_{u\to\infty}g(u)=1$, we have
% \[
% g(u)>0\quad\text{for every finite }u,\qquad
% g(u)=0\ \text{only in the limit }u\to-\infty.
% \]
Thus, for any finite $\tilde Y$,
\[
\sigma^2 > \Var(\Psi(Y)\mid \tilde Y)=\sigma^2 g(\tilde Y/\sigma)>0.
\]

For the unconditional bound, write $U=\tilde Y/\sigma$. By monotonicity,
for any threshold $s \in\mathbb R$,
\begin{equation} \label{tageq}
\sigma^2 \geq \mathbb E\!\big[\Var(\Psi(Y)\mid \tilde Y)\big]
=\sigma^2\,\mathbb E[g(U)]
\;\ge\;\sigma^2\,g(s)\,\mathbb P(U\ge s).
\end{equation}

By the ReLU active-region condition in
Assumption~\ref{ass:preactivation-regularity}, there exist a constant 
\(s\in\mathbb R\) and \(\pi_0>0\), independent of \(n,l,k\), and \(t\), such that $P(U\geq s) \geq \pi_0$. 
% Therefore,
% \[
% \begin{aligned}
% \E\left[
% g\left(
% \frac{\widetilde Y_{l,k}^{(t)}}{\sigma_l}
% \right)
% \right]
% &\ge
% g_a
% \mathbb P\left(
% \frac{\widetilde Y_{l,k}^{(t)}}{\sigma_l}\ge a
% \right)  \\
% &\ge
% g_a\pi_0 .
% \end{aligned}
% \]
% Thus
% \[
% \E\left[
% \operatorname{Var}
% \left\{
% (Y_{l,k}^{(t)})_+
% \mid
% \widetilde Y_{l,k}^{(t)}
% \right\}
% \right]
% \ge
% g_a\pi_0\sigma_l^2 .
% \]
% ++++++++++++
% Therefore, a strictly positive lower bound follows if there exists a finite bound 
% $s \in (-\infty,\infty)$ with $\mathbb P(\tilde Y\ge s \sigma)>0$. 
% Such a finite bound $s$ exits following from the conditions: 
% $\Theta$ is compact by assumption \ref{ass:1}, $\bX \in [0,1]^p$ by Assumption \ref{ass:3}, and $\sigma_l^2$ are bounded for $l=1,2,\ldots, h$ by Assumption \ref{ass:1}-(v). 
Hence, there exists a constant $c_+=g(s)\pi_0>0$ such that 
\begin{equation} \label{cebound0}
\mathbb E\!\big[\Var(\Psi(Y)\mid \tilde Y)\big] \geq c_+ \sigma^2.
\end{equation}

% \paragraph{Unconditional variance (scalar).}
By the law of total variance and the conditional mean correction formula (\ref{eq:meancorrection}), 
\begin{equation} \label{eq:totalPsiY}
\begin{split} 
\Var\!\big(\Psi(Y)\big)
&=\Var\!\big(\mathbb{E}[\Psi(Y)\mid \tilde{Y}]\big)
  + \mathbb{E}\!\big[\Var(\Psi(Y)\mid \tilde{Y})\big]\\
&=\Var\!\big(\Psi(\tilde{Y})+r(\tilde{Y})\big) + \mathbb{E}\!\big[\Var(\Psi(Y)\mid \tilde{Y})\big].
\end{split} 
\end{equation}
By \eqref{cebound0}, we have the lower bound: 
\begin{equation}\label{eq:relu-scalar-lb-ub}
\Var\!\big(\Psi(Y)\big)
\;\ge\;\mathbb{E}\!\big[\Var(\Psi(Y)\mid \tilde{Y})\big]
\;\ge\; c_{+}\,\sigma^2.
\end{equation}
% By (\ref{eq:totalPsiY}) and \eqref{tageq}, we have the upper bound: 
% \begin{equation}\label{eq:relu-scalar-ub}
% \Var\!\big(\Psi(Y)\big)
% \;\le\;\Var\!\big(\Psi(\tilde{Y})\big)+\Var\!\big(r(\tilde{Y})\big)
%       +2\sqrt{\Var(\Psi(\tilde{Y}))\,\Var(r(\tilde{Y}))}+\sigma^2.
% \end{equation}

\medskip\noindent
\textbf{Eigenvalues of the covariance matrix at layer \(l\).}
Let \(\bSigma_l=\Cov\!\big(\Psi(\bY_l)\big)\).
Since different components of \(\be_l\) are mutually independent
with variance \(\sigma_{l}^2\), (\ref{eq:relu-scalar-lb-ub}) implies 
\[
\lambda_{\min}(\bSigma_l)\ge  c_{+}\,\sigma_{l}^2,
\]
which completes the proof. 
\end{proof}

\subsection{Proof of Lemma \ref{lemma:logistic}} \label{app:3}

\begin{proof}
Consider the multinomial logistic regression model with \(m+1\) classes. Let
\(\bx^{(i)}\in\mathbb R^p\) denote the covariate vector of observation \(i\), where
\(p\le n\) may increase with \(n\). Let
\[
\bpi^{(i)}
=
(\pi_0^{(i)},\pi_1^{(i)},\ldots,\pi_m^{(i)})^\top
\]
denote the class-probability vector, where
\[
\pi_j^{(i)}
=
\frac{\exp\{\bbeta_j^\top\bx^{(i)}\}}
{1+\sum_{k=1}^m \exp\{\bbeta_k^\top\bx^{(i)}\}},
\qquad j=0,1,\ldots,m,
\]
with the convention \(\bbeta_0={\bf 0}\). Let
\[
\vec{\bB}
=
(\bbeta_1^\top,\ldots,\bbeta_m^\top)^\top
\in\mathbb R^{mp}.
\]

For observation \(i\), the negative Hessian of the log-likelihood with respect to
\(\vec{\bB}\) is
\[
H^{(i)}
:=
-\nabla_{\vec{\bB}}^2 L^{(i)}
=
K^\top
\left\{
\Lambda_{\pi^{(i)}}-\bpi^{(i)}\bpi^{(i)\top}
\right\}
K
\otimes
\left(\bx^{(i)}\bx^{(i)\top}\right)
:=
A^{(i)}\otimes
\left(\bx^{(i)}\bx^{(i)\top}\right),
\]
where
\[
\Lambda_{\pi^{(i)}}
=
\operatorname{diag}\{\pi_0^{(i)},\pi_1^{(i)},\ldots,\pi_m^{(i)}\}, \quad \mbox{and} \ 
K=
\begin{pmatrix}
{\bf 0}_m^\top\\
I_m
\end{pmatrix}.
\]
The full negative Hessian is therefore
\[
H=\sum_{i=1}^n H^{(i)}
=
\sum_{i=1}^n
A^{(i)}\otimes
\left(\bx^{(i)}\bx^{(i)\top}\right).
\]

We first establish a uniform lower bound for \(A^{(i)}\). Let
\[
\bp^{(i)}
=
(\pi_1^{(i)},\ldots,\pi_m^{(i)})^\top,
\qquad
D^{(i)}
=
\operatorname{diag}\{\pi_1^{(i)},\ldots,\pi_m^{(i)}\}.
\]
Under the baseline parameterization \(\bbeta_0={\bf 0}\), we have
\[
A^{(i)}
=
K^\top
\left\{
\Lambda_{\pi^{(i)}}-\bpi^{(i)}\bpi^{(i)\top}
\right\}
K
=
D^{(i)}-\bp^{(i)}\bp^{(i)\top}.
\]
For any \(\ba=(a_1,\ldots,a_m)^\top\in\mathbb R^m\),
\[
\begin{split}
\ba^\top A^{(i)}\ba
&=
\sum_{j=1}^m \pi_j^{(i)}a_j^2
-
\left(\sum_{j=1}^m \pi_j^{(i)}a_j\right)^2 .
\end{split}
\]
Let
$s_i=\sum_{j=1}^m\pi_j^{(i)}=1-\pi_0^{(i)}$.
By the Cauchy--Schwarz inequality,
\[
\left(\sum_{j=1}^m\pi_j^{(i)}a_j\right)^2
\le
\left(\sum_{j=1}^m\pi_j^{(i)}\right)
\left(\sum_{j=1}^m\pi_j^{(i)}a_j^2\right)
=
s_i
\sum_{j=1}^m\pi_j^{(i)}a_j^2 .
\]
Therefore,
\[
\begin{split}
\ba^\top A^{(i)}\ba
&\ge
(1-s_i)
\sum_{j=1}^m\pi_j^{(i)}a_j^2  
=
\pi_0^{(i)}
\sum_{j=1}^m\pi_j^{(i)}a_j^2  
\ge
\pi_0^{(i)}
\left(\min_{1\le j\le m}\pi_j^{(i)}\right)
\|\ba\|_2^2 .
\end{split}
\]

We next show that the class probabilities are bounded away from zero with probability
tending to one. Let
\[
\mathcal X_n(E)
=
\prod_{r=1}^p[\mu_r-E,\mu_r+E],
\]
and define the event
\[
\mathcal E_n(E)
=
\left\{
X_{ir}\in[\mu_r-E,\mu_r+E],
\quad
1\le i\le n,\ 1\le r\le p
\right\}.
\]
Since \(X_{ir}\sim N(\mu_r,\varsigma_n^2)\) with
\(\varsigma_n^2=n^{-\alpha}\), a Gaussian tail bound gives
\[
P\{\mathcal E_n(E)^c\}
\le
2np\exp\left(-\frac{E^2}{2\varsigma_n^2}\right)
=
2np\exp\left(-\frac12E^2n^\alpha\right)
\to0,
\]
because \(p\le n\).
Define
\[
\pi_*
=
\inf_n
\inf_{\bx\in\mathcal X_n(E)}
\min_{0\le j\le m}
\frac{\exp\{\bbeta_j^\top\bx\}}
{\sum_{k=0}^m\exp\{\bbeta_k^\top\bx\}},
\qquad \bbeta_0={\bf 0}.
\]
We assume that \(\pi_*>0\). Equivalently, on the high-probability covariate region
\(\mathcal X_n(E)\), all class probabilities are uniformly bounded away from zero.
This condition is satisfied, for example, if the linear predictors
\(\bbeta_j^\top\bx\) are uniformly bounded on \(\mathcal X_n(E)\).

On the event \(\mathcal E_n(E)\), we have
\[
\min_{1\le i\le n}\min_{0\le j\le m}\pi_j^{(i)}
\ge
\pi_* .
\]
Hence, on \(\mathcal E_n(E)\),
\[
\ba^\top A^{(i)}\ba
\ge
\pi_*^2\|\ba\|_2^2,
\qquad i=1,\ldots,n.
\]
Thus
\[
A^{(i)}\succeq \nu_0 I_m,
\qquad
\nu_0=\pi_*^2>0,
\qquad i=1,\ldots,n.
\]

We now lower-bound the full Hessian \(H\). For any
$\mathbf b=(\bb_1^\top,\ldots,\bb_m^\top)^\top\in\mathbb R^{mp}$,
define
\[
\mathbf z_i
=
(\bx^{(i)\top}\bb_1,\ldots,\bx^{(i)\top}\bb_m)^\top
\in\mathbb R^m.
\]
Then, on \(\mathcal E_n(E)\),
\[
\begin{split}
\mathbf b^\top H\mathbf b
&=
\sum_{i=1}^n
\mathbf z_i^\top A^{(i)}\mathbf z_i  
\ge \nu_0 \sum_{i=1}^n
\|\mathbf z_i\|_2^2  =
\nu_0
\sum_{i=1}^n
\sum_{j=1}^m
(\bx^{(i)\top}\bb_j)^2  
=
\nu_0
\sum_{j=1}^m
\bb_j^\top
\mathbb X^\top\mathbb X
\bb_j .
\end{split}
\]
By the eigenvalue condition on \(\mathbb X^\top\mathbb X\),
$\lambda_{\min}(\mathbb X^\top\mathbb X)
\ge
n\kappa_{\min}$.
Therefore,
\[
\begin{split}
\mathbf b^\top H\mathbf b
&\ge
n\nu_0\kappa_{\min}
\sum_{j=1}^m\|\bb_j\|_2^2  
=
n\nu_0\kappa_{\min}
\|\mathbf b\|_2^2 .
\end{split}
\]
It follows that, on \(\mathcal E_n(E)\),
\[
\lambda_{\min}(H)
\ge
n\nu_0\kappa_{\min}.
\]
Since \(P\{\mathcal E_n(E)\}\to1\), this lower bound holds with probability tending
to one.

Finally, by the asymptotic normality of the multinomial-logistic MLE, conditional on
\(\mathbb X\),
\[
\mathbb E_{\mathbb Y\mid\mathbb X}
\left[
(\widehat{\vec{\bB}}-\vec{\bB})
(\widehat{\vec{\bB}}-\vec{\bB})^\top
\right]
=
H^{-1}+o\{\|H^{-1}\|_2\}
\]
in operator norm. Hence, with probability tending to one,
\[
\begin{split}
\left\|
\mathbb E_{\mathbb Y\mid\mathbb X}
\left[
(\widehat{\vec{\bB}}-\vec{\bB})
(\widehat{\vec{\bB}}-\vec{\bB})^\top
\right]
\right\|_2
&\le
(1+o(1))\|H^{-1}\|_2  
=\frac{1+o(1)}{\lambda_{\min}(H)}  
\le \frac{1+o(1)}{n\nu_0\kappa_{\min}}.
\end{split}
\]
Equivalently, for every unit vector \(\bu\in\mathbb R^{mp}\),
\[
\mathbb E_{\mathbb Y\mid\mathbb X}
\left[
\left\{\bu^\top
(\widehat{\vec{\bB}}-\vec{\bB})\right\}^2
\right]
\le
\frac{1+o(1)}{n\nu_0\kappa_{\min}},
\]
with probability tending to one.
This proves the lemma after absorbing the factor \((1+o(1))/\nu_0\) into a generic positive constant.
\end{proof}

\subsection{Proof of Lemma \ref{lemma:one-step-op}} \label{proof:lem5}

 \begin{proof}
Fix a layer \(l\), and write
\[
p_l=d_{l-1}+1,
\qquad
\Delta_l=\Delta_l^{(t)}
=
\widehat{\bar\bW}_{l,n}^{(t)}-\bar\bW_{l,*}^{(t)}
\in \mathbb R^{d_l\times p_l}.
\]
The \(k\)th row of \(\Delta_l\) is denoted by 
$\Delta_{lk}^\top$ with
\[
\Delta_{lk}
=
\hat{\bbeta}_{lk}^{(t)}-\bbeta_{lk,*}^{(t)}
\in\mathbb R^{p_l},
\qquad k=1,\ldots,d_l .
\]
Let \(\mathcal F_l\) denote the sigma-field generated by the imputed covariates used in the
layer \(l\) regressions. By  
Lemmas~\ref{lemma:linearR} and~\ref{lemma:logistic}, the coefficient-estimation error for
each neuron-wise regression satisfies
\[
\left\|
E\left[
\Delta_{lk}\Delta_{lk}^{\top}
\mid \mathcal F_l
\right]
\right\|_{\rm op}
\le
\frac{C}{n}\frac{\sigma_l^2}{\sigma_{l-1}^2},
\]
where \(\sigma_0^2=\kappa_{\min}\) is fixed, and its effect is absorbed into the constant \(C\).
Equivalently, for every unit vector \(\bv\in\mathbb R^{p_l}\),
\[
E\left[
(\bv^\top\Delta_{lk})^2
\mid \mathcal F_l
\right]
\le
\frac{C}{n}\frac{\sigma_l^2}{\sigma_{l-1}^2}.
\]
Under the conditional sub-Gaussian version of this bound, there exists a constant \(C_0>0\)
such that
\[
\|\bv^\top\Delta_{lk}\|_{\psi_2\mid\mathcal F_l}
\le
C_0\frac{\sigma_l}{\sigma_{l-1}\sqrt n},
\qquad
\|\bv\|_2=1 .
\]
Set
\[
\tau_l
=
C_0\frac{\sigma_l}{\sigma_{l-1}\sqrt n}.
\]

We now pass from row-wise directional bounds to a matrix operator-norm bound. By definition,
\[
\|\Delta_l\|_{\rm op}
=
\sup_{\bu\in\mathbb S^{d_l-1},\ \bv\in\mathbb S^{p_l-1}}
\bu^\top\Delta_l\bv,
\]
where $\mathbb{S}^{d-1}=\{\bu\in \mR^d: \|\bu\|_2=1\}$ denotes a 
unit sphere in $\mR^d$. 
For fixed unit vectors \(\bu\in\mathbb R^{d_l}\) and
\(\bv\in\mathbb R^{p_l}\),
\[
\bu^\top\Delta_l\bv
=
\sum_{k=1}^{d_l}u_k\,\Delta_{lk}^\top\bv .
\]
Conditional on \(\mathcal F_l\), the random variables
\(\Delta_{lk}^\top\bv\), \(k=1,\ldots,d_l\), are mean-zero sub-Gaussian with
sub-Gaussian norm bounded by \(\tau_l\). Hence, since \(\|\bu\|_2=1\),
\[
\left\|
\bu^\top\Delta_l\bv
\right\|_{\psi_2\mid\mathcal F_l}
\le
C\tau_l
\left(\sum_{k=1}^{d_l}u_k^2\right)^{1/2}
=
C\tau_l .
\]
Therefore, for fixed \((\bu,\bv)\),
\begin{equation} \label{eq:tailbound}
P\left(
|\bu^\top\Delta_l\bv|>x
\mid\mathcal F_l
\right)
\le
2\exp\left(
-\frac{c x^2}{\tau_l^2}
\right)
\end{equation}
for some universal constant \(c>0\).

Next, we construct finite nets of the unit spheres. Let
\(\mathcal N_d\) be a \(1/4\)-net of \(\mathbb S^{d_l-1}\), and let
\(\mathcal N_p\) be a \(1/4\)-net of \(\mathbb S^{p_l-1}\). Such nets can be chosen with the cardinalities   
\[
|\mathcal N_d|\le 9^{d_l},
\qquad
|\mathcal N_p|\le 9^{p_l}.
\]
Indeed, more generally, for the unit sphere
$\mathbb S^{d-1}$,
there exists an \(\varepsilon\)-net \(\mathcal N_\varepsilon\) such that
$|\mathcal N_\varepsilon|
\le
\left(1+\frac{2}{\varepsilon}\right)^d$. 
To see this, take \(\mathcal N_\varepsilon\) to be a maximal \(\varepsilon\)-separated subset
of \(\mathbb S^{d-1}\). By maximality, it is also an \(\varepsilon\)-net. The Euclidean balls
$B\left(\bu,\frac{\varepsilon}{2}\right)$,
 $\bu\in\mathcal N_\varepsilon$,
are disjoint, and they are all contained in the Euclidean ball centered at zero with radius
\(1+\varepsilon/2\). Comparing volumes gives
\[
|\mathcal N_\varepsilon|
\left(\frac{\varepsilon}{2}\right)^d
\operatorname{Vol}(B_d)
\le
\left(1+\frac{\varepsilon}{2}\right)^d
\operatorname{Vol}(B_d),
\]
where $B_d$ denotes the $d$-dimensional unit ball, 
and hence
$|\mathcal N_\varepsilon|
\le
\left(1+\frac{2}{\varepsilon}\right)^d$.
Taking \(\varepsilon=1/4\) yields
$|\mathcal N_\varepsilon|\le 9^d$.
This gives the stated bounds for \(\mathcal N_d\) and \(\mathcal N_p\).

A standard net argument gives
\begin{equation} \label{eq:opbound}
\|\Delta_l\|_{\rm op}
\le
2
\max_{\bu\in\mathcal N_d,\ \bv\in\mathcal N_p}
|\bu^\top\Delta_l\bv|.
\end{equation}
For completeness, we recall the argument. Let \(\bu_0,\bv_0\) be unit vectors such that
$\|\Delta_l\|_{\rm op}=|\bu_0^\top\Delta_l\bv_0|$.
Choose \(\bu\in\mathcal N_d\) and \(\bv\in\mathcal N_p\) such that
$\|\bu-\bu_0\|_2\le\frac14$ and 
$\|\bv-\bv_0\|_2\le\frac14$.
Then
\[
\begin{split}
|\bu_0^\top\Delta_l\bv_0|
&\le
|\bu^\top\Delta_l\bv|
+
|(\bu_0-\bu)^\top\Delta_l\bv_0|
+
|\bu^\top\Delta_l(\bv_0-\bv)|  \\
&\le
\max_{\tilde\bu\in\mathcal N_d,\tilde\bv\in\mathcal N_p}
|\tilde\bu^\top\Delta_l\tilde\bv|
+
\frac14\|\Delta_l\|_{\rm op}
+
\frac14\|\Delta_l\|_{\rm op}.
\end{split}
\]
Thus (\ref{eq:opbound}) holds. 

Using the tail bound \eqref{eq:tailbound} for each fixed pair \((\bu,\bv)\) and applying a union bound over
\(\mathcal N_d\times \mathcal N_p\), we get
\[
\begin{split}
P\left(
\|\Delta_l\|_{\rm op}>2x
\mid \mathcal F_l
\right) & \le
P\left(
\max_{\bu\in\mathcal N_d,\ \bv\in\mathcal N_p}
|\bu^\top\Delta_l\bv|>x
\mid \mathcal F_l
\right)  \\
&\le
2|\mathcal N_d||\mathcal N_p|
\exp\left(
-\frac{c x^2}{\tau_l^2}
\right)  \\
&\le
2\cdot 9^{d_l+p_l}
\exp\left(
-\frac{c x^2}{\tau_l^2}
\right).
\end{split}
\]
Choosing
$x=C_1\tau_l\sqrt{d_l+p_l}$
with \(C_1\) sufficiently large gives
\[
\|\Delta_l\|_{\rm op}
=
O_p\left(\tau_l\sqrt{d_l+p_l}\right).
\]
Recalling that $\tau_l
=C_0\frac{\sigma_l}{\sigma_{l-1}\sqrt n}$ and 
$p_l=d_{l-1}+1$,
we obtain
\[
\left\|
\widehat{\bar\bW}_{l,n}^{(t)}-\bar\bW_{l,*}^{(t)}
\right\|_{\rm op}
=
O_p\left\{
\frac{\sigma_l}{\sigma_{l-1}}
\sqrt{\frac{d_l+d_{l-1}+1}{n}}
\right\}.
\]

Moreover, integrating the above tail bound yields
\[
E\left(\|\Delta_l\|_{\rm op}^2\mid\mathcal F_l\right)
\le
C\tau_l^2(d_l+p_l)
=
\frac{C}{n}
(d_l+d_{l-1}+1)
\frac{\sigma_l^2}{\sigma_{l-1}^2}.
\]
Therefore,
\[
E\left[
\sum_{l=1}^{h+1}
\|\Delta_l\|_{\rm op}^2
\right]
\le
\frac{C}{n}
\sum_{l=1}^{h+1}
(d_l+d_{l-1}+1)
\frac{\sigma_l^2}{\sigma_{l-1}^2}
=
C a_n^2 .
\]
By Markov's inequality,
\[
\sum_{l=1}^{h+1}
\left\|
\widehat{\bar\bW}_{l,n}^{(t)}-\bar\bW_{l,*}^{(t)}
\right\|_{\rm op}^2
=
O_p(a_n^2).
\]
Since \(a_n\to0\), the aggregate bound is \(o_p(1)\). This completes the proof.
\end{proof}

\begin{remark}[A sublinear DNN with many narrow downstream layers]
\label{rem:architecture}
Consider the increasing hidden-layer noise schedule in Remark~\ref{rem:noise}.  
Let $M_n=
\max_{1\le l\le h}
\left\{
(hB_{l,n})^{1/(h+1-l)},
d_l^{2/(h+1-l)}
\right\}$.
Then the condition in \eqref{eq:noise-cond} is satisfied provided that
$T_n\log(n)M_n\prec n$.
  
% Indeed, with this choice of \(r_n\),
% \[
% \frac{r_nT_n}{n}=\frac1{\log n}=o(1),
% \]
% and the remaining requirement is exactly \(M_n\prec r_n\), or equivalently
% \(T_n\log(n)M_n\prec n\).

This condition becomes mild for a tapered architecture with many narrow downstream
layers. Suppose that \(h=h_n\to\infty\), and that only the first \(s=s_n\) hidden
layers are allowed to grow with \(n\), where $s_n\prec h_n$,
while the downstream layers are narrow:
$d_{s+1},d_{s+2},\ldots,d_h,d_{h+1}=O(1)$.
Assume further that the growing widths are polynomially bounded, say
$d_l=O(n^{\gamma_{l,n}})$, for $l=1,\ldots,s_n$,
with \(\sup_{l\le s_n}\gamma_{l,n}<1\). Then, for \(l\le s_n\),
$B_{l,n}
=
O\left(
n^{\gamma_{l,n}+2\sum_{i=l+1}^{s_n}\gamma_{i,n}}
\right)$,
because the factors \(d_{s+1},\ldots,d_h,d_{h+1}\) are bounded. Hence,
$(hB_{l,n})^{1/(h+1-l)}
=
n^{o(1)}$
whenever \(s_n\prec h_n\) and \(h_n\) grows at most subexponentially. Similarly,
$d_l^{2/(h+1-l)}=n^{o(1)}$.
For \(l>s_n\), all widths involved in \(B_{l,n}\) are bounded, and the only remaining
factor in \(M_n\) is due to \(h\). Thus, if the depth grows subpolynomially,
$h_n=n^{o(1)}$,
then
$M_n=n^{o(1)}$. 
Consequently, the architecture condition reduces to
$T_n\log(n)n^{o(1)}\prec n$.
Equivalently, up to a subpolynomial factor, it is enough to require
\[
T_n\prec \frac{n}{\log n}.
\]
This shows that, by increasing the depth while keeping most downstream layers narrow,
one can allow the first few layers to be very wide while still satisfying
\eqref{eq:noise-cond}. In this regime,
\[
T_n
=
O\left(
h_n-s_n+\sum_{l=1}^{s_n}d_l
\right),
\]
so the sublinear neuron-count requirement can remain mild even when the number of
connection parameters is large. For example, if two adjacent early layers satisfy
\[
d_1=O(n^a),
\qquad
d_2=O(n^b),
\qquad 0<a<1, \quad 0<b<1, \quad 
a+b>1,
\]
then the number of trainable parameters contains the term
$d_1d_2=O(n^{a+b})\succ n$,
so the network is over-parameterized in the usual parameter-count sense. 
% At the same time, if the total width measure \(T_n\) satisfies
% \[
% T_n\log(n)n^{o(1)}\prec n,
% \]
% then the increasing noise schedule
% \[
%\sigma_l^2 = \sigma_{h+1}^2 \left\{\frac{\log(n)T_n}{n}\right\}^{h+1-l}, \qquad l=1,\ldots,h,
% \]
% with fixed \(\sigma_{h+1}=O(1)\), satisfies both Assumption~\ref{ass:1}-(v) and condition~\eqref{eq:an}. Thus, a deep tapered architecture can be sublinear in its number of hidden neurons while still being over-parameterized in its number of connection parameters.
\end{remark}

\subsection{Proof of Lemma \ref{lemma:IRO-op}} \label{proof:lem6}

% \begin{lemma}[IRO consistency in a layerwise operator-norm metric]
% \label{lemma:IRO-op}
% Suppose Assumptions~\ref{ass:1}--\ref{ass:4} hold and \(a_n\to0\), where \(a_n\)
% is defined in Lemma~\ref{lemma:one-step-op}. Let 
% $\btheta=(\bar\bW_1,\ldots,\bar\bW_{h+1})$, where 
% $\bar\bW_l=(\bb_l,\bW_l)$. 
% Define the layerwise operator-norm distance
% \[
% d_{\rm op}(\btheta,\btheta')
% =
% \left(\sum_{l=1}^{h+1}
% \|\bar\bW_l-\bar\bW_l'\|_{\rm op}^2 \right)^{1/2}.
% \]
% Then, for the IRO estimator
% \(\widehat{\btheta}_n^{(t)}\),
% \[
% d_{\rm op}\!\left(\widehat{\btheta}_n^{(t)},\btheta^*\right)
% =
% O_p\!\left((\lambda^*)^t\right)+O_p(a_n),
% \]
% where \(0<\lambda^*<1\) is the local contraction constant as defined 
% in Assumption \ref{ass:4}. Consequently,
% \[
% d_{\rm op}\!\left(\widehat{\btheta}_n^{(t)},\btheta^*\right)
% \stackrel{p}{\rightarrow}0, \quad 
% \mbox{as \(t\to\infty\) and \(n\to\infty\)}.
% \]
% \end{lemma}
 
\begin{proof}
Let \(M(\btheta)\) denote the population IRO update map. By
Lemma~\ref{lemma:one-step-op},
\begin{equation} \label{eq:metric1}
d_{\rm op}\!\left(
\widehat{\btheta}_n^{(t)},M(\widehat{\btheta}_n^{(t-1)})
\right)
=
O_p(a_n).
\end{equation}
By the local contraction condition in Assumption \ref{ass:4}, 
\begin{equation} \label{eq:metric2}
d_{\rm op}\left(M(\widehat{\btheta}_n^{(t-1)}),M(\btheta^*)\right) 
\leq \lambda^* d_{\rm op}\left( \widehat{\btheta}_n^{(t-1)},\btheta^* \right). 
\end{equation}
Indeed, by the mean-value representation, for
$\btheta_s=\btheta^*+s(\btheta-\btheta^*)$, with $0\le s\le 1$,
we have
\[
M(\btheta)-M(\btheta^*)
=
\int_0^1
DM(\btheta_s)[\btheta-\btheta^*]\,ds.
\]
Therefore, if \(\btheta_s\in U(\btheta^*)\) for all \(s\in[0,1]\), then
Assumption~\ref{ass:4} gives
\[
\begin{aligned}
d_{\rm op}\{M(\btheta),M(\btheta^*)\}
&\le
\int_0^1
\|DM(\btheta_s)[\btheta-\btheta^*]\|_{d_{\rm op}}\,ds  \\
&\le
\int_0^1
\lambda^*
\|\btheta-\btheta^*\|_{d_{\rm op}}\,ds  \\
&=
\lambda^*
d_{\rm op}(\btheta,\btheta^*).
\end{aligned}
\]

Since \(M(\btheta^*)=\btheta^*\), combining (\ref{eq:metric1}) and 
(\ref{eq:metric2}) yields 
\[
d_{\rm op}\!\left(\widehat{\btheta}_n^{(t)},\btheta^*\right)
\le
O_p(a_n)
+
\lambda^*
d_{\rm op}\!\left(\widehat{\btheta}_n^{(t-1)},\btheta^*\right).
\]
Iterating the recursion gives
\[
d_{\rm op}\!\left(\widehat{\btheta}_n^{(t)},\btheta^*\right)
\le
(\lambda^*)^t
d_{\rm op}\!\left(\widehat{\btheta}_n^{(0)},\btheta^*\right)
+
\frac{O_p(a_n)}{1-\lambda^*}.
\]
Thus
\[
d_{\rm op}\!\left(\widehat{\btheta}_n^{(t)},\btheta^*\right)
=
O_p\!\left((\lambda^*)^t\right)+O_p(a_n),
\]
which converges to zero as \(t\to\infty\) and \(n\to\infty\).
\end{proof}

\subsection{Proof of Theorem \ref{thm:new1}}

\begin{proof}
We first prove part~(i). By Lemma~\ref{lemma:IRO-op}, 
\[
d_{\rm op}\!\left(\hat{\btheta}_n^{(t)},\btheta^*\right)
\stackrel{p}{\rightarrow}0, \quad \mbox{as $t\to \infty$ and $n\to \infty$}.
\]
In particular, for each hidden layer \(l=1,\ldots,h\),
\[
\|\widehat{\bW}_l^{(t)}-\bW_l^*\|_{\rm op}
\stackrel{p}{\rightarrow}0,
\]
because \(\widehat{\bW}_l^{(t)}-\bW_l^*\) is a submatrix of
\(\widehat{\bar\bW}_l^{(t)}-\bar\bW_l^*\).

Let $\Delta_l^{(t)}=\widehat{\bW}_l^{(t)}-\bW_l^*$.
Then
\[
\begin{split}
\widehat{\bA}_l^{(t)}-\bA_l^*
&=
\widehat{\bW}_l^{(t)\top}\widehat{\bW}_l^{(t)}
-
\bW_l^{*\top}\bW_l^* 
=
\bW_l^{*\top}\Delta_l^{(t)}
+
\Delta_l^{(t)\top}\bW_l^*
+
\Delta_l^{(t)\top}\Delta_l^{(t)} .
\end{split}
\]
Taking operator norms gives
\[
\|\widehat{\bA}_l^{(t)}-\bA_l^*\|_{\rm op}
\le
2\|\bW_l^*\|_{\rm op}\|\Delta_l^{(t)}\|_{\rm op}
+
\|\Delta_l^{(t)}\|_{\rm op}^2 .
\]
By compactness of the parameter space, \(\|\bW_l^*\|_{\rm op}\) is bounded. Therefore,
\begin{equation} \label{eq:A-op}
\|\widehat{\bA}_l^{(t)}-\bA_l^*\|_{\rm op}
\stackrel{p}{\rightarrow}0 .
\end{equation}

% The eigenvalue consistency now follows from Weyl's eigenvalue perturbation inequality. 
Since both \(\widehat{\bA}_l^{(t)}\) and \(\bA_l^*\) are symmetric,  
\[
\max_{1\le j\le d_{l-1}}
\left|
\lambda_j(\widehat{\bA}_l^{(t)})
-
\lambda_j(\bA_l^*)
\right|
\le
\|\widehat{\bA}_l^{(t)}-\bA_l^*\|_{\rm op},
\]
by Weyl's eigenvalue perturbation inequality for Hermitian matrices
\citep[see, e.g.,][]{HornJohnson2013MatrixAnalysis}.
Hence, by \eqref{eq:A-op}, 
\[
\max_{1\le j\le d_{l-1}}
\left|
\lambda_j(\widehat{\bA}_l^{(t)})
-
\lambda_j(\bA_l^*)
\right|
\stackrel{p}{\longrightarrow}0 .
\]

Next, under the eigengap condition
$\lambda_{r_l}(\bA_l^*)-\lambda_{r_l+1}(\bA_l^*)\ge \delta_l>0$,
the Davis-Kahan sin-theta theorem \citep{DavisKahan1970} implies
\[
\left\|
\widehat{\bV}_l^{(t)}\widehat{\bV}_l^{(t)\top}
-
\bV_l^*\bV_l^{*\top}
\right\|_{\rm op}
\le
C
\frac{
\|\widehat{\bA}_l^{(t)}-\bA_l^*\|_{\rm op}
}{
\delta_l
},
\]
where \(C>0\) is a universal constant. By \eqref{eq:A-op}, 
we obtain
\[
\left\|
\widehat{\bV}_l^{(t)}\widehat{\bV}_l^{(t)\top}
-
\bV_l^*\bV_l^{*\top}
\right\|_{\rm op}
\stackrel{p}{\longrightarrow}0 .
\]
Thus the top \(r_l\)-dimensional neural-feature subspace is consistently estimated.

If the \(k\)th eigenvalue is simple and separated from the remaining eigenvalues, the same Davis--Kahan perturbation bound applied to the one-dimensional eigenspace gives
\[
\left\|
\widehat{\bv}_{l,k}^{(t)}\widehat{\bv}_{l,k}^{(t)\top}
-
\bv_{l,k}^*\bv_{l,k}^{*\top}
\right\|_{\rm op}
\stackrel{p}{\longrightarrow}0 .
\]
For one-dimensional subspaces, convergence of the projection matrices is equivalent to convergence of the unit eigenvectors up to sign. Hence there exists
\(s_{l,k}^{(t)}\in\{-1,1\}\) such that
\[
\left\|
\widehat{\bv}_{l,k}^{(t)}
-
s_{l,k}^{(t)}\bv_{l,k}^*
\right\|_2
\stackrel{p}{\longrightarrow}0 .
\]
This proves the eigenvalue and eigenvector consistency claims for the IRO-produced estimator.

We next prove part~(ii). Under the sublinear architecture condition, 
one can choose the hidden-layer noise levels so that the StoNet surrogate satisfies the required noise-scaling conditions and the StoNet and DNN objectives are asymptotically equivalent. By the StoNet--DNN loss equivalence and the population separation condition, the DNN estimator in \eqref{est2} has the same limiting target as the corresponding StoNet estimator, up to loss-invariant transformations. 
By Lemma \ref{lem:dnn-argmax-dop},  
the DNN estimator in \eqref{est2} is also consistent with respect to $\btheta^*$ in operator norm.
Therefore, (\ref{eq:transfer-theta}) holds by the triangular inequality 
\begin{equation} \label{eq:stonet-DNN-transfer}
 d_{\rm op}(\hat{\btheta}_n^{(t)},\widehat{\btheta}_{\rm DNN,n}) \leq    d_{\rm op}(\hat{\btheta}_n^{(t)},\btheta^*) +  d_{\rm op}(\widehat{\btheta}_{\rm DNN,n},\btheta^*)\stackrel{p}{\to} 0. 
\end{equation}
Repeating the same argument as above with
\[
\widehat{\bA}_{{\rm DNN},l}
=
\widehat{\bW}_{{\rm DNN},l}^{\top}
\widehat{\bW}_{{\rm DNN},l}
\]
gives
\[
\|\widehat{\bA}_{{\rm DNN},l}-\bA_l^*\|_{\rm op}
\stackrel{p}{\longrightarrow}0.
\]
Weyl's inequality then yields eigenvalue consistency:
\[
\max_{1\le j\le d_{l-1}}
\left|
\lambda_j(\widehat{\bA}_{{\rm DNN},l})
-
\lambda_j(\bA_l^*)
\right|
\stackrel{p}{\longrightarrow}0.
\]
Under the same eigengap condition, the Davis--Kahan theorem gives
\[
\left\|
\widehat{\bV}_{{\rm DNN},l}
\widehat{\bV}_{{\rm DNN},l}^{\top}
-
\bV_l^*\bV_l^{*\top}
\right\|_{\rm op}
\stackrel{p}{\longrightarrow}0.
\]
If the eigenvalue of interest is simple, the corresponding individual eigenvector is also consistent up to sign. This completes the proof.
\end{proof}

\subsection{Proof of Lemma \ref{lem:Gamma-forward-stability}} 
\label{proof:lem7}

\begin{proof}
Let
\[
\Delta_l=\bar\bW_l-\bar\bW_l^*,
\qquad l=1,\ldots,h+1.
\]
We prove the result by replacing the layers of \(\btheta^*\) with the
corresponding layers of \(\btheta\) one at a time.

For \(l=0,1,\ldots,h+1\), define the hybrid parameter
\[
\btheta^{[l]}
=
(\bar\bW_1,\ldots,\bar\bW_l,
 \bar\bW_{l+1}^*,\ldots,\bar\bW_{h+1}^*).
\]
Thus
\[
\btheta^{[0]}=\btheta^*,
\qquad
\btheta^{[h+1]}=\btheta.
\]
By telescoping,
\[
f_{\btheta}(\bX)-f_{\btheta^*}(\bX)
=
\sum_{l=1}^{h+1}
\left\{
f_{\btheta^{[l]}}(\bX)
-
f_{\btheta^{[l-1]}}(\bX)
\right\}.
\]
Therefore,
\[
\|f_{\btheta}-f_{\btheta^*}\|_{L^2(P_{\bX})}
\le
\sum_{l=1}^{h+1}
\|f_{\btheta^{[l]}}-f_{\btheta^{[l-1]}}\|_{L^2(P_{\bX})}.
\]

We now bound the \(l\)-th telescoping term. The two networks
\(\btheta^{[l]}\) and \(\btheta^{[l-1]}\) have the same layers before layer
\(l\). Hence their input to layer \(l\) is the same:
\[
\bar\bh_{l-1}(\bX;\btheta^{[l]})
=
\bar\bh_{l-1}(\bX;\btheta^{[l-1]}).
\]
The only difference at layer \(l\) is the replacement of
\(\bar\bW_l^*\) by \(\bar\bW_l\). By the Lipschitz property of \(\Psi_l\),
\[
\begin{aligned}
\left\|
\bh_l(\bX;\btheta^{[l]})
-
\bh_l(\bX;\btheta^{[l-1]})
\right\|_2
& \le
L_l
\left\|
(\bar\bW_l-\bar\bW_l^*)
\bar\bh_{l-1}(\bX;\btheta^{[l-1]})
\right\|_2  \\
& \le
L_l
\|\Delta_l\|_{\rm op}
\|\bar\bh_{l-1}(\bX;\btheta^{[l-1]})\|_2 .
\end{aligned}
\]
% Since the augmentation adds the same leading coordinate \(1\) to both hidden
% vectors, the augmented difference has the same norm:
% \[
% \left\|
% \bar\bh_l(\bX;\btheta^{[l]})
% -
% \bar\bh_l(\bX;\btheta^{[l-1]})
% \right\|_2
% =
% \left\|
% \bh_l(\bX;\btheta^{[l]})
% -
% \bh_l(\bX;\btheta^{[l-1]})
% \right\|_2.
% \]

For downstream layers \(j=l+1,\ldots,h+1\), both hybrid networks use the same
parameter matrices \(\bar\bW_j^*\). Applying the Lipschitz bound recursively,
we obtain
\[
\begin{aligned}
&
\left\|
f_{\btheta^{[l]}}(\bX)
-
f_{\btheta^{[l-1]}}(\bX)
\right\|_2 
\le
L_l
\left[
\prod_{j=l+1}^{h+1}
L_j\|\bar\bW_j^*\|_{\rm op}
\right]
\|\Delta_l\|_{\rm op}
\|\bar\bh_{l-1}(\bX;\btheta^{[l-1]})\|_2 .
\end{aligned}
\]
Using the definition of \(K_{j,n}(V)\), this is bounded by
\[
\left\|
f_{\btheta^{[l]}}(\bX)
-
f_{\btheta^{[l-1]}}(\bX)
\right\|_2
\le
L_l
\left[
\prod_{j=l+1}^{h+1}
K_{j,n}(V)
\right]
\|\Delta_l\|_{\rm op}
\|\bar\bh_{l-1}(\bX;\btheta^{[l-1]})\|_2 .
\]
Taking \(L^2(P_{\bX})\)-norms gives
\[
\begin{aligned}
&
\|f_{\btheta^{[l]}}-f_{\btheta^{[l-1]}}\|_{L^2(P_{\bX})}
\le
L_l
\left[
\prod_{j=l+1}^{h+1}
K_{j,n}(V)
\right]
\|\Delta_l\|_{\rm op}
\left[
\mathbb E
\|\bar\bh_{l-1}(\bX;\btheta^{[l-1]})\|_2^2
\right]^{1/2}.
\end{aligned}
\]
Because \(V(\btheta^*)\) is a layerwise product neighborhood, each hybrid
parameter \(\btheta^{[l-1]}\) belongs to \(V(\btheta^*)\). Hence
\[
\mathbb E
\|\bar\bh_{l-1}(\bX;\btheta^{[l-1]})\|_2^2
\le
\tau_{l-1,n}(V).
\]
Therefore,
\[
\|f_{\btheta^{[l]}}-f_{\btheta^{[l-1]}}\|_{L^2(P_{\bX})}
\le
B_{l,n}(V)\|\Delta_l\|_{\rm op},
\]
where
\[
B_{l,n}(V)
=
L_l\tau_{l-1,n}^{1/2}(V)
\prod_{j=l+1}^{h+1}
K_{j,n}(V).
\]

Combining the telescoping bound over all layers yields
\[
\|f_{\btheta}-f_{\btheta^*}\|_{L^2(P_{\bX})}
\le
\sum_{l=1}^{h+1}
B_{l,n}(V)\|\Delta_l\|_{\rm op}.
\]
By Cauchy's inequality,
\[
\begin{aligned}
\sum_{l=1}^{h+1}
B_{l,n}(V)\|\Delta_l\|_{\rm op}
&\le
\left[
\sum_{l=1}^{h+1}B_{l,n}^2(V)
\right]^{1/2}
\left[
\sum_{l=1}^{h+1}\|\Delta_l\|_{\rm op}^2
\right]^{1/2} =
\Gamma_n(V)d_{\rm op}(\btheta,\btheta^*),
\end{aligned}
\]
where
\[
\Gamma_n^2(V)
=
\sum_{l=1}^{h+1}
L_l^2\tau_{l-1,n}(V)
\prod_{j=l+1}^{h+1}K_{j,n}^2(V).
\]
This proves the desired local forward-stability bound.

Finally, if \(V(\btheta^*)\) is chosen sufficiently small, the quantities
\(\tau_{l,n}(V)\) and \(K_{j,n}(V)\) may be bounded locally by their values at
\(\btheta^*\), up to a universal multiplicative constant. This gives
\[
\Gamma_n^2
\lesssim
\sum_{l=1}^{h+1}
L_l^2\tau_{l-1,n}
\prod_{j=l+1}^{h+1}
\{L_j\|\bar\bW_j^*\|_{\rm op}\}^2.
\]
The common-Lipschitz case follows immediately by taking \(L_j=L_\Psi\).
\end{proof}

% \begin{proof}
% Let
% \[
% \Delta_l=\bar\bW_l-\bar\bW_l^*,
% \qquad l=1,\ldots,h+1.
% \]
% For a perturbation at layer \(l\), the resulting output perturbation is
% propagated through layers \(l+1,\ldots,h+1\). By the Lipschitz property of the
% activation functions,
% \[
% \|D_{\bar\bW_l}f_{\btheta}[\Delta_l](\bX)\|_2
% \le
% L_l
% \left[
% \prod_{j=l+1}^{h+1}
% L_j\|\bar\bW_j\|_{\rm op}
% \right]
% \|\Delta_l\|_{\rm op}
% \|\bar\bh_{l-1}(\bX;\btheta)\|_2 .
% \]
% Taking \(L^2(P_{\bX})\)-norms and then taking the supremum over
% \(\btheta\in V(\btheta^*)\) gives
% \[
% \|D_{\bar\bW_l}f_{\btheta}[\Delta_l]\|_{L^2(P_{\bX})}
% \le
% L_l\tau_{l-1,n}^{1/2}(V)
% \left[
% \prod_{j=l+1}^{h+1}K_{j,n}(V)
% \right]
% \|\Delta_l\|_{\rm op}.
% \]
% Summing over layers yields
% \[
% \|f_{\btheta}-f_{\btheta^*}\|_{L^2(P_{\bX})}
% \le
% \sum_{l=1}^{h+1}
% B_{l,n}(V)\|\Delta_l\|_{\rm op},
% \]
% where
% \[
% B_{l,n}(V)
% =
% L_l\tau_{l-1,n}^{1/2}(V)
% \prod_{j=l+1}^{h+1}K_{j,n}(V).
% \]
% By Cauchy's inequality,
% \[
% \sum_{l=1}^{h+1}
% B_{l,n}(V)\|\Delta_l\|_{\rm op}
% \le
% \left\{
% \sum_{l=1}^{h+1}B_{l,n}^2(V)
% \right\}^{1/2}
% \left\{
% \sum_{l=1}^{h+1}\|\Delta_l\|_{\rm op}^2
% \right\}^{1/2}.
% \]
% Since
% \[
% d_{\rm op}(\btheta,\btheta^*)
% =
% \left\{
% \sum_{l=1}^{h+1}\|\Delta_l\|_{\rm op}^2
% \right\}^{1/2},
% \]
% the desired bound follows with
% \[
% \Gamma_n^2(V)
% =
% \sum_{l=1}^{h+1}
% L_l^2\tau_{l-1,n}(V)
% \prod_{j=l+1}^{h+1}K_{j,n}^2(V).
% \]
% \end{proof}

\subsection{Proof of Theorem \ref{thm:pred-consistency}}

%  \begin{theorem}[Prediction consistency] \alebl{thm:pred-consistency}
% Let \(f_{\btheta}\) denote the deterministic DNN map associated with parameter \(\btheta\). 
% Suppose the conditions of Theorem~\ref{thm:new1} hold. Assume further that the network is locally forward-stable around \(\btheta^*\), in the sense that there exists a deterministic sequence \(\Gamma_n\) such that, for all \(\btheta\) in a neighborhood of \(\btheta^*\),
% \[
% \|f_{\btheta}-f_{\btheta^*}\|_{L^2(P_{\bX})}
% \le
% \Gamma_n d_{\rm op}(\btheta,\btheta^*).
% \]
% If
% \begin{equation} \label{eq:Gamma-prediction}
% \Gamma_n\{(\lambda^*)^t+a_n\}\to0,
% \end{equation}
% then
% \begin{equation} \label{eq:pred-consistency}
% \|f_{\hat{\btheta}_n^{(t)}}-f_{\btheta^*}\|_{L^2(P_{\bX})}
% \stackrel{p}{\longrightarrow}0 .
% \end{equation}
% Consequently, under the well-specified regression model
% \[
% Y=f_{\btheta^*}(\bX)+\varepsilon,
% \qquad
% E(\varepsilon\mid\bX)=0,
% \]
% the IRO-produced DNN is prediction consistent.
% \end{theorem}

\begin{proof}
By the local forward-stability condition,
\[
\|f_{\hat{\btheta}_n^{(t)}}-f_{\btheta^*}\|_{L^2(P_{\bX})}
\le
\Gamma_n d_{\rm op}(\hat{\btheta}_n^{(t)},\btheta^*).
\]
Lemma \ref{lemma:IRO-op} gives
\[
d_{\rm op}(\hat{\btheta}_n^{(t)},\btheta^*)
=
O_p\{(\lambda^*)^t\}+O_p(a_n),
\]
and thus, 
\[
\|f_{\hat{\btheta}_n^{(t)}}-f_{\btheta^*}\|_{L^2(P_{\bX})}
=
O_p\!\left(\Gamma_n\{(\lambda^*)^t+a_n\}\right).
\]
Therefore, if \eqref{eq:Gamma-prediction} holds, then the prediction is consistent. 
\end{proof}

\subsection{On the order of $\Gamma_n$} \label{proof:Gamma-order}

\begin{proposition}[Orders of \(\Gamma_n\)]
\label{cor:Gamma-orders}
Suppose \(\tau_{l,n}=O(1)\) for \(l=0,\ldots,h_n\), and define
\[
K_n=
\max_{1\le j\le h_n+1}
L_j\|\bar\bW_j^*\|_{\rm op}.
\]
Then
\[
\Gamma_n^2
=
O\left(
\sum_{m=0}^{h_n}K_n^{2m}
\right).
\]
Consequently,
\[
\Gamma_n=
\begin{cases}
O(1), & K_n\le K<1,\\[4pt]
O(\sqrt{h_n}), & K_n=1,\\[4pt]
O(K_n^{h_n}), & K_n>1 \text{ and bounded away from }1.
\end{cases}
\]
In contrast, under only bounded activations and uniformly bounded spectral
norms, the crude ambient-width bound is
\[
\Gamma_n
=
O\left\{
\left(\sum_{l=0}^{h}d_l\right)^{1/2}
\right\}.
\]
\end{proposition}

\begin{proof}
By Lemma~\ref{lem:Gamma-forward-stability}, evaluated locally at
\(\btheta^*\), we have
\[
\Gamma_n^2
\lesssim
\sum_{l=1}^{h_n+1}
\tau_{l-1,n}
\prod_{j=l+1}^{h_n+1}
\{L_j\|\bar\bW_j^*\|_{\rm op}\}^2 .
\]
Assume that
$\tau_{l,n}=O(1)$ for $l=0,\ldots,h_n$.
Then there exists a constant \(C>0\), independent of \(l\) and \(n\), such that $\tau_{l,n}\le C$.
Therefore,
\[
\Gamma_n^2
\lesssim
\sum_{l=1}^{h_n+1}
\prod_{j=l+1}^{h_n+1}
\{L_j\|\bar\bW_j^*\|_{\rm op}\}^2 .
\]

By the definition of $K_n$, we have 
$L_j\|\bar\bW_j^*\|_{\rm op}\le K_n$ 
for every $j$.
Hence
\[
\prod_{j=l+1}^{h_n+1}
\{L_j\|\bar\bW_j^*\|_{\rm op}\}^2
\le
\prod_{j=l+1}^{h_n+1} K_n^2.
\]
There are \(h_n+1-l\) factors in this product. Thus
$\prod_{j=l+1}^{h_n+1} K_n^2
=
K_n^{2(h_n+1-l)}$.
Consequently,
\[
\Gamma_n^2
=
O\left(
\sum_{l=1}^{h_n+1}
K_n^{2(h_n+1-l)}
\right).
\]
Let $m=h_n+1-l$.
As \(l\) ranges from \(1\) to \(h_n+1\), \(m\) ranges from \(h_n\) down to \(0\). Therefore,
$\sum_{l=1}^{h_n+1}
K_n^{2(h_n+1-l)}
=\sum_{m=0}^{h_n}K_n^{2m}$.
Hence
\[
\Gamma_n^2
=
O\left(
\sum_{m=0}^{h_n}K_n^{2m}
\right).
\]

We now consider three cases.
First, suppose \(K_n\le K<1\). Then
\[
\sum_{m=0}^{h_n}K_n^{2m}
\le
\sum_{m=0}^{\infty}K^{2m}
=
\frac{1}{1-K^2}.
\]
Therefore,
$\Gamma_n=O(1)$.

Second, suppose \(K_n=1\). Then
\[
\sum_{m=0}^{h_n}K_n^{2m}
=
\sum_{m=0}^{h_n}1
=
h_n+1.
\]
Therefore, 
$\Gamma_n=O(\sqrt{h_n})$.

Third, suppose \(K_n>1\). Then the geometric sum gives
\[
\sum_{m=0}^{h_n}K_n^{2m}
=
\frac{K_n^{2(h_n+1)}-1}{K_n^2-1}.
\]
If \(K_n\) is bounded away from one from above, namely \(K_n\ge 1+\delta\) for some \(\delta>0\), then
\[
K_n^2-1\ge (1+\delta)^2-1>0.
\]
Hence
\[
\sum_{m=0}^{h_n}K_n^{2m}
=
O\left(K_n^{2(h_n+1)}\right).
\]
Equivalently, up to constants,
\[
\Gamma_n
=
O\left(K_n^{h_n+1}\right).
\]
If \(K_n\) is also uniformly bounded above, this is the same order as
\[
\Gamma_n=O(K_n^{h_n}).
\]
More explicitly, without absorbing the last factor,
\[
\Gamma_n
=
O\left(
\frac{K_n^{h_n+1}}{\sqrt{K_n^2-1}}
\right).
\]

It remains to justify the crude ambient-width bound. Suppose the activations are bounded and the depth \(h\) is fixed. If
\[
\|\Psi_l(\bz)\|_\infty\le C_\Psi,
\]
then
\[
\|\bh_l(\bX;\btheta^*)\|_2^2
\le
C_\Psi^2 d_l.
\]
Since
\[
\bar\bh_l(\bX;\btheta^*)
=
\begin{pmatrix}
1\\
\bh_l(\bX;\btheta^*)
\end{pmatrix},
\]
we have
\[
\|\bar\bh_l(\bX;\btheta^*)\|_2^2
=
1+\|\bh_l(\bX;\btheta^*)\|_2^2
\le
1+C_\Psi^2 d_l.
\]
Therefore,
\[
\tau_{l,n}
=
\mathbb E\|\bar\bh_l(\bX;\btheta^*)\|_2^2
=
O(d_l).
\]
If the layerwise spectral norms are uniformly bounded and the depth \(h\) is fixed, then all downstream products satisfy
\[
\prod_{j=l+1}^{h+1}
\{L_j\|\bar\bW_j^*\|_{\rm op}\}^2
=
O(1).
\]
Hence
\[
\Gamma_n^2
\lesssim
\sum_{l=1}^{h+1}\tau_{l-1,n}
=
O\left(
\sum_{l=0}^h d_l
\right).
\]
Taking square roots gives
\[
\Gamma_n
=
O\left\{
\left(\sum_{l=0}^h d_l\right)^{1/2}
\right\}.
\]
This proves the proposition.
\end{proof}

\subsection{Proof of Theorem \ref{prop:1}} \label{app:proof-prop1}

\begin{proof}
\citet{Poggio2017WhyWhen} analyze the approximation power of deep neural networks for hierarchically compositional functions whose constituent maps have bounded arity (at most $s$ variables; e.g., $s=2$ for a binary tree). 
For this class of functions, they show:

\begin{lemma}[Theorem~4 of \citet{Poggio2017WhyWhen}]\label{lemma:composition}
Let $f:[0,1]^{d_0}\!\to\mR$ be $L$-Lipschitz and admit a hierarchical compositional
representation in which each constituent depends on at most $s$ variables.
Then a ReLU DNN that mirrors this compositional architecture can achieve
approximation error at most $\varepsilon$ (in $\ell^p$-norm) with
\[
m=\mathcal{O}\!\big((d_0-1)\,(L/\varepsilon)^{\,s}\big)
\]
hidden neurons.  
\end{lemma}

Let $d_0=O(n^{\alpha})$ for some $0<\alpha<1$ and set $\varepsilon=n^{-(1-\alpha-\delta)/s}$ with $0<\delta<1-\alpha$. 
Then, by Lemma~\ref{lemma:composition},
\[
m=\mathcal{O}\!\big(d_0 \,\varepsilon^{-s}\big)=\mathcal{O}\!\big(n^{\alpha}\,n^{1-\alpha-\delta}\big)
=\mathcal{O}\!\big(n^{1-\delta}\big),
\]
which satisfies the structural constraint in Theorem~\ref{thm:new1}.
% In practice, this approximation result is hard to deploy directly because the target’s compositional structure is typically unknown. 
% However, the parameter-estimation consistency in Theorem~\ref{thm:new1} implies that sublinear-width DNNs can recover this structure. 
% To further enhance structure learning, one may impose sparsity penalties,
%  such as Lasso \citep{Tibshirani1996}, SCAD \citep{FanL2001}, or MCP \citep{Zhang2010}), in training the deep networks. 
%  See \citet{SunLiang2025UQ} for theory on sparse structure recovery for DNNs under the StoNet framework. 
Consequently, for sublinear-width ReLU networks, Theorem~\ref{prop:1}-(i)  follows from Lemma~\ref{lemma:composition}. 

For sublinear-width DNNs with smooth activations, including sigmoid and $\tanh$, analogous guarantees hold for continuously differentiable, hierarchically compositional functions following from Theorem~2 of \citet{Poggio2017WhyWhen} (omit the details). 
\end{proof}

% We refer to \citet{Liang2018BNN} for empirical examples where a sparsity-regularized neural network recovers the structure of a hierarchical compositional function; in particular, their Figure~4 recovers the structure of the nonlinear regression function 
%  $y=10x_2/(1+x_1^2)+5\sin(x_3x_4)+2x_5+\epsilon$, where $x_i$'s and $\epsilon$ are Gaussian random variables.  

\subsection{Proof of Theorem \ref{prop:2}}

\begin{proof}

In the proof, 
\citet{Montanelli2019NEB} studied the functions in the Korobov space 
$\mathcal{K}^{2,p}$ (with $p$ indicating the $\ell^p$-norm) and 
proved the following result:
 
For any $0<\varepsilon<1$ and 
any function $f\in \mathcal{K}^{2,p}([0,1]^{d_0})$ 
that satisfies  
$\big\|\partial_{x_1}^{2}\cdots \partial_{x_{d_0}}^{2} f\big\|_{\infty}\le 1$, 
there exists a deep ReLU network on inputs $(x_1,x_2,\ldots,x_{d_0})^\top \in [0,1]^{d_0}$ that approximates
$f$ to accuracy $\varepsilon$, with depth $\mathcal{O}\!\big(|\log_2 \varepsilon|\,\log_2 d_0\big)$ and the number of hidden neurons 
\begin{equation}\label{eq:DMLP-3}
m=\mathcal{O}\!\Big(\varepsilon^{-2}\,|\log_2 \varepsilon|^{\frac{3}{2}(d_0-1)+1}\,(d_0-1)\Big).
\end{equation}

 If we set 
$\varepsilon = n^{-(1/2-\delta)}$ for some $0 < \delta < 1/2$ as the target approximation accuracy, 
then a deep ReLU network can achieve this accuracy for the target function $f$, provided that the network has depth  
$O\!\big((\tfrac{1}{2}-\delta)\log_2 n \log_2 d_0\big)$
and the number of hidden neurons  
\begin{equation} \label{eq:DMLP-4}
m \;=\; O\!\Big(
  n^{\,1-2\delta}\,
  (\log_2 n)^{\tfrac{3}{2}(d_0-1)+1}\,
  \bigl(\tfrac{1}{2}-\delta\bigr)^{\tfrac{3}{2}(d_0-1)+1}\,
  (d_0-1)
\Big)
= o \big(n^{\,1-\delta}\big),
\end{equation}
Here the input dimension $d_0$ is fixed or grows with $n$ at the rate 
$d_0=o\left({\log n}/{\log \log n}\right)$.
% by noting $\bigl(\tfrac{1}{2}-\delta\bigr)^{\tfrac{3}{2}(d_0-1)+1}(d_0-1) \prec 1$ when $d_0$ is reasonably large.  
% If the input dimension is allowed to grow with \(n\), say \(d=d_n\), assume $d_n=o\left({\log n}/{\log\log n}\right)$.
\end{proof}

% 

% %\end{document}

% \newpage
% \appendix
% %\onecolumn
% \begin{center} 
% {\LARGE \bf Appendix}
% \end{center}

% \setcounter{assumption}{0}
% \renewcommand{\theassumption}{A\arabic{assumption}}
% \setcounter{table}{0}
% \renewcommand{\thetable}{A\arabic{table}}
\providecommand{\theHtable}{}
\renewcommand{\theHtable}{A.\arabic{table}}
\providecommand{\theHequation}{}
\renewcommand{\theHequation}{A.\arabic{equation}}
\providecommand{\theHfigure}{}
\renewcommand{\theHfigure}{A.\arabic{figure}}
\providecommand{\theHlemma}{}
\renewcommand{\theHlemma}{A.\arabic{lemma}}

\clearpage
\begingroup
\appendix
\setcounter{section}{0}
\renewcommand{\thesection}{S\arabic{section}}
\renewcommand{\thesubsection}{S\arabic{section}.\arabic{subsection}}
\renewcommand{\thesubsubsection}{S\arabic{section}.\arabic{subsection}.\arabic{subsubsection}}
\providecommand{\theHsection}{}
\renewcommand{\theHsection}{S.\arabic{section}}
\providecommand{\theHsubsection}{}
\renewcommand{\theHsubsection}{S.\arabic{section}.\arabic{subsection}}
\setcounter{assumption}{0}
\renewcommand{\theassumption}{S\arabic{assumption}}
\setcounter{table}{0}
\renewcommand{\thetable}{S\arabic{table}}
\providecommand{\theHtable}{}
\renewcommand{\theHtable}{S.\arabic{table}}
\setcounter{equation}{0}
\renewcommand{\theequation}{S.\arabic{equation}}
\providecommand{\theHequation}{}
\renewcommand{\theHequation}{S.\arabic{equation}}
\setcounter{figure}{0}
\renewcommand{\thefigure}{S\arabic{figure}}
\providecommand{\theHfigure}{}
\renewcommand{\theHfigure}{S.\arabic{figure}}
\setcounter{lemma}{0}
\renewcommand{\thelemma}{S\arabic{lemma}}
\providecommand{\theHlemma}{}
\renewcommand{\theHlemma}{S.\arabic{lemma}}

\begin{center}
{\bf \large SUPPLEMENTARY MATERIAL}
\end{center}

\section{Additional Numerical Results}\label{ref.Numerical}

 This subsection presents numerical results that supplement those in the main body of the paper. The narrow DNN models were trained with a Tesla T4, while the wider DNN model were trained with an NVIDIA A100-SXM4-40GB.

 \subsection{Experimental setting}
 Table \ref{TabSetting} shows the detailed hyper parameter settings described in the manuscripts. All models were trained using SGD with a momentum coefficient of $0.9$, except for the MNIST  for which a momentum coefficient of $0.95$ was used. 

\begin{table}[!htbp]
\caption{Learning schedule for experiments, where $L$ denotes the width of hidden layers.}
\label{TabSetting}
\centering
\begin{adjustbox}{width=1.0\textwidth}
\begin{tabular}{cccccccc}
\toprule
Dataset & MNIST & Simulated Data & Boston & Yacht & Energy & Protein & CelebA \\
\midrule
Learning rate & 0.01 $(L\leq 2000)$ & 0.005 $(L\leq 10000)$, 0.001  $(10000<L\leq 30000)$ & 0.0005 & 0.0002 & 0.0005 & 0.0002 &0.05\\  
              & 0.001 $(L>2000)$ &  0.0005 $(30000<L\leq 70000)$, 0.0001 $(L>70000)$ &                    \\
Total epochs & 4000 & 12000 & 10000 & 10000 & 10000& 2000& 100 \\
Mini-batch size & 128 & 100 & 50 & 50 & 50 & 300&  64\\
\bottomrule
\end{tabular}
\end{adjustbox}
\end{table}

\subsection{A Test for the IRO Algorithm} \label{supp:IRO}

We simulate data from the following model:  
\begin{equation} \label{simueq}
\begin{split}
  x_{ij} &\sim Unif[-2,2],\quad j=1,\dots,5, \\
  y_i & = \frac{5x_{i1}}{1+x_{i2}^2}+5sin(x_{i3}x_{i4})+2x_{i5}+e_i, \\
\end{split}
\end{equation}
where $e_i \sim N(0,1)$. We set the training sample size $n_{train}=500$ and the test sample size $n_{test}=250$. 
We trained a one-hidden-layer DNN: $p$-L-1, with $p=5$ and  $L\in{2,4,8,16,32,50,100,120, 150}$,  and the ReLu activation function.  We train the models using both  SGD and IRO.
 Using SGD, we train the model for 12000 epochs with a learning rate of 0.005, a momentum coefficient of 0.9, and a batch size of 100. 
 Using IRO, we train the model for 6000 steps with the model initialized by the model at the 6000 epochs of SGD training. 
 For IRO, the imputation step is conducted by one step Langevin Dynamics update with step size $1e-6$, $\sigma_1^2 = 1e-2, \sigma_2^2=1e-3$. 
 
 Figure \ref{IRO_vs_SGD} shows that the performance of IRO and SGD are similar, see also Table \ref{backup_table} for numerical details. In practice, the IRO algorithm needs to solve a series of regressions on the entire data set for every iteration, it could be slow for large data sets and networks. So we use SGD in all of our experiments, while using the StoNet together with IRO as a bridge for 
transferring some properties of linear models to deep neural networks.  
    
    \begin{figure}[htbp]
        \centering
        \includegraphics[scale=0.5]{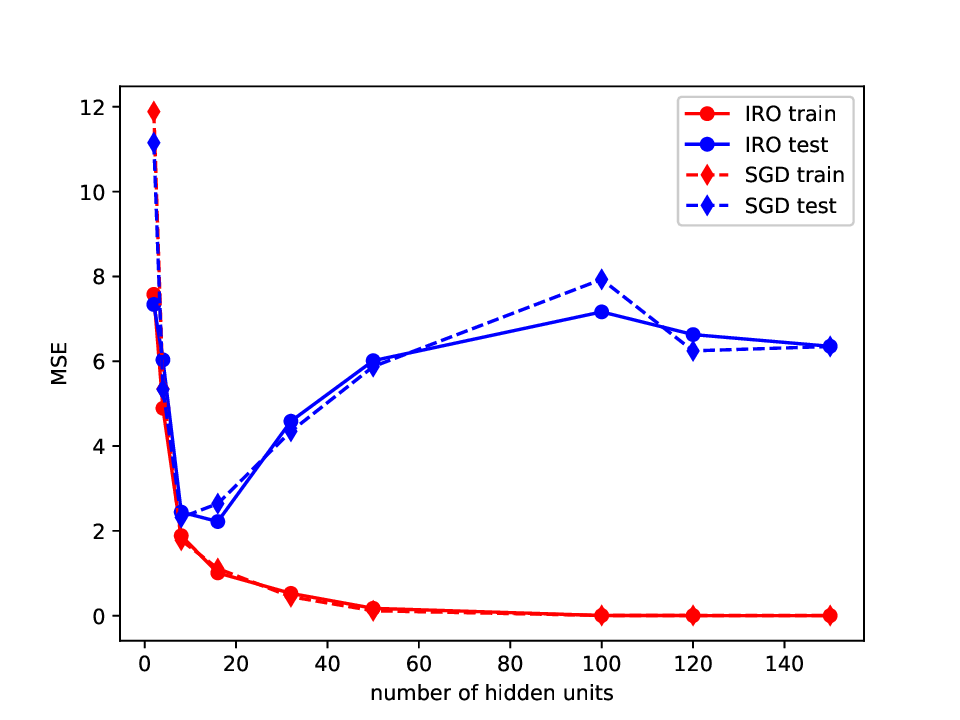}
        \caption{Training MSE and Testing MSE of neural network trained with SGD and IRO}
        \label{IRO_vs_SGD}
    \end{figure}
    
    \begin{table}[H]
        \centering
         \caption{Comparison of the training MSE and test MSEs produced by the neural networks trained using SGD and IRO 
          for the simulated nonlinear regression example.}
          \label{backup_table}
        \begin{adjustbox}{width=1.0\textwidth}
        \begin{tabular}{cccccccccc} \toprule
             & L = 2 &  L = 4 & L = 8 & L = 16 & L = 32 & L = 50 & L = 100 & L = 120 & L = 150 \\ \midrule
            IRO train &  7.577 &  4.891 &  1.886 &  1.009 &  0.526 &  0.174 &  0.004 &  0.002 &  0.002 \\
            IRO test & 7.339 &  6.033 &  2.443 &  2.219 &  4.585 &  6.013 &  7.164 &  6.629 &  6.352  \\ \midrule
            SGD train & 11.887 &  5.349 &  1.779 &  1.112 &  0.446 &  0.115 &  0.005 &  0.005 &  0.002  \\
            SGD test & 11.153 &  5.337 &  2.316 &  2.642 &  4.347 &  5.873 &  7.931 &  6.245 &  6.348  \\ \bottomrule
        \end{tabular}
        \end{adjustbox}
    \end{table}

\subsection{A nonlinear regression example} \label{suppRegression}

%We used batch normization for every layer.
% The data were generated from the following model:  
% \begin{equation} \label{simueq}
% \begin{split}
%   x_{ij} &\sim Unif[-2,2],\quad j=1,\dots,5, \\
%   y_i & = \frac{5x_{i1}}{1+x_{i2}^2}+5sin(x_{i3}x_{i4})+2x_{i5}+e_i, \\
% \end{split}
% \end{equation}
% where $e_i \sim N(0,1)$. We set the training sample size $n_{train}=500$ 
% and the test sample size $n_{test}=250$. 
% We trained a one-hidden-layer DNN: 5-L-1, with $L$ ranging from 2 to 256,000 and the ReLu activation function. 
% Figure \ref{fig:nonlinearReg} shows the plots for the resulting training and test errors. 

The datasets were simulated from model (\ref{simueq}) as in Section \ref{supp:IRO}. We 
trained a one-hidden-layer DNN: $p$-L-1, with $L$ ranging from 2 to 256,000 and the ReLu activation function. 
Figure \ref{fig:nonlinearReg} shows the plots for the training and test errors. 
Additionally, we fixed the weights from 
the input layer to the first hidden layer, and trained those from the second hidden layer to the output layer only. 
In this case, the neural network can still achieve zero training 
errors, but the test errors are very large.  The results are 
also summarized in Figure \ref{fig:nonlinearReg}.

\begin{figure}[htbp]
    \centering
    \includegraphics[width=0.45\textwidth]{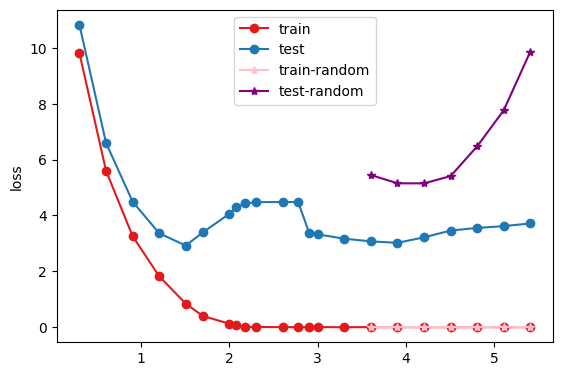}
    \caption{A simulated example for model (\ref{simueq}) with the network 
    architecture $p$-L-1, 
    where the $x$-axis represents $\log_{10}(L)$ and 
    the $y$-axis represents the mean squared error (MSE) averaged over 5 different datasets;  and the `train random' and `test random' represent, respectively, the training and test errors obtained 
    by the networks with the first layer weights being fixed to 
    random numbers.}
    \label{fig:nonlinearReg}
\end{figure}

In our experience, the performance of SGD is primarily sensitive to  learning rates. 
To explore this issue, 
we re-ran the  experiment with different learning rates. 
Specifically, we set the learning rate in the form \(\alpha\,\gamma_t\), where \(\{\gamma_t\}\) denotes the baseline (“standard”) schedule employed previously and reported in Table~\ref{TabSetting}, and \(\alpha \in \{2/3,\, 4/5,\, 5/4,\, 3/2\}\). 
Figure~\ref{fig:double-descent-lr} summarizes 
the training and test errors across different values of $\alpha$. 
The comparison shows: 
\begin{itemize} 
\item[(i)] The network training errors are fairly 
 robust to learning rates 
 (see red curves): When \(L\) is reasonably large, say \(L\ge 100\) (equivalently, \(\log_{10} L \ge 2\)), the training errors 
  consistently converge to 0.

\item[(ii)] For sublinear-width networks, the test errors (blue curves)
 are fairly robust to learning rates: In the sublinear regime (with $L<500$ or, equivalently,
\(\log_{10}(L)\le 2.7\)), the network test-error
curves are nearly unchanged as \(\alpha\) varies. 

\item[(iii)] For wide networks, the test errors are  sensitive to the learning rate: In the wide regime (with $L\geq 500$ or, equivalently, 
 \(\log_{10}(L) \geq 2.7\)), the network test-error trajectories differ noticeably when $\alpha$ is large.
 \end{itemize}
% \textcolor{black}{As indicated by Figure \ref{fig:double-descent-lr}, 
% our experiments also show that there exist sublinear-width networks that achieve near-optimal test error across a broad range of hidden-layer widths.}

\begin{figure}[H]
    \centering
    \begin{subfigure}{0.4\linewidth}
        \centering
        \includegraphics[width=\linewidth]{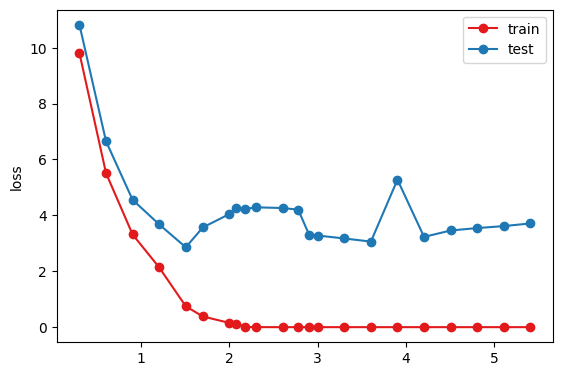}
        \caption*{(a) Larger learning rate with $\alpha=\tfrac{5}{4}$}\label{fig:big125}
    \end{subfigure}\hfill
    \begin{subfigure}{0.4\linewidth}
        \centering
        \includegraphics[width=\linewidth]{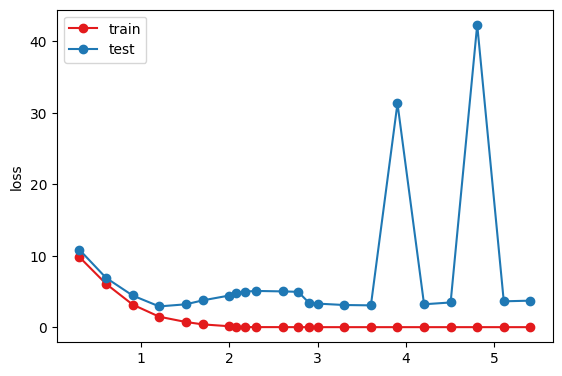}
        \caption*{(b)Larger learning rate with $\alpha=\tfrac{3}{2}$}\label{fig:big150}
    \end{subfigure}
    %\vspace{0.6em}
    \begin{subfigure}{0.4\linewidth}
        \centering
        \includegraphics[width=\linewidth]{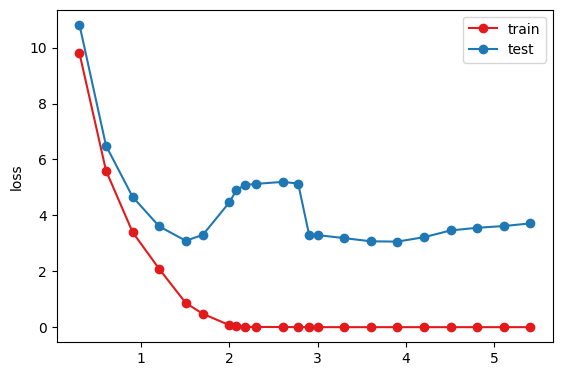}
        \caption*{(c) Smaller learning rate with $\alpha=\tfrac{4}{5}$}\label{fig:small080}
    \end{subfigure}\hfill
    \begin{subfigure}{0.4\linewidth}
        \centering
        \includegraphics[width=\linewidth]{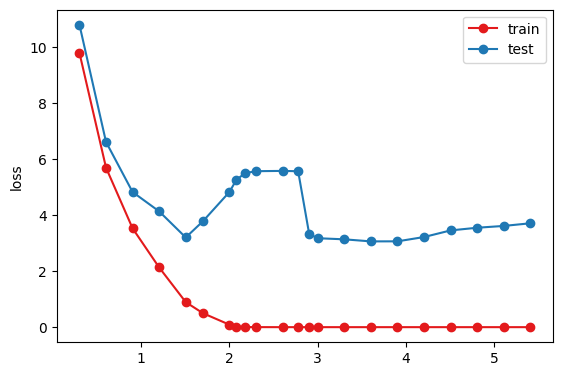}
        \caption*{(d)Smaller learning rate with $\alpha=\tfrac{2}{3}$}\label{fig:small066}
    \end{subfigure}
    \caption{
   An ablation study (with the sample size $n=500$) 
    for learning rates, where the horizontal axis is \(\log_{10}(L)\); the vertical axis is the mean squared error (MSE),
averaged over five independent datasets. 
%In particular, the region with\(\log_{10}(L)\le 2.7\) corresponds to sublinear-width DNNs, while \(\log_{10}(L)\ge 2.7\) corresponds to wide DNNs.
    % Ablation study for the learning rate of SGD, where the sample size 
    % $n=500$, the $x$-axis represents $\log_{10}(L)$, and 
    % the $y$-axis represents the mean squared error (MSE) averaged over 5 different datasets. In particular, 
    % the region with $\log(L) \leq 2.7$ corresponds to sublinear DNNs, and that with $\log(L) \geq 2.7$ corresponds to wide DNNs. 
    }
    \label{fig:double-descent-lr}
\end{figure}

\subsection{Comparison of sublinear-width and wide DNNs for nonlinear regression} \label{result:sublinear}

The datasets were simulated from model (\ref{simueq}) as in Section \ref{supp:IRO}. 
We trained neural networks of different architectures: $p$-L-1, $p$-L-L-1, and $p$-L-L-L-L-1.  
The training parameter settings are given in  Table~\ref{TabSetting}, except that a learning rate of 0.002 was used 
for the architecture $p$-L-L-L-L-1 with $L<1000$. 
For the  architecture  $p$-L-L-L-L-1 with $L=1000$, the learning rate was further reduced to 0.0001 in order to prevent gradient explosion.
Table~\ref{tab:narrow-wide comparison} reports the mean squared error (MSE) (with standard deviations given in parentheses), averaged over five independent runs, for both training and test sets. 
The comparison strongly suggests that sublinear-width networks can even outperform wide networks in prediction 
for this example. 

\begin{table}[htbp]
\centering
\caption{Training and test errors, measured in  MSE averaged over 5 independent runs (with standard deviations given in parentheses), 
 for nonlinear regression (\ref{simueq}) with network  architecture 
     ``$p$-L-$\cdots$-L-1'', where $h$ represents the number of
      hidden layers. The cases with the test MSE$<3.10$  
       are highlighted in red.}
\label{tab:narrow-wide comparison}
%\resizebox{\textwidth}{!}{%
\begin{adjustbox}{width=0.85\textwidth}
\begin{tabular}{ccccccccc} \hline
\multirow{2}{*}{Regime} & \multirow{2}{*}{Width (\textbf{$L$})} & 
\multicolumn{2}{c}{h=1} & 
\multicolumn{2}{c}{h=2} & 
\multicolumn{2}{c}{h=4} \\
\cline{3-8}
& & Train & Test & Train & Test & Train & Test \\
\hline
\multirow{20}{*}{Sublinear} & 
 2   & 9.80 (2.43) & 10.82 (3.00) & 12.51 (7.46) & 13.56 (6.10)  
& 20.23 (9.83) & 20.78 (10.25) \\
&16  & 2.02 (0.11) & 3.75 (0.50) & 0.35 (0.10) & 3.43 (0.20) 
%& 0.15 (0.05) & 3.53 (0.33) 
& 0.06 (0.02) & 4.21 (0.83) \\
&50  & 0.57 (0.05) & 4.13 (0.27) & 0.00 (0.00) & 3.78 (0.29) 
%& 0.00 (0.00) & 3.75 (0.49) 
& 0.00 (0.00) & 3.33 (0.47) \\
&100 & 0.15 (0.05) & 4.95 (0.56) & 0.00 (0.00) & 3.21 (0.33)  
%0.00 (0.00) & 3.14 (0.28) 
& {\bfseries\color{red} 0.00 (0.00)} & {\bfseries\color{red} 3.05 (0.41)} \\
&125 & 0.06 (0.07) & 4.54 (0.60) & {\bfseries\color{red} 0.00 (0.00)} & {\bfseries\color{red} 3.00 (0.29)} 
%& 0.00 (0.00) & 3.14 (0.40) 
& 0.00 (0.00) & 3.24 (0.33) \\
& 150 & 0.01 (0.00) & 4.98 (0.89) & 0.00 (0.00) & 3.25 (0.47) 
%0.00 (0.00) & 3.26 (0.42) 
& {\bfseries\color{red} 0.00 (0.00)} & {\bfseries\color{red} 2.99 (0.33)} \\
& 175 & 0.00 (0.00) & 5.00 (0.95) & {\bfseries\color{red} 0.00 (0.00)} & {\bfseries\color{red}  2.97 (0.32)} 
%& 0.00 (0.00) & 3.22 (0.30) 
& {\bfseries\color{red} 0.00 (0.00)} & {\bfseries\color{red} 2.96 (0.29)} \\
& 200 & 0.00 (0.00) & 4.41 (0.55) & 0.00 (0.00) & 3.26 (0.36) 
%& 0.00 (0.00) & 3.11 (0.23) 
& {\bfseries\color{red} 0.00 (0.00)} & {\bfseries\color{red} 3.01 (0.31)} \\
& 225 & 0.00 (0.00) & 4.29 (0.55) & 0.00 (0.00) & 3.11 (0.28) 
%& 0.00 (0.00) & 3.17 (0.40) 
& {\bfseries\color{red} 0.00 (0.00)} & {\bfseries\color{red} 3.08 (0.29)} \\
& 250 & 0.00 (0.00) & 4.27 (0.53) & 0.00 (0.00) & 3.18 (0.37) 
%& 0.00 (0.00) & 3.16 (0.30) 
& 0.00 (0.00) &  3.10 (0.31) \\
& 275 & 0.00 (0.00) & 4.17 (0.59) & 0.00 (0.00) & 3.18 (0.36) 
%& 0.00 (0.00) & 3.29 (0.38) 
& {\bfseries\color{red} 0.00 (0.00)} & {\bfseries\color{red} 3.05 (0.34)} \\
& 300 & 0.00 (0.00) & 4.10 (0.67) & {\bfseries\color{red}0.00 (0.00)} & {\bfseries\color{red} 3.07 (0.31)} 
%& 0.00 (0.00) & 3.23 (0.40) 
& {\bfseries\color{red}0.00 (0.00)} & {\bfseries\color{red} 3.05 (0.31)} \\
& 325 & 0.00 (0.00) & 3.74 (0.56) & 0.00 (0.00) & 3.11 (0.36) 
%& 0.00 (0.00) & 3.31 (0.39) 
& 0.00 (0.00) & 3.15 (0.32) \\
& 350 & 0.00 (0.00) & 3.93 (0.51) & 0.00 (0.00) & 3.14 (0.38) 
%& 0.00 (0.00) & 3.24 (0.34) 
& {\bfseries\color{red} 0.00 (0.00)} & {\bfseries\color{red} 3.05 (0.32)} \\
& 375 & 0.00 (0.00) & 3.82 (0.51) & {\bfseries\color{red} 0.00 (0.00)} & {\bfseries\color{red} 3.01 (0.30)} 
%& 0.00 (0.00) & 3.21 (0.35) 
& 0.00 (0.00) & 3.12 (0.35) \\
& 400 & 0.00 (0.00) & 3.76 (0.52) & 0.00 (0.00) & 3.16 (0.30) 
%& 0.00 (0.00) & 3.23 (0.42) 
& 0.00 (0.00) & 3.18 (0.40) \\
& 425 & 0.00 (0.00) & 3.70 (0.54) & 0.00 (0.00) & 3.10 (0.26) 
%& 0.00 (0.00) & 3.12 (0.29) 
& 0.00 (0.00) & 3.16 (0.31) \\
& 450 & 0.00 (0.00) & 3.74 (0.52) & 0.00 (0.00) & 3.13 (0.40) 
%& 0.00 (0.00) & 3.29 (0.37) 
& {\bfseries\color{red}0.00 (0.00)} & {\bfseries\color{red} 3.05 (0.31)} \\
& 475 & 0.00 (0.00) & 3.57 (0.45) & {\bfseries\color{red}0.00 (0.00)} & {\bfseries\color{red} 3.08 (0.37)} 
%& 0.00 (0.00) & 3.23 (0.29) 
& 0.00 (0.00) & 3.19 (0.34) \\ \midrule 
\multirow{6}{*}{Wide} 
& 500 & 0.00 (0.00) & 3.56 (0.52) & 0.00 (0.00) & 3.22 (0.37) 
%& 0.00 (0.00) & 3.32 (0.33) 
& 0.00 (0.00) &  3.10 (0.27) \\  
& 600 & 0.00 (0.00) & 3.78 (0.64) & {\bfseries\color{red}0.00 (0.00)} & {\bfseries\color{red} 3.06 (0.37)} 
%& 0.00 (0.00) & 3.29 (0.33) 
& 0.00 (0.00) & 3.25 (0.42) \\
& 700 & 0.00 (0.00) & 3.62 (0.54) & 0.00 (0.00) & 3.13 (0.41) 
%& 0.00 (0.00) & 3.35 (0.50) 
& 0.00 (0.00) & 3.24 (0.45) \\
& 800 & 0.00 (0.00) & 3.49 (0.49) & {\bfseries\color{red}0.00 (0.00)} & {\bfseries\color{red} 3.05 (0.31)} 
%& 0.00 (0.00) & 3.56 (0.45) 
& 0.00 (0.00) & 3.36 (0.42) \\
& 900 & 0.00 (0.00) & 3.47 (0.55) & 0.00 (0.00) & 3.20 (0.41) 
%& 0.00 (0.00) & 3.65 (0.54) 
& 0.00 (0.00) & 3.12 (0.33) \\
& 1000 & 0.00 (0.00) & 3.32 (0.49) & 0.00 (0.00) & 3.17 (0.44) 
%& 0.00 (0.00) & 3.51 (0.68) 
& 0.00 (0.00) & 3.83 (0.52) \\
\hline
\end{tabular} 
\end{adjustbox}
\end{table}

\newpage

\subsection{Feature Learning Consistency}

See Table \ref{CCtab:supp} and  Table \ref{CCtab:eigenvalue}.

\begin{table}[!ht]
\centering
\caption{
Canonical correlations $\rho_{4,1:k'}$ and $\rho_{5,1:k'}$ achieved by the network $p$-5-5-1 with $n=50,000$ for the simulated example in Section \ref{sect:feature}, 
where the canonical correlation and its standard deviation (in the parenthesis) are calculated by averaging over 5 independent datasets.}
\label{CCtab:supp}
\begin{adjustbox}{max width=1.0\textwidth} 
\begin{tabular}{cccccc} \toprule 
 $k$ & $\rho_{k,1:1}$ & $\rho_{k,1:2}$ & $\rho_{k,1:3}$ & $\rho_{k,1:4}$ & $\rho_{k,1:5}$  \\ \midrule 
 4& 0.03(0.01) & 0.05(0.01) & 0.05(0.01) & 1.00(0.00) & 1.00(0.00)
\\ 
5 & 0.14(0.07) & 0.17(0.07) & 0.18(0.07) & 0.19(0.07) & 1.00(0.00)

\\ \bottomrule
\end{tabular}
\end{adjustbox}
\end{table}

\begin{table}[!ht]
\centering
\caption{
Eigenvalues of $\bw_1^T \bw_1$ (denoted by ``true'') and those of 
$\hat{\bw}_1^T\hat{\bw}_1$ obtained by the network $p$-5-5-1 with 
different sample sizes for the simulated example in Section \ref{sect:feature}, where the mean and standard deviation (in the parenthesis) are calculated by averaging over 5 independent datasets.  
} 
\label{CCtab:eigenvalue}
\begin{adjustbox}{max width=1.0\textwidth} 
\begin{tabular}{ccccccc} \toprule 
 $n$ &  Model  & $\lambda_1$ & $\lambda_2$ & $\lambda_3$ & 
 $\lambda_4$ &  $\lambda_5$ \\ \midrule 
 --- & True & 16.86(0.55)& 13.97(0.24)& 10.32(0.46)& 6.44(0.73)& 4.56(0.33)\\ \midrule 
 500& p-5-5-1 &20.43(1.72)&12.34(1.26)&8.69(1.16)&5.01(0.97)&
 3.13(0.63)\\
%  & p-1000-1000-1 & 21.67(0.12) & 20.76(0.05) & 19.96(0.21) & 19.41(0.19)&18.93(0.16)\\
%  & p-10-10-1 & 10.87(1.36)  & 7.69(1.34)  & 5.63(1.1)  &4.06(0.88)  & 1.91(0.18)
% \\
%  & p-10-10-10-1 & 9.35(1.22) & 6.65(1.23) & 4.87(1.07) & 3.25(0.91) & 1.80(0.18) 
% \\
%  & p-10-10-10-10-1 & 7.51(0.72) & 5.25(0.67) & 3.95(0.77) & 3.18(0.67) & 1.83(0.11)\\
%  & p-10-10-10-10-10-1 & 6.79(0.71) & 5.17(0.61) & 3.49(0.42) & 2.75(0.46) & 1.81(0.17) \\
 \midrule 
 50000 & p-5-5-1 & 17.62(0.58) & 14.13(0.27) & 10.24(0.51) &  6.61(0.79) & 4.56(0.35) \\
%  & p-1000-1000-1 & 37.88(2.31) & 34.59(1.96) & 31.87(2.21) & 26.40(1.18) & 23.92(1.08)\\
%  & p-10-10-1 & 19.32(1.85) & 15.03(0.77) & 10.89(0.91) & 6.67(0.89) & 4.57(0.39)\\
%  & p-10-10-10-1 & 14.25(0.99) & 12.10(0.54) &  8.38(0.93) & 5.38(0.79)  & 3.69(0.42) \\
%  & p-10-10-10-10-1 & 13.72(0.19) & 11.15(0.34) & 8.52(0.72) & 4.68(0.49) & 3.70(0.39) \\
% & p-10-10-10-10-10-1 &  12.56(1.5) & 10.48(0.93) & 7.20(1.02) &  4.46(0.83) &  3.19(0.34)
%  \\
 \bottomrule
\end{tabular}
\end{adjustbox}
\end{table}

\subsection{UCI} 

See Table \ref{boston}, Table \ref{yacht}, and Table \ref{energy}.

\begin{table}[!ht]
    \caption{Training and test  errors, measured in MSE,
     for the Boston Housing Dataset ($n=506$, $p=13$) with network  
     ``$p$-L-L-$\cdots$-L-1'', where $h$ represents of the number of
      hidden layers, and five random splits were done with 
      $(n_{\rm train},n_{\rm test})=(400,106)$. he best test errors
are highlighted in bold.
      }
      \label{boston}
    \centering
    \begin{adjustbox}{max width=1.0\textwidth}
    \begin{tabular}{ccccccccc}
        \toprule
         Regime & Width (L) & & $h=2$ & $h=3$ & $h=4$ & $h=5$ & $h=6$ & $h=7$ \\
        \midrule
                \multirow{4}{*}{Sublinear} 
    %     & 20 & Train  &  \\
    %     &  &  Test  &  \\ \cline{2-9}
    % &  50 & Train  & 0.12(0.02)& 0.05(0.02)& 0.01(0.00) & 0.00(0.00) & 0.00(0.00) & 0.00(0.00) \\
    %     &  &  Test  &  4.36(0.56) & 3.93(0.16) & 3.38(0.17) & 3.44(0.14) & 3.32(0.15) & 3.42(0.23)\\ \cline{2-9}
             & 100 & Train  & 0.04(0.01) & 0.00(0.00) & 0.00(0.00) & 0.00(0.00) & 0.00(0.00) & 0.00(0.00) \\
        &  &  Test  & 3.86(0.23) & 3.59(0.25) & 3.30(0.20) & 3.21(0.13) & 3.02(0.15) &  3.00(0.14) \\ \cline{2-9}
         & 200 & Train  & 0.03(0.01) & 0.00(0.00) & 0.00(0.00) & 0.00(0.00) & 0.00(0.00) & 0.00(0.00) \\ 
        &  &  Test  & 3.59(0.30) & 3.29(0.15) & 3.23(0.19) & 3.11(0.16) & 3.09(0.19) & {\bf 2.97(0.12)} \\
        \midrule
     \multirow{6}{*}{Wide}    & 500 & Train  & 0.02(0.01) & 0.00(0.00) & 0.00(0.00) & 0.00(0.00) & 0.00(0.00) & 0.00(0.00) \\
        &  &  Test  & 3.52(0.17) & 3.27(0.12) & 3.16(0.17) & 3.07(0.16) & 3.01(0.17) & {\bf 2.97(0.20)} \\
                \cline{2-9}
         & 1000 & Train  & 0.02(0.00) & 0.00(0.00) & 0.00(0.00) & 0.00(0.00) & 0.00(0.00) & 0.00(0.00) \\
        &  &  Test  & 3.23(0.16) & 3.07(0.11) & 3.20(0.14) & 3.07(0.15) &  3.01(0.15) & 3.03(0.16) \\
                \cline{2-9}
        & 2000 & Train  & 0.02(0.00) & 0.00(0.00) & 0.00(0.00) & 0.00(0.00) & 0.00(0.00) & 0.00(0.00) \\
        &  &  Test  & 3.29(0.18) & 3.09(0.11) & 3.16(0.16) & 3.00(0.13) &  2.98(0.18) & 3.05(0.19) \\
        \bottomrule
    \end{tabular}
    \end{adjustbox}
\end{table}

\begin{table}[!htbp]
\caption{Training and test  errors, measured in MSE,
     for the Yacht Dataset ($n=308$, $p=6$) with network structure
``$p$-L-L-$\cdots$-L-1'', where $h$ represents the number of
hidden layers, and five random splits were done with $(n_{\rm train},n_{\rm test})=(270,38)$.
 The best test errors are highlighted in bold.}
\label{yacht}
\centering
\begin{adjustbox}{width=1.0\textwidth}
\begin{tabular}{ccccccccc}
\toprule
Regime & Width (L) & & $h=2$ & $h=3$ & $h=4$ & $h=5$ & $h=6$ & $h=7$ \\
\midrule
\multirow{4}{*}{Sublinear} & 100 & Train  & 0.07(0.00) & 0.04(0.00) & 0.03(0.00) & {\bf 0.03(0.00)} & 0.03(0.00) & 0.02(0.00) \\
& & Test  & 0.49(0.08) & 0.37(0.07) & 0.36(0.09) & {\bf 0.29(0.05)} & 0.34(0.07) & 0.44(0.07) \\
\cline{2-9}
& 200 & Train  & 0.05(0.00) & 0.03(0.00) & 0.02(0.00) & 0.02(0.00) & 0.02(0.00) & 0.02(0.00) \\
& & Test  & 0.46(0.07) & 0.36(0.08) & 0.36(0.08) & 0.32(0.04) & 0.36(0.07) & 0.31(0.05) \\
\midrule
\multirow{6}{*}{Wide} & 500 & Train Error & 0.05(0.00) & 0.03(0.00) & 0.02(0.00) & 0.01(0.00) & {\bf 0.01(0.00)} & 0.01(0.00) \\
& & Test  & 0.39(0.06) & 0.33(0.08) & 0.32(0.07) & 0.30(0.05) & {\bf 0.29(0.05)} & 0.35(0.05) \\
\cline{2-9}
& 1000 & Train & 0.06(0.00) & 0.03(0.00) & {\bf 0.02(0.00)} & 
{\bf 0.02(0.00)} & 0.02(0.00) & 0.02(0.00) \\
& & Test & 0.35(0.09) & 0.32(0.07) & {\bf 0.29(0.05)} & {\bf 0.29(0.06)} & 0.32(0.06) & 0.37(0.07) \\
\cline{2-9} 
& 2000 & Train  & 0.06(0.00) & 0.03(0.00) & 0.02(0.00) & 0.02(0.00) & 0.02(0.00) & 0.03(0.00) \\
& & Test & 0.34(0.07) & 0.33(0.07) &  0.30(0.06) & 0.32(0.06) & 0.32(0.06) & 0.32(0.07) \\
\bottomrule
\end{tabular}
\end{adjustbox}
\end{table}

% \begin{table}[!htbp]
% \caption{Training and test  errors, measured in MSE,
%      for the Energy Dataset ($N=768$, $p=8$)  with network structure
% ``$p$-L-L-$\cdots$-L-1'', where $H$ represents the number of
% hidden layers}
% \label{energy}
% \centering
% \begin{adjustbox}{width=1.0\textwidth}
% \begin{tabular}{ccccccccc}
% \toprule
% Regime & Width (L) & & $H=2$ & $H=3$ & $H=4$ & $H=5$ & $H=6$ & $H=7$ \\
% \midrule
% \multirow{6}{*}{Narrow} & 100 & Train  & 0.14(0.01) & 0.03(0.01) & 0.02(0.00) & 0.01(0.00) & 0.01(0.00) & 0.01(0.00) \\
% & & Test  & 0.63(0.05) & 0.55(0.04) & 0.56(0.04) & 0.66(0.04) & {\bf 0.50(0.03)} & 0.57(0.04) \\
% \cline{2-9}
% & 200 & Train  & 0.06(0.00) & 0.01(0.00) & 0.00(0.00) & 0.00(0.00) & 0.00(0.00) & 0.00(0.00) \\
% & & Test  & 0.54(0.03) & {\bf 0.51(0.03)} & 0.54(0.04) & 0.54(0.07) & 0.55(0.04) & 0.57(0.05) \\
% \cline{2-9}
% & 500 & Train  & 0.04(0.00) & 0.00(0.00) & 0.00(0.00) & 0.00(0.00) & 0.00(0.00) & 0.00(0.00) \\
% & & Test  & {\bf 0.51(0.02)} & 0.55(0.03) & 0.53(0.02) & 0.55(0.04) & 0.52(0.05) & 0.54(0.03) \\
% \midrule
% \multirow{4}{*}{Wide} & 1000 & Train & 0.03(0.00) & 0.00(0.00) & 0.00(0.00) & 0.00(0.00) & 0.00(0.00) & 0.00(0.00) \\
% & & Test  & 0.55(0.03) & 0.50(0.01) & 0.50(0.03) & {\bf 0.46(0.02)} & 0.51(0.03) & 0.49(0.02) \\
% \cline{2-9}
% & 2000 & Train  & 0.03(0.00) & 0.00(0.00) & 0.00(0.00) & 0.00(0.00) & 0.00(0.00) & 0.00(0.00) \\
% & & Test  & 0.53(0.03) & 0.53(0.04) & {\bf 0.47(0.03)} & 0.54(0.04) & 0.57(0.06) & 0.58(0.05) \\
% \bottomrule
% \end{tabular}
% \end{adjustbox}
% \end{table}

\begin{table}[!htbp]
\centering
\caption{Training and test errors (MSE) for the Energy Dataset ($n=768$, $p=8$) with network structure ``$p$-L-L-$\cdots$-L-1'', where $h$ represents the number of hidden layers, and five random splits were done with $(n_{\rm train},n_{\rm test})=(650,118)$.
The best test errors are highlighted in bold.
}
\label{energy}
\begin{adjustbox}{width=1.0\textwidth}
\begin{tabular}{ccccccccc}
\toprule
Regime & Width (L) & & $h=2$ & $h=3$ & $h=4$ & $h=5$ & $h=6$ & $h=7$ \\
\midrule
\multirow{6}{*}{Sublinear} & 100 & Train  & 0.16(0.02) & 0.03(0.00) & 0.02(0.01) & 0.00(0.00) & 0.01(0.00) & 0.01(0.00) \\
& & Test  & 0.71(0.01) & 0.68(0.04) & 0.72(0.04) & 0.66(0.05) & { 0.64(0.05)} & 0.65(0.03) \\
\cline{2-9}
& 200 & Train  & 0.06(0.01) & 0.00(0.00) & 0.00(0.00) & 0.00(0.00) & 0.00(0.00) & 0.00(0.00) \\
& & Test  & 0.67(0.05) & 0.66(0.03) & 0.67(0.03) & 0.64(0.03) & 0.68(0.05) & { 0.64(0.05)} \\
\cline{2-9}
& 500 & Train  & 0.05(0.00) & 0.00(0.00) & 0.00(0.00) & 0.00(0.00) & {\bf 0.00(0.00)} & 0.00(0.00) \\
& & Test  & 0.63(0.04) & 0.59(0.02) & 0.61(0.02) & 0.66(0.03) & {\bf 0.58(0.03)} & 0.65(0.03) \\
\midrule
\multirow{4}{*}{Wide} & 1000 & Train & 0.03(0.00) & 0.00(0.00) & 0.00(0.00) & {\bf 0.00(0.00)} & 0.00(0.00) & 0.00(0.00) \\
& & Test  & 0.65(0.03) & 0.61(0.03) & 0.61(0.02) & {\bf 0.58(0.03)} & 0.62(0.03) & 0.60(0.03) \\
\cline{2-9}
& 2000 & Train  & 0.02(0.00) & 0.00(0.00) & {\bf 0.00(0.00)} & 0.00(0.00) & 0.00(0.00) & 0.00(0.00) \\
& & Test  & 0.65(0.03) & 0.61(0.03) & {\bf 0.58(0.03)} & 0.60(0.05) & 0.61(0.03) & 0.68(0.03) \\
\bottomrule
\end{tabular}
\end{adjustbox}
\end{table}

\endgroup

\end{document}